\documentclass[12pt]{iopart}
\usepackage{subfig}
\usepackage{graphicx,color}
\usepackage{iopams}
\usepackage{amsopn,amsmath,amssymb,amsthm,mathtools}
\usepackage{comment}
\usepackage{xcolor}
\usepackage[colorlinks,pdfdisplaydoctitle,citecolor=blue, filecolor=black,urlcolor=blue]{hyperref}
\usepackage{cleveref}
\usepackage{tikz}
\usepackage{todonotes}

\usepackage{color}
\usepackage{ulem}
\usepackage[numbers]{natbib}

\definecolor{darkred}{rgb}{.7,0,0}
\definecolor{darkgreen}{rgb}{.3,.7,0}

\newcommand{\ti}{{\mathsf h}}

\newcommand{\R}{\mathbb{R}}
\newcommand{\N}{\mathbb{N}}

\newcommand{\dd}{\textrm{d}}

\newcommand{\Trace}{Tr}

\newcommand{\bu}{x}

\newcommand{\bB}{W}
\newcommand{\bBrev}{B}
\newcommand{\pr}{p}

\newcommand{\xs}{\mathsf{x}}
\newcommand{\ys}{\mathsf{y}}
\newcommand{\score}{\mathsf{s}}
\newcommand{\sreg}{\score_{\textsf{reg}}}
\newcommand{\seb}{\score_{\textsf{eb}}}

\newcommand{\J}{\mathsf{J}}
\newcommand{\Jcond}{\J_{\textsf{cond}}}
\newcommand{\Jcondzero}{\J_{\mathsf{cond},0}}
\newcommand{\Jreg}{\J_{\textsf{reg}}}
\newcommand{\Jregzero}{\J_{\textsf{reg},0}}
\newcommand{\I}{\mathsf{I}}

\newcommand{\ww}{\omega}
\newcommand{\w}{w}
\newcommand{\wwtilde}{\widetilde{\ww}}
\newcommand{\wtilde}{\widetilde{w}}

\newcommand{\la}{\langle}
\newcommand{\ra}{\rangle}
\newcommand{\Dm}{D_{-}}

\newtheorem{theorem}{Theorem}[section]

\newtheorem{example}[theorem]{Example}
\newtheorem{remark}[theorem]{Remark}
\newtheorem{proposition}[theorem]{Proposition}
\newtheorem{lemma}[theorem]{Lemma}
\newtheorem{assumption}[theorem]{Assumption}
\newtheorem{definition}[theorem]{Definition}
\numberwithin{equation}{section}

\newtheoremstyle{restated}{}{}{\itshape}{}{\bfseries}{.}{.5em}{\thmnote{#3}}
\theoremstyle{restated}
\newtheorem*{theorem_restated}{}

\renewcommand{\appendix}{\par
  \setcounter{section}{0}
  \setcounter{subsection}{0}
  \gdef\thesection{\Alph{section}}
}

\begin{document}

\title{Memorization and Regularization in Generative Diffusion Models}

\author{Ricardo Baptista$^{1}$, Agnimitra Dasgupta$^{2}$, Nikola B. Kovachki$^{3}$, Assad Oberai$^{2}$, and Andrew M. Stuart$^{1}$}
\address{$^{1}$Computing and Mathematical Sciences, California Institute of Technology}
\address{$^{2}$Aerospace and Mechanical Engineering, University of Southern California}
\address{$^{3}$NVIDIA Corporation}
\ead{\{rsb,astuart\}@caltech.edu, \{adasgupt,aoberai\}@usc.edu, nkovachki@nvidia.com}

\begin{abstract}
Diffusion models have emerged as a powerful framework for generative modeling.
At the heart of the methodology is score matching: learning gradients of families of log-densities for noisy versions of the data distribution at different scales. When the loss function adopted in score matching is evaluated using empirical data, rather than the population loss, the minimizer corresponds to the score of a time-dependent Gaussian mixture.
However, use of this analytically tractable minimizer leads to data memorization: in both unconditioned and conditioned settings, the generative model returns the training samples. This paper contains an analysis of the dynamical mechanism underlying memorization. The analysis highlights the need for regularization to %
avoid reproducing the analytically tractable minimizer; and, in so doing, lays the foundations for a principled
understanding of how to regularize. Numerical experiments investigate the properties of: (i) Tikhonov regularization; (ii) regularization designed to promote asymptotic consistency; and (iii) regularizations induced by under-parameterization of a neural network or by early stopping when training a neural network. 
These experiments are evaluated in the context of memorization,
and directions for future development of regularization are highlighted.
\end{abstract}

\section{Introduction} \label{sec:introduction}

\subsection{Context}

The goal of generative modeling is to characterize an unknown probability distribution $p_0$ that underlies a collection of given samples 
$\{x_0^n\}_{n=1}^N \sim p_0$ (assumed i.i.d.). 
A successful generative model learns the structural properties of 
the data and is able to generate {\it new} 
samples with the same characteristics; moreover, it is desirable that the model is
asymptotically consistent with $p_0$ as $N \to \infty.$ 
Over the last decade, various categories of generative models 
have been successfully used in applications including, for example,  
image generation~\citep{dhariwal2021diffusion}, 
audio synthesis~\citep{kong2020diffwave},  
drug discovery~\citep{chenthamarakshan2020cogmol} 
and the prediction of natural disasters~\citep{ravuri2021skilful}. 

Score-based diffusion models are a recent class of generative models that exploit a
noising and denoising process for the creation of new samples~\citep{song2020score}.
These models have achieved impressive results for benchmark problems and have become a core part of commercial image generators such as DALLE-2~\citep{ramesh2022hierarchical}, Stable Diffusion~\citep{rombach2022high} and Imagen~\citep{saharia2022photorealistic}.
The noising-denoising process may be formulated using the distributional time-reversal of stochastic differential equations~\citep{anderson1982reverse,haussmann1986time}. The 
papers~\cite{song2020score, song2021maximum} present two canonical Gaussian diffusion processes
which are useful in the context of algorithms for noising-denoising: the {\it variance exploding} and the {\it variance preserving} diffusion processes. 
A challenge facing all generative models, and score-based diffusion models in particular,
is that a perfect unregularized model only sees an empirical distribution and may learn to simply reproduce the training data, a form of overfitting known as {\it memorization}. Memorization can cause privacy and copyright risks as training data can be extracted from a learned model~\citep{duan2023diffusion}; it is also symptomatic of a lack of diversity in the generative model. 
 
Score-based diffusion models utilize the solution of an inverse problem to recover the {\it score function}, the gradient of the log-density defined at each noise level of the noising process. The loss function adopted, when evaluated using empirical data rather than the population loss, is minimized at a score defined by a time-dependent Gaussian mixture. Defining the reverse process with this score leads to memorization. This suggests the need for principled regularization of the inverse problem. Therefore, there is a need to mathematically understand the behavior that leads to memorization in order to regularize the inverse problem for the score and as a result to prevent memorization. The aim of this paper is to develop a mathematical analysis of memorization and use this as a lens through which to interpret various regularization methods.

\subsection{Contributions and Overview}

We make three primary contributions in this paper:
\begin{enumerate}
    \item We provide a self-contained formulation of score-based generative 
    modeling as an inverse problem. The variational problem is formulated as the minimization of a 
    loss function defined via time-averaging of expectations with respect to the time marginals of a Gaussian diffusion process; 
    we distinguish between denoising with diffusions and with deterministic flows.
    \item We prove that use of the closed-form minimizer in the deterministic flow leads to memorization, exploiting the Voronoi diagram generated by the dataset to facilitate a dynamical systems analysis; both the variance exploding and the variance preserving diffusion processes are studied.
    \item We study a variety of regularization techniques, classical and modern, through the lens of memorization; in doing so we provide insight into potential directions for methodological developments in the field.
\end{enumerate}

This section concludes with a literature review and then the remainder of this article is organized as follows. Contribution (i) is contained in 
Sections~\ref{sec:background} and~\ref{sec:score_matching}: Section~\ref{sec:background} outlines the background on score-based diffusion models; %
Section~\ref{sec:score_matching} studies the denoising loss function used to learn score functions. Section~\ref{sec:singularity_data_diffusion} analyzes the behavior of an ordinary differential equation, which defines a transport based on the learned score function. We prove memorization, thereby addressing contribution (ii), for both the variance exploding and preserving methodologies; the analysis is accompanied by illustrative numerical results.
Section~\ref{sec:regularization} is devoted to contribution (iii),
proposing regularization approaches to avoid memorization; numerical results are presented on both low-dimensional distributions with Lebesgue density, and distributions over images. Section~\ref{sec:conditioning} shows that the framework extends to conditional sample generation.
Details of some of the proofs are reserved for the appendices to allow for a streamlined exposition of ideas in the main text. The code used to produce all the numerical results can be found at \url{www.github.com/baptistar/DiffusionModelDynamics}. A preliminary version of the results contained herein were presented at the Isaac Newton Institute \it{Programme on Diffusions in Machine Learning: Foundations, Generative Models and Non-convex Optimisation.} A video of the presentation, from a workshop, that took place
July 15th--19th, can be found at \url{www.newton.ac.uk/seminar/43284}.

\subsection{Literature Review}

Memorization generally appears when learning a distribution from an empirical dataset. With a sufficiently rich model, the maximum likelihood estimator for the distribution is a combination of Dirac deltas on the data points, which is thus a form of memorization. For classic kernel 
density estimation, memorization is avoided by careful choices of the 
bandwidth parameters in the kernel to ensure the model consistently recovers the 
data distribution with an increasing number of data samples~\citep{silverman1978weak,silverman2018density}. 
For neural network estimators, memorization has been 
observed for deep generative models such as GANs~\citep{bai2022reducing, nagarajan2018theoretical}, large language models~\citep{carlini2021extracting} and, most 
recently, for diffusion models. Indeed, for diffusion models, memorization has 
been observed for both image generation in~\cite{carlini2023extracting, ma2024could,somepalli2023understanding,somepalli2023diffusion} and video generation in~\cite{chen2024investigating}; memorization in this context is observed in both the unconditioned and conditioned settings. The factors contributing to memorization behavior and the transition from generalization to memorization regimes have been studied in~\cite{gu2023memorization, yoon2023diffusion, zhang2023emergence}. 
Necessary conditions to avoid sampling only from the support of the training data points are given in~\cite{pidstrigach2022score}.
However, whilst neural networks and training procedures often implicitly regularize, the mechanisms by which this occurs are not well understood.

Starting from the seminal work in~\cite{song2020score}, presenting a continuous-time framework for diffusion models, most analysis has focused on the consistency of sampling, assuming an accurate score function. For example,~\cite{chen2022sampling} exhibits an analysis relating small errors in the initial condition for denoising, the score function and the (reverse) time discretization to the quality of the approximate distribution for the resulting samples. The paper~\cite{song2021maximum} relates the error in the score function to maximizing the likelihood of the approximated data distribution. This analysis has also been extended to distributions that satisfy the manifold hypothesis~\citep{lu2023understanding, de2022convergence}. In practice, however,
memorization is a frequently observed issue; the paper~\cite{li2024good} provides an example of a score function that has a small estimation error, but only learns to generate empirical data with Gaussian blurring and hence lacks the ability to create genuinely new content. Work that is closely related to ours, concerning  the analysis of memorization in ODE trajectories arising in flow matching models
\cite{boffi2024flow} (a cousin of of diffusion models) is presented in~\cite{wan2024elucidating}, a paper which
appeared concurrently with our work.

Studies of memorization have motivated the development of novel regularization methods to mitigate this behavior in generative models. For diffusion models, these include approaches that directly modify the score function for the empirical data: the paper~\cite{scarvelis2023closed} proposes a smoothed estimator for the score which samples from barycenters of the training data; the paper~\cite{wibisono2024optimal} considers regularization based on empirical Bayes; and the paper~\cite{zhang2024wasserstein} incorporates the geometry of the data distribution in the learned score function. When working with neural network score functions, approaches to mitigate memorization modify the sampling process when the generation of training data is detected~\citep{wen2024detecting, chen2024towards, ren2024unveiling} or limit the capacity of the fine-tuned network during training~\citep{dutt2024capacity}. These efforts have also motivated studies of data extraction~\citep{carlini2023extracting, chen2024towardstheory}.

So far, proposed approaches to mitigate memorization have not gained traction in
methodologies employed in practice; and whilst the neural network 
methodologies employed in practice do often implicitly regularize, 
the mechanisms by which this occurs are not well understood. 
In this work, we provide a mathematical framework for the understanding 
of memorization; and we provide computational and analytic studies of regularization 
mechanisms that can control memorization behavior and produce new samples, in the context of score-based diffusion models.

\section{Score-Based Generative Modeling} \label{sec:background}

We start by assuming that we have access to the distribution $p_0$ on $\R^d$;
later we will assume only that we have access to samples from $p_0.$ 
Score-based diffusion models are based on a procedure to sample from $\pr_0$ by time-reversing a stochastic differential equation (SDE). A {\it forward process}, defined by an SDE, transforms an initial distribution $\pr_0$ at time $t=0$ into a specified reference distribution $\pr_T$ at time $t=T$; then a {\it reverse process}, defined by an SDE or an ordinary differential equation (ODE), transforms distribution $\pr_T$ 
at time $T$ back into the data distribution $\pr_0$ at time $t=0.$ These forward and reverse processes are characterized in Subsections~\ref{ssec:F} and \ref{ssec:R}, respectively.
In Subsection~\ref{ssec:AP} we describe how the drift function for the reverse process is approximated
to enable generative modeling based only on samples
from $p_0$ rather than knowledge of $p_0$ itself; the resulting approximations
lead to methodologies that are implementable in practice.

\subsection{Forward Process}
\label{ssec:F}

Let $\bu \in C([0,T],\R^d)$ satisfy the forward SDE
\begin{equation} \label{eq:ForwardSDE}
\frac{\dd\bu}{\dd t} = f(\bu,t) + \sqrt{g(t)}\frac{\dd\bB}{\dd t}, \quad \bu(0) = \bu_0
\sim \pr_0,
\end{equation}
where $W$ is a standard $\R^d$-valued Brownian motion independent of $p_0$, $f \colon \R^d \times [0,T] \rightarrow \R^d$ is a vector-valued drift, and $g \in C([0,T];\mathbb{R}_{\geq 0})$ 
is a scalar-valued diffusion coefficient. The density of $x(t)$ solving \eqref{eq:ForwardSDE}, at each time $t \in [0,T]$, is given by the 
solution $p(\cdot,t)$ of the Fokker-Planck equation
\begin{equation}
    \label{eq:sFP2}
    \partial_t p =-\nabla_x \cdot (fp) + \frac{g}{2}\triangle_x p, \quad p(x,0)=p_0(x).
\end{equation}
Conditions on $f$ to ensure almost sure pathwise existence and uniqueness of the
process $x$ may be found in \cite{mao2007stochastic}.

\subsection{Reverse Process}
\label{ssec:R}

Let $p_T(\cdot) \coloneqq p(\cdot,T)$ denote the solution of Fokker-Planck equation \eqref{eq:sFP2}
at time $t=T.$ Now let $\bu \in C([0,T],\R^d)$ satisfy the backward SDE
\begin{equation} \label{eq:BackwardSDE}
    \frac{\dd \bu}{\dd t} = f(\bu,t)-\alpha_1 g(t) \nabla_x \log \pr(\bu,t) + \sqrt{\alpha_2 g(t)} \frac{\dd \bBrev}{\dd t}, \quad \bu(T) = \bu_T
\sim \pr_T,
\end{equation}
where $\pr(x,t)$ solves \eqref{eq:sFP2}, $\alpha_1, \alpha_2$ are positive 
constants, and $\bBrev$ is a reverse-time $\R^d$-valued Brownian motion independent 
of $p(\cdot,t)$ and, in particular, independent of $\bB$. We seek to find constraints on $(\alpha_1,\alpha_2)$ so that the
reverse process is equal in law, at each time
$t \in [0,T]$, to the forward process governed by~\eqref{eq:sFP2}. 

The density $\rho(x,t)$ of $x(t)$ solving \eqref{eq:BackwardSDE} satisfies the  following backward Fokker-Planck equation with final time condition given by the density $p_T$: 
\begin{equation}
    \label{eq:sFP22}
    \partial_t \rho =-\nabla_x \cdot (f\rho)+ \alpha_1 g\nabla_x \cdot (\nabla_x \log\pr\, \rho) - \alpha_2\frac{g}{2}\triangle_x \rho, \quad \rho(x,T)=p_T(x).
\end{equation}
Note the negative sign on the diffusion term resulting from integrating backwards
in time; in particular, since $\alpha_2 \ge 0$, this sign makes the equation
well-posed in backwards time (under mild conditions on $f,g,p$~\citep{pavliotis2014stochastic,rogers2000diffusions}).
We now set $2\alpha_1=\alpha_2+1.$ Then, the choice $\rho(x,t)=\pr(x,t)$ solves  equation \eqref{eq:sFP22} as the equation reduces to \eqref{eq:sFP2}. Thus, with  $2\alpha_1=\alpha_2+1$, solving \eqref{eq:BackwardSDE} backwards from $t=T$ to $t=0$ will result in  $\bu(t) \sim p(\cdot;t)$
for each $t \in [0,T]$; and, in particular, $\bu(0) \sim p_0.$

\begin{example}
\label{ex:revs}
The choice $\alpha_1=\alpha_2=1$ gives the reverse SDE
\begin{equation} \label{eq:ReverseSDE_approximate_score}
  \frac{\dd\bu}{\dd t} = f(\bu,t) - g(t) \nabla_x \log \pr(\bu,t)+ \sqrt{g(t)}\frac{\dd \bBrev}{\dd t}.
\end{equation}
\end{example}
\begin{example}
\label{ex:revo}
The choice $\alpha_1=\frac12, \alpha_2=0$ gives the reverse ODE
\begin{equation}
\label{eq:mf_ode}
    \frac{\dd \bu}{\dd t}= f(\bu,t) - \frac{g(t)}{2}\nabla_x \log \pr(\bu,t).
\end{equation}
\end{example}

Thus, in principle, we may integrate either of \eqref{eq:ReverseSDE_approximate_score}
or \eqref{eq:mf_ode} backwards in time from $\bu(T) \sim p_T$ and at $t=0$ the
solution will be distributed according to $p_0.$ Introducing judicious approximations
leads to a methodology to approximately generate new samples from $p_0$, 
given only a finite set of $N$ existing samples: a generative model. 
We describe this in the next subsection.

\subsection{The Reverse Process In Practice}
\label{ssec:AP}

As is common in diffusion-based generative models, we will consider $f$ given by
\begin{equation}
\label{eq:fbeta}
f(\bu,t) = -\frac{1}{2}\beta(t)\bu,
\end{equation}
where $\beta \in C([0,T],\mathbb{R}_{\geq 0})$.
With such $f(\cdot,t)$, which is linear in $x$,  we have almost sure pathwise existence and uniqueness. And with $p_0$ a Dirac, use of Girsanov formula shows that the solution $\bu(\cdot)$ of SDE~\eqref{eq:ForwardSDE} with linear drift~\eqref{eq:fbeta} is Gaussian on pathspace; it
hence has a Gaussian marginal distribution on $\bu(t)$ for all $t$. Thus, in this setting, $p(\cdot,t)$ is Gaussian for each $t \in [0,T].$ Two widely adopted choices for $(\beta,g)$ are now defined.

\begin{example} \label{ex:VE} Setting $\beta=0$ yields the 
variance exploding process (VE)
\begin{equation} \label{eq:VE_SDE}
  \frac{\dd\bu}{\dd t} = \sqrt{g(t)}\frac{\dd\bB}{\dd t}.  
\end{equation}
\end{example}

\begin{example} \label{ex:VP} Setting $\beta=g$ yields the variance preserving process (VP) 
\begin{equation} \label{eq:VP_SDE}
  \frac{\dd\bu}{\dd t} = -\frac{1}{2}g(t)\bu + \sqrt{g(t)}\frac{\dd\bB}{\dd t}.  
\end{equation}
\end{example}

Thus, for the remainder of this work we will consider the forward SDE~\eqref{eq:ForwardSDE} with linear drift~\eqref{eq:fbeta} given by
\begin{equation} \label{eq:ForwardSDE_lineardrift}
    \frac{\dd\bu}{\dd t} = -\frac{1}{2}\beta(t)x + \sqrt{g(t)}\frac{\dd\bB}{\dd t}, \quad \bu(0) = \bu_0
\sim \pr_0,
\end{equation}
The functions $\beta,g$ are chosen so that 
$\pr_T$ is approximately a prespecified Gaussian $\mathfrak{g}_T$,
and, in particular, so that the information from $\pr_0$ remaining at time $T$ is 
negligible. Sampling from $\mathfrak{g}_T$ rather than $p_T$ at time $T$, 
in the backward SDE~\eqref{eq:BackwardSDE} with drift~\eqref{eq:fbeta} and
$2\alpha_1=\alpha_2 +1$, introduces an error and it is no longer
the case that $x(0) \sim p_0.$
However, this error can be controlled by choosing $g(t)$
so that $\mathfrak{g}_T$ and $p_T$ are close. 

\begin{example}
For the variance exploding process in~\eqref{eq:VE_SDE} we may choose $g(t)=t^a$, $a>0$, and choose $T$ sufficiently large. Specifically, $p_T \approx \mathfrak{g}_T \coloneqq \mathcal{N}(0,\frac{T^{a + 1}}{a+1} I_d)$ by choosing $T$ sufficiently large so that the initial draw from $p_0$ is small compared to the noise.
\end{example}

\begin{example}
For the variance preserving process in~\eqref{eq:VP_SDE} we may choose $g(t)=1$. Then, the effect of the initial condition drawn from $p_0$ is discounted by $\exp(-t/2)$ and $p_T \approx \mathfrak{g}_T \coloneqq \mathcal{N}(0,I_d)$ if $T$ is chosen sufficiently large.
\end{example}

The unknown true score $\nabla_x \log \pr(x,t)$ is approximated by the methodology described in the next section, leading to an approximation $\score(x,t).$ Then, the reverse process \eqref{eq:BackwardSDE}, with $f$ defined by~\eqref{eq:fbeta} and with the desired constraint $2\alpha_1=\alpha_2+1$, gives the diffusion process
\begin{equation} \label{eq:BackwardSDE2}
    \frac{\dd \bu}{\dd t} = -\frac{1}{2}\beta(t)\bu-\frac12(1+\alpha_2) g(t) \score(\bu,t) + \sqrt{\alpha_2 g(t)}\frac{\dd \bBrev}{\dd t}, \quad \bu(T) = \bu_T
\sim \mathfrak{g}_T.
\end{equation}
This equation is integrated, by a numerical
method, backwards from time $t=T$, to time $t=0$; numerical integration
introduces another controllable error. 
The solution at $t=0$ gives our approximate sample. 
In practice, the diffusion process is also often stopped 
early at some small positive time $t,$ introducing another 
controllable approximation of the exact reverse process. 
Indeed, in practical implementations of score-based generative modeling,
the only uncontrolled part of the approximate reverse process
is the approximation $\score(x,t)$ of the score. The next section
is devoted to computational methods to find approximations $\score(x,t)$. We identify
a particular approximation, $\score^N(x,t)$, that will play a distinguished role in our analysis.

\section{Approximation of the Score} \label{sec:score_matching}

In Subsection~\ref{ssec:SM} we describe the inverse problem for the
score $\score$, introducing \textit{score matching} via a population loss $\J$ 
and its empirical approximation $\J^N.$ Subsection \ref{ssec:RGM} studies the
approximate score $\score^N$, resulting from exact minimization of $\J^N$, pointing towards the memorization effect. 
In Subsection~\ref{ssec:prac} we describe in
practice how $\J^N$ is minimized, by consideration of an empirical loss $\J_0^N$, which differs from $\J^N$ by a constant,
but leads to straightforward implementation. 
In Subsection~\ref{ssec:LTN} we 
describe a variant on basic score matching, defined by learning
to explicitly remove the additive noise, leading to the modified objective $\I_0$ and its empirical approximation $\I_0^N.$

\subsection{Score Matching}
\label{ssec:SM}

Let $\lambda$ denote a time-dependent weighting function in $\Lambda := C((0,T];(0,\infty))$ and $|\cdot|$ the Euclidean norm in $\R^d.$ Then define the
\textit{score matching loss}
\begin{equation} \label{eq:score_matching_loss}
  \J(\score) = \int_0^{T} \lambda(t) \mathbb{E}_{x \sim \pr(\cdot,t)}|\score(x,t) - \nabla_\bu \log \pr(\bu,t)|^2 \dd t,  
\end{equation}
where $p(x,t)$ solves equation \eqref{eq:sFP2}. The loss function $\J(\cdot)$ is minimized at $\score(x,t)=\nabla_x \log p(x,t)$, delivering the exact score. 

In practice we wish to solve the inverse problem of recovering $\nabla_x \log p(x,t)$
given only samples from $p_0;$ thus we do not have $p(x,t).$ 
To define this inverse problem,
let $p_0^N$ denote the empirical distribution generated by the data $\{x_0^n\}_{n=1}^N$,
assumed to be drawn i.i.d.\thinspace from $p_0:$
\begin{equation}
\label{eq:empirical}
p_0^N=\frac{1}{N} \sum_{n=1}^N \delta_{x_0^n}.
\end{equation}

Let $\pr^N(x,t)$ denote the probability distribution defined by the law of $\bu(t)$
solving~\eqref{eq:ForwardSDE_lineardrift} when $p_0$ is replaced by $p_0^N.$
Because we consider only forward SDEs with linear drift, %
$\pr^N(x,t)$
is a time-dependent Gaussian mixture. Consider now the approximation $\J^N$ of the loss
$\J$, defined by
\begin{equation} \label{eq:score_matching_lossN}
  \J^N(\score) = \int_0^{T} \lambda(t) \mathbb{E}_{x \sim \pr^N(\cdot,t)}|\score(x,t) - \nabla_\bu \log \pr^N(\bu,t)|^2 \dd t.  
\end{equation}
The loss function $\J^N(\cdot)$ is minimized at the {\it empirical score} $\score^N(x,t)=\nabla_x \log p^N(x,t)$, and this
minimizer is explicitly known through the Gaussian mixture structure of $\pr^N(x,t)$. It is this structure which leads to memorization 
of the data, a failure of the generative model,
and hence to a need for regularization.

\subsection{Properties of the Empirical Score}
\label{ssec:RGM}

Recall that when introducing the forward SDE in~\eqref{eq:ForwardSDE_lineardrift} we have
assumed that the drift %
is linear. %
For the random variable $\bu(t)$ solving~\eqref{eq:ForwardSDE_lineardrift}, its time-varying density can be written as\footnote{All integrals are over $\R^d$ unless stated to the contrary.}
\begin{equation} \label{eq:MarginalDensityXt}
p(x,t) = \int p(x,t|x_0) p_0(x_0) \dd x_0,
\end{equation}
where $p(x,t|x_0)$ is the Gaussian distribution of the 
SDE~\eqref{eq:ForwardSDE_lineardrift} %
when initialized at $\bu(0)=x_0.$ Furthermore, if $p_0$ is replaced by its empirical
approximation \eqref{eq:empirical}, then we obtain
\begin{equation} \label{eq:MarginalDensityXtN}
p^N(x,t) = \int p(x,t|x_0) p_0^N(x_0) \dd x_0=\frac{1}{N} \sum_{n=1}^N p(x,t|x_0^n),
\end{equation}
a Gaussian mixture. To understand its structure in more detail, we use
the following lemma, proved in Appendix~\ref{app:linearSDEs}.

\begin{lemma} \label{lemma:SDEsoln} 
For non-negative $\beta \colon [0,T] \rightarrow \R$ and strictly positive
$g \colon [0,T] \rightarrow \R_{\geq 0}$, the solution $x(t) \in \R^d$ for $t \geq 0$ of the SDE 
$$\frac{\dd \bu}{\dd t} = -\frac{1}{2}\beta(t)\bu + \sqrt{g(t)}\frac{\dd \bB}{dt}, \quad \bu(0) = \bu_0,$$
has marginal law at each time $t$ given by 
\begin{subequations}
\label{eq:msig}
\begin{align}
\bu(t)|\bu_0 &\sim \mathcal{N} \left(m(t)\bu_0, \sigma^2(t)\,I_d \right),\\
m(t) &\coloneqq \exp\left(-\frac{1}{2}\int_0^t \beta(s) \dd s\right),\\
\sigma^2(t) &\coloneqq m^2(t) \int_0^t \frac{g(s)}{m^2(s)} \dd s. \label{eq:msig_variance}
\end{align}
\end{subequations}
\end{lemma}

Using Lemma \ref{lemma:SDEsoln} we see that the score for the Gaussian conditional distribution $\pr(\bu,t|\bu_0)$ is then
\begin{equation} \label{eq:conditional_score}
\nabla_x \log p(\bu,t|\bu_0) = -\frac{1}{\sigma^2(t)}\bigl(x - m(t)\bu_0\bigr).
\end{equation}
From this identity, we deduce the following theorem, proved in Appendix~\ref{ssec:wnc}.

\begin{theorem} \label{thm:empirical_score_unconditional} The score function that minimizes the empirical loss function \eqref{eq:score_matching_lossN}, for any weighting $\lambda \in \Lambda$, has the form
\begin{align} \label{eq:empirical_score} 
\score^N(x,t) & \coloneqq -\frac{1}{\sigma^2(t)}\sum_{n=1}^N \bigl(x - m(t)x_0^n\bigr) w_n(x,t). %
\end{align}
Here, $m(t),\sigma^2(t)$ are defined by \eqref{eq:msig}
and $\w_n \colon \R^d \times [0,T] \rightarrow [0,1]$ 
are the normalized Gaussian weights 
\begin{equation} \label{eq:normalized_Gaussian_weights}
  \w_n(x,t) = \frac{\wtilde_n(x,t)}{\sum_{\ell=1}^N \wtilde_\ell(x,t)}, \qquad \wtilde_n(x,t) = \exp\left(-\frac{|x - m(t)x_0^n|^2}{2\sigma^2(t)}\right),   
\end{equation}
satisfying $\sum_{n=1}^N \w_n(x,t) = 1$ for all $(x,t)$.
\end{theorem}

\begin{remark}
The empirical score \eqref{eq:empirical_score} 
has a singularity as $t \to 0$ because $\sigma(t) \to 0$ as $t \to 0^+.$ This singularity drives a particle weight degeneracy as $t \to 0^+$, as can be seen from \eqref{eq:normalized_Gaussian_weights}. 
The singularity and weight degeneracy underlie 
the memorization phenomenon that we study in 
Section \ref{sec:singularity_data_diffusion}.
\end{remark}

\begin{example} \label{ex:VE_empirical} For the variance exploding process in Example~\ref{ex:VE} with 
$\bigl(\beta(t),g(t)\bigr) = (0,2t)$,
we have $m(t)=1, \sigma^2(t)=t^2$ and hence 
\begin{equation} \label{eq:VE_empiricalscore}
\score^N(x,t) = -\frac{1}{t^2} \left(x - \sum_{n=1}^N  w_n(x,t) x_0^n \right),
\end{equation}
where the unnormalized weights are given by
$$\wtilde_n(x,t) 
= \exp\left(-\frac{|x-x_0^n|^2}{2 t^2}\right).$$
\end{example}

\begin{example} \label{ex:VP_empirical} For the variance preserving process 
in Example~\ref{ex:VP} with $\bigl(\beta(t),g(t)\bigr) = (1,1)$, we have
$$\score^N(x,t) = -\frac{1}{1 - e^{-t}} \left(x - e^{-t/2} \sum_{n=1}^N  w_n(x,t) x_0^n \right),$$
where the unnormalized weights are given by
$$ \wtilde_n(x,t) 
= \exp\left(-\frac{|x-e^{-t/2}x_0^n|^2}{2(1 - e^{-t})}\right).$$
\end{example}

\subsection{Score Matching In Practice}
\label{ssec:prac}

In practice, the loss $\J^N$ in \eqref{eq:score_matching_lossN} is not used as the basis of algorithms. Instead, a loss $\J_0^N$ is defined, 
which differs from $\J^N$ by a constant and use of which leads to easily
actionable algorithms. To understand the relationship between these losses we 
first define
\begin{subequations}
\label{eq:OFGS}
\begin{align}
\J_0(\score) &= \int_0^T \lambda(t) \mathbb{E}_{x_0 \sim \pr_0(\cdot)} \mathsf{G}(\score,x_0,t) \dd t \label{eq:denoising_score_matching_loss} \\ 
\mathsf{G}(\score,x_0,t) &= \mathbb{E}_{x \sim \pr(\cdot,t|\bu_0)}|\score(x,t) - \nabla_\bu \log\pr(\bu,t|\bu_0)|^2, \label{eq:Gscorematching_integrand}
\end{align}
\end{subequations}
and constant $K$ given by
\begin{align} \label{eq:objective_constant}
    K = &\int_0^T \lambda(t) \mathbb{E}_{x_0 \sim \pr_0(\cdot)}\mathbb{E}_{x \sim \pr(\cdot,t|\bu_0)} |\nabla_x \log \pr(\bu,t|\bu_0) |^2 \dd t \nonumber\\ 
    &\quad\quad\quad\quad\quad\quad\quad\quad\quad- \int_0^T \lambda(t) \mathbb{E}_{x \sim \pr(\cdot,t)} |\nabla_x \log \pr(\bu,t)|^2 \dd t.
\end{align}
Conditioning on the data distribution $\pr_0$, the following is a consequence
with proof given in  \Cref{ssec:wnc}:
\begin{proposition} \label{p:ifKf} 
Assume that constant $K$ is finite. Then
$$\J(\score) = \J_0(\score) - K.$$
Thus, the minimizers of $\J$ and $\J_0$ coincide and hence
the minimizer of $\J_0$ is given by $\score(x,t) = \nabla_x \log p(x,t)$.
\end{proposition}

\begin{remark}
    \label{rem:theK}
We briefly comment on the finiteness assumption for constant $K$ in Proposition~\ref{p:ifKf}. For the Gaussian conditional score in~\eqref{eq:conditional_score} arising from all forward processes~\eqref{eq:ForwardSDE_lineardrift}, %
the first term in $K$ is given by
\begin{align*}
\int_0^T \lambda(t) \mathbb{E}_{x_0 \sim \pr_0(\cdot)}\mathbb{E}_{x \sim \pr(\cdot,t|\bu_0)} \left|\frac{x - m(t)x_0}{\sigma^2(t)}\right|^2 \dd t 
= \int_0^T d\frac{\lambda(t)}{\sigma^2(t)} \dd t.
\end{align*}
One may choose the weighting function $\lambda$ so that $\lambda(t)/\sigma^2(t)$ is integrable. Note that this occurs for the choice $\lambda(t) \propto \sigma^2(t)$, but not for $\lambda(t) \propto g(t)$ considered in~\cite{song2021maximum}; taking $h = g$ in Lemma~\ref{lemma:IntegralBlowup} provides a proof of this result. 

For data distributions $p_0$ with Lebesgue density, the second term in K will typically also be integrable. As an example, for the Gaussian data distribution $p_0 = \mathcal{N}(0,C)$ with positive definite covariance matrix $C$ and the forward process in Example~\ref{ex:VP_empirical}, the score is $\nabla_x \log p(x,t) = -(C + \sigma^2(t)I_d)^{-1}x$, and the second term in $K$ is given by $$\int_0^T \lambda(t) \mathbb{E}_{x \sim p(\cdot,t)}|(C + \sigma^2(t)I_d)^{-1}x|^2 \dd t = \int_0^T \lambda(t) \Trace((C+\sigma^2(t)I_d)^{-1})\dd t.$$
The integrand behaves like $\lambda(t)Tr(C^{-1})$ as $t \rightarrow 0^+$ and will
be integrable for $\lambda(\cdot)$ continuous on $[0,T].$ %
 \end{remark}

Since we do not have $p_0$, we instead use the empirical approximation $p_0^N$ given by~\eqref{eq:empirical} to define the $\score$ objective. Making this substitution in~\eqref{eq:denoising_score_matching_loss} leads to the loss 
\begin{subequations} \label{eq:denoising_score_matching_lossN}
\begin{align}
\J_0^N(\score) %
&= \int_0^T \lambda(t) \mathbb{E}_{x_0 \sim \pr_0^N(\cdot)} \mathsf{G}(\score,x_0^n,t) \dd t, \\
\mathsf{G}(\score,x_0^n,t) &= \mathbb{E}_{x \sim \pr(\cdot,t|\bu_0)}|\score(x,t) - \nabla_\bu \log\pr(\bu,t|\bu_0^n)|^2.
\label{eq:denoising_score_matching_lossN_b}
\end{align}
\end{subequations}
From~\Cref{p:ifKf} with $\bigl(p_0(x), p(x,t)\bigr)$ replaced by $\bigl(p_0^N(x), p^N(x,t)\bigr)$, we see that $\J_0^N$ and $\J$ also differ by a constant.

\begin{remark}
    \label{rem:nopaths}
Writing the loss functions $\J_0$ and $\J_0^N$ using $\mathsf{G}$, leading
to expressions~\eqref{eq:denoising_score_matching_loss} and~\eqref{eq:denoising_score_matching_lossN_b}, respectively, shows that 
that $\J_0$ and $\J_0^N$ depend only on the {\it marginal} properties of the distribution for $x(t)$;
in particular sample paths of the forward process are not needed.
As a consequence, the expectations required to evaluate $\J_0^N(\score)$ can be computed using independent Gaussian samples at each time; furthermore, the time-dependent
Gaussian distributions for these samples are explicitly known.
In particular, note that $\mathsf{G}(s,x_0,\cdot)$ can be 
approximately evaluated at any time $t$ using
the fact that $\pr(\cdot,t|\bu_0)$ is Gaussian. Then $\J_0^N(\score)$ can be evaluated
approximately by averaging over the data distribution $p_0^N$ and over time $t.$ 
The loss function $\J_0^N$ thus provides the basis of actionable algorithms;
the time integration is usually performed by sampling uniformly at random in $[0,T].$
\end{remark}

\subsection{Learning The Noise}
\label{ssec:LTN}

Consider the loss function $\J_0$ in~\eqref{eq:denoising_score_matching_loss}, utilizing
formula~\eqref{eq:conditional_score} for the conditional score and making 
the choice $\lambda(t) = \sigma^2(t).$ Writing $\score(x,t) = \widetilde{\score}(x,t)/\sigma(t)$, where $\widetilde{\score} \colon \R^{d} \times [0,T] \rightarrow \R^d$, 
we find that minimizing $\J_0(\score)$ over $\score$ is equivalent to minimizing $\I_0(\widetilde{\score})$ 
over $\widetilde{\score}$ where the  
\textit{denoising objective} $\I_0(\cdot)$ is given by 
\begin{align*}
    \I_0(\widetilde{\score}) &= \int_0^T  \mathbb{E}_{x_0 \sim \pr_0(\cdot)} \mathsf{F}(\widetilde{\score},x_0,t) \dd t \\
    \mathsf{F}(\widetilde{\score},x_0,t) &= \mathbb{E}_{\eta \sim \mathcal{N}(0,I_d)} |\widetilde{\score}(m(t)x_0 + \sigma(t)\eta,t) + \eta|^2.
\end{align*}
We note that $\widetilde{\score}(\cdot,t)$ has the interpretation of predicting the noise added to a data sample at time $t$. This noise prediction is used to recover $x_0$ from a noisy sample. The fact that we can work in this denoising fashion stems from the fact that the original loss $\J_0(\cdot)$ does not depend on pathwise properties of the forward process -- see Remark \ref{rem:nopaths}.

In practice $\I_0(\widetilde{\score})$ is approximated empirically to obtain $\I_0^N(\widetilde{\score}),$ replacing $p_0$ by
$p_0^N.$ This objective is often easier to optimize than is $\J_0^N,$  because the exact minimizer $\widetilde{\score}$ does not blow-up as $t \to 0^+.$ There is, however, a trade-off: the implied score function $\score$ is forced to have singular behavior near $t = 0$, since  $\sigma(t) \to 0$ as $t \to 0^+.$ While having an unbounded score function may be beneficial for distributions whose support is concentrated on low-dimensional manifolds, it will typically result in a poor approximation of the score function for data distributions with Lebesgue density. Indeed, blow-up of the score is entwined with the memorization phenomenon studied analytically in the next section. Experiments
presented in Subsection \ref{sec:2DParametricScore} argue in favour
of using $\J_0^N$, when the data is thought to come from a distribution with Lebesgue density, rather than $\I_0^N.$

\section{Theory of Memorization} \label{sec:singularity_data_diffusion} 

Memorization is a consequence of the singular behavior of the score
resulting from the Gaussian mixture minimizer $\score^N(x,t)$, of $\J_0^N(\cdot)$ and $\J^N(\cdot)$,
given in Theorem \ref{thm:empirical_score_unconditional}.
The fact that $\sigma(t) \to 0$ as $t \to 0^+$ leads to weights $w_n$ which typically
concentrate on one data point, causing memorization. As a step towards principled regularization, we analyze this memorization effect in detail. 
We study memorization for the reverse ODE
\begin{equation} \label{eq:reverseODE_empirical}
    \frac{\dd \bu}{\dd t} = -\frac12 \beta(t)\bu -\frac{g(t)}{2} \score^N(\bu,t), \quad \bu(T) = \bu_T.
\end{equation} 
This is obtained from \eqref{eq:BackwardSDE2} with 
$\alpha_2=0$, and $\score$ replaced by $\score^N$ from 
Theorem \ref{thm:empirical_score_unconditional}.

In Subsection~\ref{ssec:SUMT} we set-up notation and present the memorization result in Theorem~\ref{thm:memorization}; Subsection~\ref{ssec:disc} contains discussion
of several observations about the theorem and its proof. 
Because the analysis and numerical experiments for the variance exploding
and variance preserving cases are similar, we confine the material in this
section to the (slightly simpler) variance exploding case. Details for
the variance preserving case are given in Appendix~\ref{a:C}.
We prove the theorem for the variance exploding process in Subsection~\ref{ssec:VEanalysis}, 
and for the variance preserving process in Appendix~\ref{ssec:VPanalysis}.
Numerical experiments illustrating these two settings are contained in
Subsection~\ref{sec:numerics_VE} and Appendix~\ref{ssec:numerics_VP} for the variance
exploding and preserving processes, respectively.

We make the following standing assumption about the i.i.d.\thinspace data; assuming that the points are distinct will be useful for the memorization theorems proved in this section and it is true with probability one
for any $p_0$ that has density with respect to the Lebesgue measure.
\begin{assumption}
    \label{asp:stand}
    The data $\{x_0^n\}_{n=1}^N$ is chosen i.i.d.\thinspace from $p_0.$ Furthermore, this set
    comprises $N$ distinct points in $\R^d$; the minimum distance between data points satisfies
    \begin{equation}
        \label{eq:sep}
        \Dm \coloneqq \min_{\ell \ne m \in \{1, \dots, n\}} |x^\ell_0 - x^m_0| >0.
    \end{equation}
\end{assumption}
The next assumption, which leads to simplifications in Lemma~\ref{lemma:SDEsoln} concerning marginal Gaussian
laws, is also useful in the theory of memorization. To interpret the assumption,
recall the definitions in Lemma \ref{lemma:SDEsoln}.
\begin{assumption} \label{asp:abg}  
Functions $\beta \in C([0,T],\mathbb{R}_{\geq 0})$ 
and $g \in C([0,T],\mathbb{R}_{\geq 0})$ are chosen so that: 
(i) either $\beta(\cdot) \equiv 0$ (variance exploding) or
$\beta(\cdot) \equiv g(\cdot)$ (variance preserving); and
(ii) the conditional variance of the forward process, $\sigma^2(t)$, 
is invertible on $(0,T]$.
\end{assumption}

\subsection{Set-Up and Main Theorem} \label{ssec:SUMT}

Equation \eqref{eq:reverseODE_empirical}
is integrated from $x_T \sim \mathfrak{g}_T$,
a pre-specified Gaussian, with specific choice of $\mathfrak{g}_T$ 
depending on $(\beta,g).$
Then $x(0)$ can be used as an approximate sample from the distribution
$p_0$ underlying the data, as described in Section~\ref{ssec:AP}. Theorem~\ref{thm:memorization} states that the 
limit points of~\eqref{eq:reverseODE_empirical} 
are confined to either the training data points, thereby resulting in \textit{memorization},
or to the hyperplanes between the data; furthermore, numerical
evidence demonstrates that limit points on the hyperplanes are a measure zero event with respect to the random draw from $\mathfrak{g}_T,$ so that memorization is to be expected. 

To state and prove Theorem~\ref{thm:memorization}  
we introduce definitions concerning the Voronoi structure 
of the data and define the limit point precisely. Recall the standing Assumption~\ref{asp:stand}. For each training data point $x^n_0$, we define the \textit{Voronoi cell} 
\begin{equation}
\label{eq:VC}
V(x^n_0) \coloneqq \{x \in \R^d: |x - x^n_0|^2 < |x - x^\ell_0|^2 \;\forall\; \ell \neq n\}.
\end{equation}
The collection of Voronoi cells for all data points $V \coloneqq \{V(x_0^1),\dots,V(x_0^N)\}$ is known as a \textit{Voronoi diagram} or \textit{Voronoi tessellation}. The diagram defines a partitioning of $\R^d$ into (possibly unbounded) convex polytopes $V(x^n_0)$; each
polytope $V(x^n_0)$ contains the points in $\R^d$ that are closer to $x^n_0$ 
in Euclidean distance than to any other data point. 
The boundary of the Voronoi cell for $x^n_0$ is defined as
\begin{equation}
\label{eq:VB}
\partial V(x^n_0) \coloneqq \{x \in \R^d: \text{ for some } \ell\neq n\quad 
|x - x^n_0|^2 = |x - x^\ell_0|^2\}.
\end{equation} 
Note that $\overline{V(x^n_0)}=V(x^n_0) \cup \partial V(x^n_0)$ and that
the union of $\overline{V(x^n_0)}$ over all $n \in \{1,\cdots,N\}$ 
is the whole of $\R^d.$ Let $B(0,r)$ denote an open centered ball in the Euclidean norm. 
Given finite positive $r$ we localize the Voronoi diagram, and the boundary set, by defining
\begin{subequations}
\begin{align} 
V^r(x^n_0) &\coloneqq V(x^n_0) \cap B(0,r) ,\\
\partial V^r(x^n_0) &\coloneqq \partial V(x^n_0) \cap B(0,r),\\
    V^r &\coloneqq \cup_{n=1}^N V(x^n_0) \cap B(0,r)=\cup_{n=1}^N V^r(x^n_0), \\
    \partial V^r &\coloneqq \cup_{n=1}^N \partial V(x^n_0) \cap B(0,r)=\cup_{n=1}^N \partial V^r(x^n_0). 
\end{align}
\end{subequations}
It follows that
\begin{equation}
\label{eq:addcite22}
     V^r \cup \partial V^r= \cup_{n=1}^N \overline{V(x^n_0)} \cap B(0,r)=B(0,r).
\end{equation}

We define a \textit{limit point} $x^\star$ of the 
ODE~\eqref{eq:reverseODE_empirical} to be a point 
for which there is a decreasing sequence of positive times
$\{t_k\}_{k \in \N}$, with $t_k \to 0^+$ as $k \to \infty$, 
such that $x(t_k) \to x^\star$ as $k \to \infty.$
We then have the following theorem, proved separately for the
variance exploding and variance preserving processes, in Subsections~\ref{ssec:VEanalysis}
and~\ref{ssec:VPanalysis}, respectively.

\begin{theorem} \label{thm:memorization} Consider ODE~\eqref{eq:reverseODE_empirical} under  Assumptions~\ref{asp:stand} and~\ref{asp:abg}.
Then, given any final time condition $x(T)=x_T$,
ODE~\eqref{eq:reverseODE_empirical} has a unique solution for $t 
\in (0,T].$ Furthermore, there exists $r > 0$ such that, for any point $x_T \in \R^d$ and resulting
solution $x(t)$ of ODE~\eqref{eq:reverseODE_empirical} initialized at 
$x(T) = x_T$, there is $t^\ast = t^\ast(r,x_T)$ such that $x(t) 
\in B(0,r)$ for $t \in (0,t^\ast).$ Any limit point $x^\star$ of the trajectory is either in $V^r$ or on $\partial V^r$. If $x^\star \in V^r$, then it is one of the data points $\{x_0^n\}_{n=1}^N$ and the entire 
solution converges and does so at a rate equal to the conditional standard deviation of the forward process:
\begin{align} \label{eq:ConvergenceRate_original}
    |x(t) - x^\star| &\lesssim \sigma(t), \quad t \rightarrow 0^+.
\end{align}
\end{theorem}

\subsection{Discussion of Main Theorem}
\label{ssec:disc}

The existence of limit points on the boundary $\partial V^r$ of the Voronoi tessellation is non-vacuous, as the next example shows. However, we conjecture that with probability one with respect to $x_T$ drawn from a Gaussian, the limit
points are the data points; numerical experiments consistent with this conjecture
may be found in Subsection~\ref{sec:numerics_VE} and Appendix~\ref{ssec:numerics_VP}.

\begin{example}
\label{ex:refer_equidistant_data}
Consider the variance exploding process in Example~\ref{ex:VE_empirical}, with $N = 2$ points in dimension $d=2$. Assume that the two data points $x_0^1, x_0^2$ 
are on the horizontal axis and equidistant from the origin, i.e., $x_0^1 = -x_0^2$. 
Then ODE~\eqref{eq:reverseODE_empirical} has the form
\begin{align}
\label{eq:odeex}
\frac{\dd x}{\dd t} 
&= \frac{1}{t}\Bigl(x - x_0^1\bigl(\w_1(x,t) - \w_2(x,t)\bigr)\Bigr),
\end{align}
with unnormalized weights given by \eqref{eq:normalized_Gaussian_weights} as
\begin{equation} \label{eq:normalized_Gaussian_weights2}
\wtilde_n(x,t) = \exp\left(-\frac{|x - x_0^n|^2}{2t^2}\right),
\quad n \in \{1,2\}.   
\end{equation}
Now consider an initial condition $x(T)$ on the vertical axis noting that, on this axis, the two weights are equal: $w_1(x,T) = w_2(x,T) = \frac{1}{2}.$
Then, the vertical axis is invariant for the dynamics. The
equation~\eqref{eq:odeex} simplifies to 
$$\frac{\dd x}{\dd t} = \frac{1}{t}x,$$ 
whose solution satisfies $|x(t) - x_0^1| = |x(t) - x_0^2|$ for all $t \le T$
and the weights are equal for all time: $w_1(x,t) = w_2(x,t) = \frac{1}{2}.$
In particular, the solution is $x(t)=tx(T)/T$, which is the unique solution of~\eqref{eq:odeex} for initializations on the vertical axis. 
Moreover, the trajectory converges to limit point $x^\star=0$ at the origin. 
This is not a data point, but rather a point on 
$\partial V(x_0^1)=\partial V(x_0^2).$ However, initializations off the vertical axis
all converge to one of the data points, as shown in the numerical illustration of this example provided in Section~\ref{sec:numerics_VE}; see Figure~\ref{eq:ODE_twopoints} for a two-dimensional example with $x_0^1 = (1,0)$ and $x_0^2 = (-1,0)$. 
\end{example}

We now define some relaxations of the Voronoi cell definitions. These relaxations
are integral to the proofs underlying Theorem \ref{thm:memorization}; they are not
needed, however, for its statement. Given $\varepsilon>0$ we let $V_{\varepsilon}(x^n_0)$ denote the set of points in cell $V(x^n_0)$ 
satisfying
$$V_{\varepsilon}(x^n_0) = \{x \in V(x^n_0): |x - x^n_0|^2 < |x - x^\ell_0|^2 -
 \varepsilon^2 \;\forall\;\ell \neq n\}.$$
This may be interpreted as the set of points in cell $V(x^n_0)$ separated 
from the cell boundary by distance at least $\varepsilon$. 
We also define the localized cell interior
$$V_{\epsilon}^{r}(x^n_0) \coloneqq V_{\epsilon}(x^n_0) \cap B(0,r).$$
We define the boundaries of these two sets by 
\begin{align*}
\partial V_{\varepsilon}(x^n_0) &\coloneqq \{y \in V(x^n_0): \text{ for some } \ell\neq n\quad |y - x^n_0|^2 = |y - x^\ell_0|^2 - \varepsilon^2\},\\
\partial V_{\varepsilon}^{r}(x^n_0) &\coloneqq \partial V_{\varepsilon}(x^n_0) \cap B(0,r).
\end{align*}
We will consider these sets only for $\varepsilon<\Dm.$ 
We note the nesting property 
\begin{equation}
\label{eq:nest}
\varepsilon>\varepsilon' \Rightarrow V_{\varepsilon}(x^n_0) \subset V_{\varepsilon'}(x^n_0).
\end{equation}
It is also useful to define $L_\varepsilon(\cdot;x^n_0): \partial V_{\varepsilon}(x^n_0) \to 
{\pi}(\{1,\dots,N\})$, where $\pi(\cdot)$ computes the set of all subsets, by
\begin{equation}
    \label{eq:Ld}
    L_\varepsilon(x;x^n_0) \coloneqq \{\ell \in \{1,\dots,N\}: |x - x^n_0|^2 = |x - x^\ell_0|^2 - \varepsilon^2\}.
\end{equation}

A key lemma used in the proof of Theorem \ref{thm:memorization}, in both the
variance exploding and variance preserving cases, is the following characterization
of the components that make up the boundary set $\partial V_{\varepsilon}(x^n_0).$
The proof may be found in Appendix~\ref{app:proofs}.

\begin{lemma}
    \label{l:NFS}
    Let $\varepsilon<\Dm.$ Any point $y$ where
    \begin{equation}
        \label{eq:NFS55}
        |y - x^n_0|^2 = |y - x^\ell_0|^2 - \varepsilon^2
    \end{equation}
defines a hyperplane dividing $\R^d$ into two sets, one containing $x_0^n$ and the other
containing $x_0^\ell.$ Any point $y$ satisfying \eqref{eq:NFS55} also satisfies the identity
\begin{equation}
    \label{eq:NFS99}
    2 \langle x^n_0-y, x^n_0-x^\ell_0 \rangle =|x^n_0-x^\ell_0|^2-\varepsilon^2>0.
\end{equation}
Given any $y$ on this hyperplane, for any point $x$ satisfying
\begin{equation}
    \label{eq:NFS11}
    \langle x-y, x^n_0-x^\ell_0 \rangle >0,
\end{equation}
the vector field $x-y$ is pointing into the set containing $x^n_0.$
\end{lemma}

\subsection{Analysis: Variance Exploding Process} \label{ssec:VEanalysis}

Recall \Cref{lemma:SDEsoln} and Theorem~\ref{thm:empirical_score_unconditional} in the variance exploding setting where $\beta(t)=0.$ Since $m(t) \equiv 1$ in this case, the normalized weights are given by
\begin{equation} \label{eq:normalized_Gaussian_weightsVE}
  \w_n(x,t)  = \frac{\wtilde_n(x,t)}{\sum_{\ell=1}^N \wtilde_\ell(x,t)}, \quad \wtilde_n(x,t) = \exp\left(-\frac{|x - x_0^n|^2}{2\sigma^2(t)}\right),\quad
\sigma^2(t) \coloneqq  \int_0^t g(s) \dd s.
\end{equation}
We define $\xs_N(x,t)$ to be the space-time dependent convex combination of the data points
\begin{equation} \label{eq:convex_combination_data}
  \xs_N(x,t) = \sum_{n=1}^N x_0^n \w_n(x,t). 
\end{equation}
Theorem~\ref{thm:empirical_score_unconditional} then shows that
\begin{equation}
\label{eq:SNN}
    \score^N(x,t)=-\frac{1}{\sigma^2(t)}\sum_{n=1}^N (x-x_0^n) \w_n(x,t). 
\end{equation}
Using~\eqref{eq:convex_combination_data}, and the fact that weights sum to one,
this may also be written as
\begin{equation}
\label{eq:SNNalt}
    \score^N(x,t)=-\frac{1}{\sigma^2(t)}\left(x-\xs_N(x,t)\right).
\end{equation}
The reverse ODE for the variance exploding process, found from~\eqref{eq:reverseODE_empirical} in the
case $\beta(t) \equiv 0$, takes the explicit form
\begin{align}
\label{eq:mf_ode_emp}
    \frac{\dd x}{\dd t} &=\frac{g(t)}{2\sigma^2(t)}\bigl(x-\xs_N(x,t)\bigr),
    \quad x(T)=x_T.
\end{align}
Recall that the generative model uses $x_T \sim \mathfrak{g}_T \coloneqq N(0,\sigma^2(T)I_d)$;
however, Theorem \ref{thm:memorization} concerns behavior for any given $x_T.$ 
We solve the ODE \eqref{eq:mf_ode_emp} starting from $t = T$ sufficiently large,
backwards-in-time to $x(0)$. We show that this leads to memorization.

\begin{remark}
For any variance exploding process with a Taylor series expansion that satisfies $g(t) = \Theta(t^p)$ as 
$t \rightarrow 0^+$, Lemma~\ref{lemma:IntegralBlowup} shows that $g(t)/(2\sigma^2(t))$ in~\eqref{eq:mf_ode_emp} equals $(p+1)/(2t) + o(1/t)$ as 
$t \rightarrow 0^+$. The resulting singular drift is what drives memorization---return of the reverse ODE to the data. 
For example, if $g(t) = 2t$, the ODE in~\eqref{eq:mf_ode_emp} has the form  
$$\frac{\dd x}{\dd t} = \frac{1}{t}\bigl(x-\xs_N(x,t)\bigr).$$
\end{remark}

To analyze equation \eqref{eq:mf_ode_emp} 
it is convenient to define time transformation $t \mapsto s$ as
\begin{equation}
\label{eq:tt}
s=\ti (t;c)= -\frac{1}{2}\ln \bigl(\sigma^2(t)+c\bigr)  
\end{equation}
for any $c \ge 0.$ Note that, using~\eqref{eq:msig_variance} with $m \equiv 1$, we have
\begin{equation}
\label{eq:tt2}
\frac{\dd}{\dd t}\ti (t;c)= -\frac{1}{2}\frac{g(t)}{\sigma^2(t)+c}.  
\end{equation}
Here we use $c=0$, but other values of $c$ will be useful
in Section~\ref{sec:regularization}.

For $c=0$ we have $s=-\frac{1}{2}\ln \sigma^2(t)$. 
From Assumption~\ref{asp:abg}, $\sigma^2$ is invertible as a mapping
taking $(0,T]$ into $(0, \sigma^2(T)].$ Thus, if we define the mapping
$s \mapsto t$ by $t = (\sigma^2)^{-1}(e^{-2s})$, then $s=-\frac{1}{2}\ln \sigma^2(t)$ is
well-defined as a mapping taking 
the time-reverse of interval $(0,T]$ into $[-\ln \sigma(T), \infty).$ We define the transformed variable and convex combination of the data points
$$y(s)=x\bigl(t\bigr), \quad \ys_N(y,s)=\xs_N\bigl(y,t\bigr).$$
Thus, recalling the definition of the weights in \eqref{eq:normalized_Gaussian_weightsVE} and of
$\xs_N$ in \eqref{eq:convex_combination_data}, we obtain
\begin{equation}
\ys_N(y,s)=\sum_{n=1}^N \ww_n(y,s)x_0^n, \quad \ww_n(y,s)=\w_n
\Bigl(y,\bigl(\ti(\cdot;0)\bigr)^{-1}(s)\Bigr),
\end{equation}
and hence
\begin{equation} \label{eq:Snormalized_Gaussian_weights}
  \ww_n(y,s) = \frac{\wwtilde_n(y,s)}{\sum_{\ell=1}^N \wwtilde_\ell(y,s)}, \qquad \wwtilde_n(y,s) = \exp\left(-\frac{|y - x_0^n|^2}{2e^{-2s}}\right).   
\end{equation}
It follows from~\eqref{eq:mf_ode_emp} and~\eqref{eq:tt2} that
\begin{equation}
\label{eq:mf_ode4}
    \frac{dy}{ds}=-\bigl(y-\ys_N(y,s)\bigr), \quad y(s_0)=x_T,
\end{equation}
where we solve this equation starting from 
$s=s_0\coloneqq-\ln \sigma(T),$ and integrate to $s=\infty.$ This corresponds to solving~\eqref{eq:mf_ode_emp} backwards from $t=T$ down to $t=0^+.$

\begin{remark}
\label{rem:idea}
The essence of the proof of Theorem~\ref{thm:memorization}, recast in terms of the
time variable $s$ rather than $t$, is that for all large enough $s$, $\ys_N(y,s) \approx x_0^n$ for some $n$. This follows from collapse of almost all normalized weights to zero for large $s$, as can be deduced from \eqref{eq:Snormalized_Gaussian_weights}; and except in some special cases, indeed only one weight will remain non-zero and will hence approach one as the other weights converge to zero. Hence, from \eqref{eq:mf_ode4} we see that the evolution for all large enough $s$ is approximately given by
\begin{equation}
\label{eq:mf_ode4A}
    \frac{dy}{ds} \approx -\bigl(y-x_0^n\bigr).
\end{equation}
Integrating this equation to $s=\infty$ gives $y(s) \to x_0^n$, i.e., memorization of the data. Furthermore, the rate of convergence is universal as stated in
Theorem~\ref{thm:memorization}.
\end{remark}

We now state an equivalent form of Theorem~\ref{thm:memorization}, specialized to the variance exploding case, and written in the transformed time variable $s$ instead of $t.$
Indeed, the proof of this restated and specialized version of Theorem~\ref{thm:memorization} is undertaken in the $s$ variable. The proof relies on a set of lemmas for equation~\eqref{eq:mf_ode4} that characterize its limit points and invariant sets. The remainder of this section provides the proof outline, based
around statements of these lemmas; their proofs are found in Appendix~\ref{app:proofs}.

We define a \textit{limit point} $y^\star$ of the 
ODE~\eqref{eq:mf_ode4} to be a point 
for which there is an increasing sequence of positive times
$\{s_k\}_{k \in \N}$, with $s_k \to \infty$ as $k \to \infty$, 
such that $y(s_k) \to y^\star$ as $k \to \infty.$

\begin{theorem_restated}[\textbf{Restatement of 
Theorem~\ref{thm:memorization} (Transformed Time, Variance Exploding Case)}] 
Let Assumptions \ref{asp:stand} and \ref{asp:abg} hold.
Then, given any initial condition $x_T$,
ODE~\eqref{eq:mf_ode4} has a unique solution defined for 
$s \in [s_0,\infty).$ Furthermore,
there exists $r > 0$ such that, for any point $x_T \in \R^d$ and resulting
solution $y(\cdot)$ of ODE~\eqref{eq:mf_ode4} initialized at 
$y(s_0) = x_T$, there is $s^\ast \coloneqq s^\ast(r,x_T)$ such that $y(s) 
\in B(0,r)$ for $s \ge s^\ast.$
Any limit point $y^\star$ of the trajectory is either on $\partial V^r$ or in $V^r$. If $y^\star \in V^r$, then it is one of the data points $\{x_0^n\}_{n=1}^N$ and the entire 
solution converges and does so at a rate $e^{-s}$:
\begin{align} \label{eq:ConvergenceRate_transformed}
    |y(s) - y^\star| &\lesssim e^{-s}, \quad s \rightarrow \infty.
\end{align}
\end{theorem_restated}

\begin{proof}[Proof of Theorem~\ref{thm:memorization}: Variance Exploding Case] By Lemma~\ref{thm:existence_uniqueness}, the solution of equation~\eqref{eq:mf_ode4} initialized at
$y(s_0)=x_T$, 
has a unique solution for $s \in [-\ln \sigma(T),\infty).$
Lemma~\ref{lemma:limiting_points} shows that the solution is bounded for all $s \geq 0$ and its $\limsup$ is contained in a ball of radius 
\begin{equation}
\label{eq:boundd}
|x|_\infty \coloneqq \max\{|x_0^1|,\dots,|x_0^N|\}.
\end{equation} 
Thus, for any $r > |x|_\infty$, there is $s^\ast=s^\ast(r,x_T) \ge 0$
such that $y(s) \in B(0,r)$ for all $s \geq s^\ast$. By the
relative compactness of $B(0,r)$ in $\R^d$, it follows that limit points $y^\star$ in the closure of $B(0,r)$ exist. In fact, since
$r>|x|_\infty$ is arbitrary, we deduce that $y^\star \in B(0,r)$; and then, by \eqref{eq:addcite22}, such limit points lie either on $\partial V^r$ or in $V^r$. If the limit point is on $\partial V^r$ then the proof is complete; we thus assume henceforth that it is in $V^r$.

For any limit point $y^\star \in V^r$, there exists data point $x_0^n$ 
and $\delta \in (0,\Dm)$ such that $y^\star \in V_{2\delta}^{r}(x^n_0).$ 
By definition of the limit point there is a sequence $\{s_j\} \to \infty$ 
such that $y(s_j) \to y^\star$. Furthermore, there is 
some $J \in \mathbb{N}$ such that $y(s_j) \in V_{\delta}^{r}(x^n_0)$ for all $j \geq J;$ this follows from the nesting property \eqref{eq:nest}.

In Lemma \ref{lem:weight_bounds} we obtain bounds on the
weights $\ww(y,s)$ for $y(s)$ within $V_{\delta}^{r}(x^n_0).$ 
Using these bounds we show that the vector field in the
defining equation~\eqref{eq:mf_ode4} for $y$ points into the interior of 
$V_{\delta}^{r}(x^n_0)$ from its boundary, for all $s \ge s_J$, possibly by increasing $J$; the essence of the proof is that, for large enough $s$, $\ys_N(y,s) \approx x_0^n$ on $\partial V_{\delta}^{r}(x^n_0)$---see \Cref{rem:idea} and then Lemma \ref{lemma:boundary_intersection} where weight collapse (all but the weight on $x^n_0$ are close to zero for large enough $s$) is employed.  
By Lemma~\ref{lem:invariant_sets} we deduce that $y(s) \in V_{\delta}^{r}(x^n_0)$ 
for all $s \geq s_J$. Again using the fact that $\ys_N(y,s) \approx x_0^n$ in $V_{\delta}^{r}(x^n_0)$, the dynamics of ODE~\eqref{eq:mf_ode4} are then approximated
by \Cref{eq:mf_ode4A}. Using a quantitative version of this approximation, Lemma~\ref{cor:exponential_convergence_newcoords} establishes that $y(s) \rightarrow x^n_0$ as $s \rightarrow \infty$ 
and the convergence is at the desired rate $e^{-s}$.
\end{proof}

\paragraph{Existence, Uniqueness and Bounds on the Solution.} We first show the ODE in~\eqref{eq:mf_ode4} has a unique solution. We consider the setting where $x$ and the $\{x_0^n\}_{n=1}^N$ all lie in Hilbert space 
$\bigl(H, \la \cdot, \cdot \ra, \|\cdot\|\bigr)$, with norm $\|\cdot\|$ induced by the inner-product. The Euclidean norm 
$|\cdot |$ is replaced by the 
Hilbert space norm $\|\cdot\|$
within the definition of normalized weights $w_n(\cdot,\cdot).$

\begin{lemma} \label{thm:existence_uniqueness} 
Let Assumptions \ref{asp:stand} and \ref{asp:abg} hold.
For all $x_T \in H$ equation \eqref{eq:mf_ode4} has a unique solution $y \in C^1\bigl([s_0,\infty);H\bigr).$ Furthermore,
\begin{subequations}
\label{eq:DistanceLimSup}
\begin{align}
    \|y(s)\| &\leq \max\{\|x_T\|,\|x_0^1\|,\dots,\|x_0^N\|\}, \; \forall s \geq 0,\\
  \limsup_{s \rightarrow \infty} \|y(s)\| &\leq \max\{\|x_0^1\|,\dots,\|x_0^N\|\}.  
\end{align}
\end{subequations}
\end{lemma}

The following lemma provides a partial characterization of such limit points
$x^\star.$ It identifies a bounded set, defined by the training data, to which
any limit point $x^\star$ must be confined. Recall definition \eqref{eq:boundd}.

\begin{lemma} \label{lemma:limiting_points}
Let Assumptions \ref{asp:stand} and \ref{asp:abg} hold. Then,
\begin{itemize}
\item Any limit point $y^\star$ of the dynamics defined by ODE~\eqref{eq:mf_ode4} is contained in the closure
of the set $B(0,|x|_\infty)$;
\item For any trajectory of ODE~\eqref{eq:mf_ode4} and any $r > |x|_\infty,$ there exists a time $s^r=s^r(x_T,s_0) > s_0$ such that $y(s) \in B(0,r)$ for $s \geq s^r$.
\end{itemize}
\end{lemma}

\paragraph{Invariant Set for Dynamics.} 

As part of establishing an invariant set for the dynamics and studying dynamics within it,
we start by obtaining a bound on the weights when evaluated
in $V_{\delta}(x^n_0)$. We do this for any data point $x_0^n$ and any $\delta \in (0,\Dm).$

\begin{lemma} \label{lem:weight_bounds} 
Let Assumptions \ref{asp:stand} and \ref{asp:abg} hold.
For any $\delta \in (0,\Dm)$ and for all $y \in V_{\delta}(x^n_0)$ and $s \geq s_0$, the normalized weight $\ww_n$ satisfies 
\[\frac{1}{1 + (N-1)\exp\left(-\frac{e^{2s}\delta^2}{2}\right)} < \ww_n(y,s) \leq 1,\]
while all other weights, for $l \neq n$, satisfy 
\[\qquad 0 \leq \ww_\ell(y,s) < \exp\left(-\frac{e^{2s}\delta^2}{2}\right).\]
\end{lemma}

 The following lemma shows that, for all sufficiently large times $s$, the vector field defining ODE \eqref{eq:mf_ode4} points
 inward, towards $x^n_0$, on $\partial V_{\delta}^{r}(x^n_0).$
 The essence of the proof is that, for large enough $s$, 
$\ys_N(y,s) \approx x_0^n$ on $\partial V_{\delta}^{r}(x^n_0)$.
To understand the statement and proof, recall definition~\eqref{eq:Ld} of function $L_{\epsilon}(\cdot;x^n_0)$, returning the indices of neighboring Voronoi cells. Furthermore, define
the maximum distance between data points
\begin{equation}
\label{eq:DD}
D^+ \coloneqq \max_{\ell,m \in \{1, \dots, n\}} |x^\ell_0 - x^m_0|.
\end{equation}

\begin{lemma} \label{lemma:boundary_intersection}
Let Assumptions \ref{asp:stand} and \ref{asp:abg} hold.
Fix any $\delta \in (0,\Dm)$ and $r > |x|_\infty$. Define $\alpha_n>0$ by
\begin{equation} \label{eq:minimum_boundary_angle}
\alpha_n\coloneqq \inf_{y \in \partial V_{\delta}^r(x^n_0)} \min_{\ell \in L_\delta(y;x^n_0)} \langle x^n_0 - y, x^n_0 - x^\ell_0 \rangle. 
\end{equation}
Choose $s^\alpha > s_0$ so that 
\begin{equation} \label{eq:s_constraint}
(N-1) \exp \left(-\frac{e^{2s^\alpha} \delta^2}{2} \right) 
|x|_\infty < \frac{\alpha_n}{4D^+}.
\end{equation}
Then, for all $s \geq s^\alpha$, the vector field for ODE~\eqref{eq:mf_ode4} satisfies
\begin{equation}
\label{eq:NFS66}
\inf_{y \in \partial V_{\delta}^{r}(x^n_0)} \min_{\ell \in L_\delta(y,x^n_0)} \langle \ys_N(y,s) - y, x^n_0 - x^\ell_0 \rangle \geq \frac{\alpha_n}{2} > 0.
\end{equation}
\end{lemma}

\paragraph{Invariace of, and Convergence Within, $V_{\delta}^{r}(x^n_0)$.}

Consider any trajectory with limit point contained in the set of data.
A consequence of the preceding lemma is that there is data point $x_0^n$, depending on initial
condition $x_T$, such that, after a sufficiently long time, 
any solution $y(s)$ enters, and does not leave, the constrained Voronoi cell $V_{\delta}^{r}(x^n_0)$.

\begin{lemma} \label{lem:invariant_sets} 
Let Assumptions \ref{asp:stand} and \ref{asp:abg} hold and let $y(s)$ solve equation~ODE~\eqref{eq:mf_ode4}.
Fix any $\delta \in (0,\Dm)$ and $r > |x|_\infty$. If $y(s^\ast) \in V_{\delta}^{r}(x^n_0),$ for some $x^n_0$ and $s^\ast > \max(s^r,s^\alpha)$ for $s^r$ as in Lemma~\ref{lemma:limiting_points} and $s^\alpha$ satisfying~\eqref{eq:s_constraint}, then $y(s) \in V_{\delta}^{r}(x^n_0)$ for all $s \geq s^\ast$.
\end{lemma}
\begin{proof}
This follows from Lemma~\ref{l:NFS}, using inequality~\eqref{eq:NFS66}.
\end{proof}

Thus, $\{V_{\delta}^{r}(x_0^1),\dots,V_\delta^{r}(x_N)\}$ are invariant sets for dynamics: after a sufficiently long time, if the trajectory is contained within a given Voronoi cell, then the solution of the ODE will remain there for all time and indeed remain separated
from the boundary by a specified amount.
When the trajectory of the ODE in~\eqref{eq:mf_ode4} is contained strictly in the interior of a cell, the following lemma shows that the solution in fact converges exponentially fast to the cell center in $s$.

\begin{lemma} \label{cor:exponential_convergence_newcoords} 
Let Assumptions \ref{asp:stand} and \ref{asp:abg} hold and let $y(s)$ solve ODE~\eqref{eq:mf_ode4}.
Fix any $\delta \in (0,\Dm)$ and $r > |x|_\infty$. Recall $s^r$ as in Lemma~\ref{lemma:limiting_points} and $s^\alpha$ satisfying~\eqref{eq:s_constraint}. 
If $y(s^\ast) \in V_{\delta}^r(x^n_0)$ for some $x_0^n$ and $s^\ast > \max(s^r,s^\alpha)$,  
then for all $s \geq s^\ast$ we have
\begin{subequations}
\begin{align}
|y(s) - x^n_0| &\leq Ke^{-s}, \\
            K &= |y(s_0) - x_0^n|e^{s_0} + \frac{2(N-1)|x|_{\infty}}{\delta^2}e^{-s_0}e^{-e^{2s_0}\delta/2}.
\end{align}
\end{subequations}
\end{lemma}

\subsection{Numerics: Variance Exploding Process} 
\label{sec:numerics_VE}

In this section we provide numerical demonstrations that illustrate the behavior of the ODE with the optimal empirical score function. For these experiments we consider the specific variance exploding process from Example~\ref{ex:VE}, with $g(t) = 10^t$
and $T=1.$ Figure~\ref{fig:random_Voronoi_teselation} plots the dynamics of the reverse ODE in~\eqref{eq:mf_ode_emp} given $N = 20$ i.i.d.\thinspace samples $\{x_0^n\}_{n=1}^N$ (in blue) drawn from the two-dimensional standard Gaussian distribution $p_0 =\mathcal{N}(0,I_2)$. The trajectories of the ODE solution $x(t)$ for $0 < t \leq T = 1$ starting from four independent initial conditions $x(T) \sim \mathcal{N}(0,\sigma^2(T)I_2)$ are plotted in red. In each plot, the Voronoi tessellation for the samples is shown in black. 

We make a few observations. First, as predicted by Theorem~\ref{thm:memorization} we see memorization: indeed all four trajectories
converge to a limit which is from the empirical data distribution $p_0^N$. The dynamics are 
affected by the data points in a complicated initial stage, getting pulled 
toward the bulk of the data. %
However, after a sufficiently large time, the weight 
from a single data point corresponding to the Voronoi cell where the dynamics are contained leads to 
exponential convergence toward a data point. This can cause a change in the direction of the dynamics. Figure~\ref{fig:random_exponential_convergence} displays the convergence behavior of 30 independent trajectories as measured in terms of the Euclidean distance between the ODE solution and the limit point of each trajectory. The convergence matches the expected rates from Theorem~\ref{thm:memorization} in both the transformed time variable $s$ and original time variable $t$; in particular, this rate is independent of the data distribution and the initial condition.

\begin{figure}[!ht]
\centering
\includegraphics[width=0.48\textwidth]{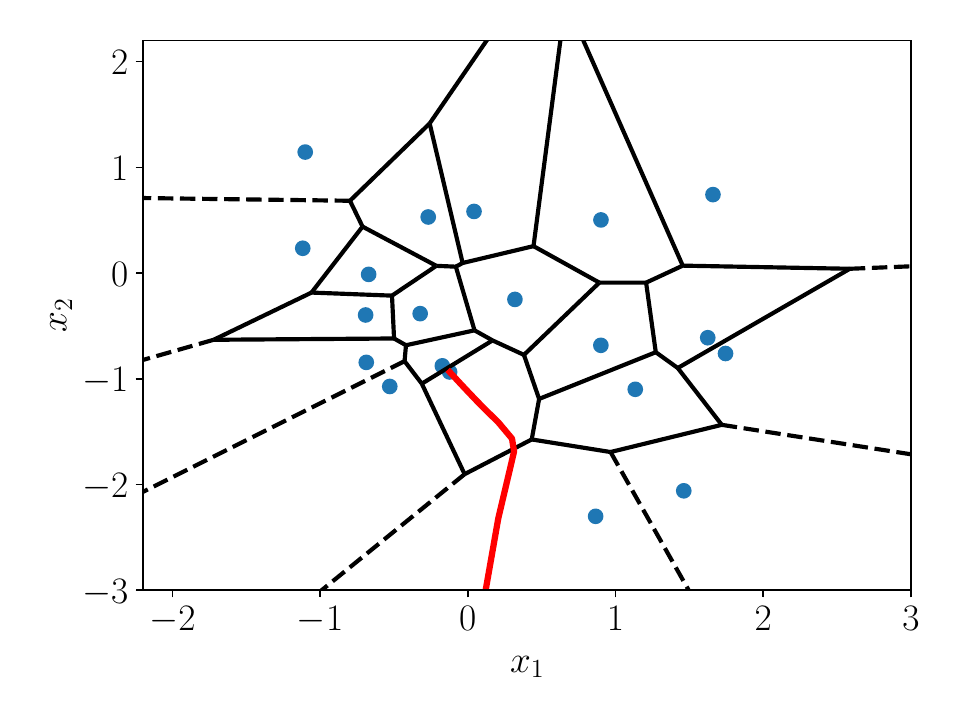}
\includegraphics[width=0.48\textwidth]{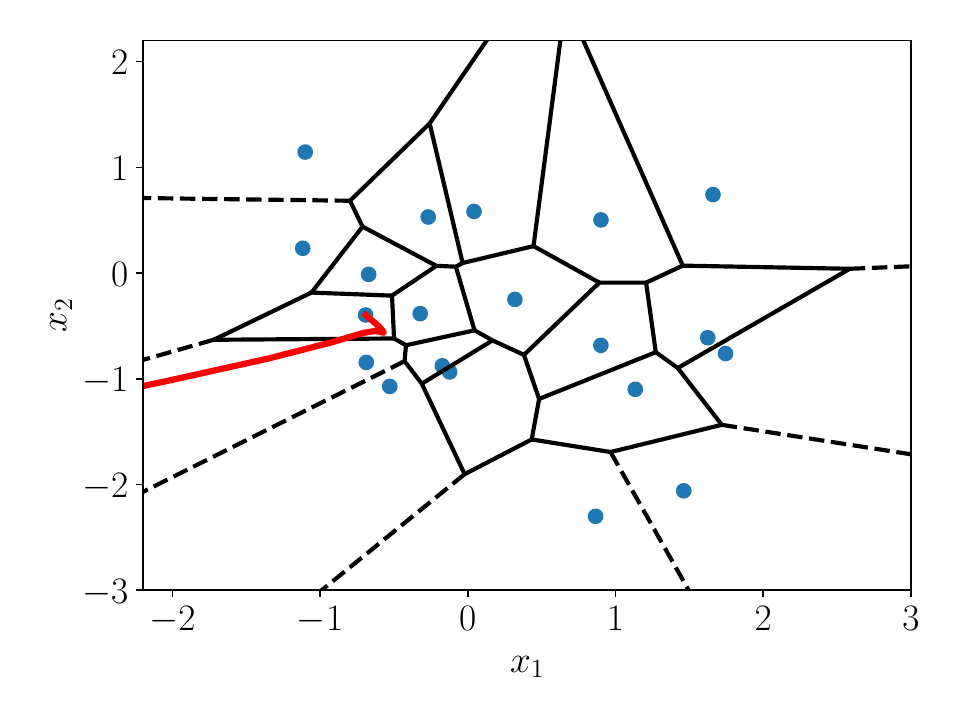}\\
\includegraphics[width=0.48\textwidth]{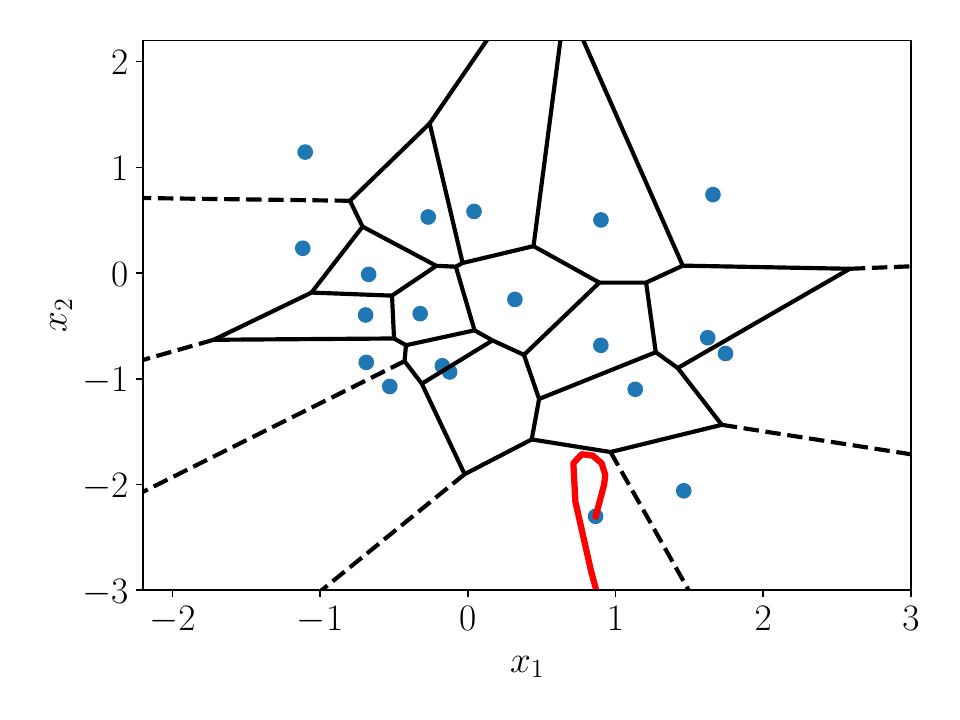}
\includegraphics[width=0.48\textwidth]{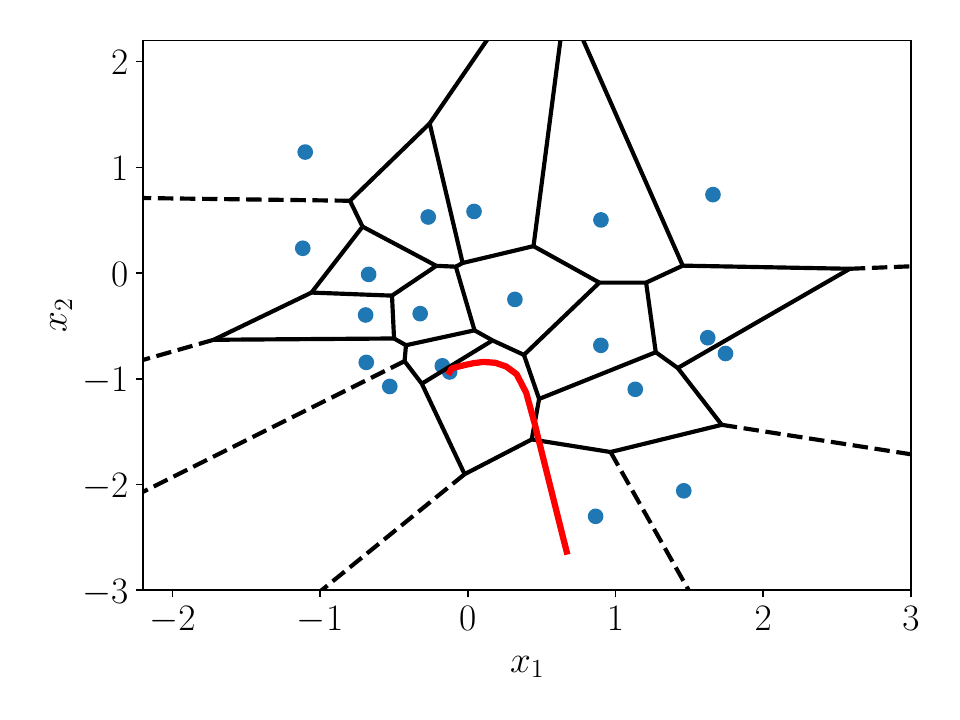}
\caption{Convergence of the reverse ODE trajectories for the variance exploding process to the data points in red when using the empirical score function starting from four initial conditions. The Voronoi tessellation is plotted for $N = 20$ samples in blue. \label{fig:random_Voronoi_teselation}}
\end{figure}

\begin{figure}[!ht]
\centering
\includegraphics[width=0.45\textwidth]{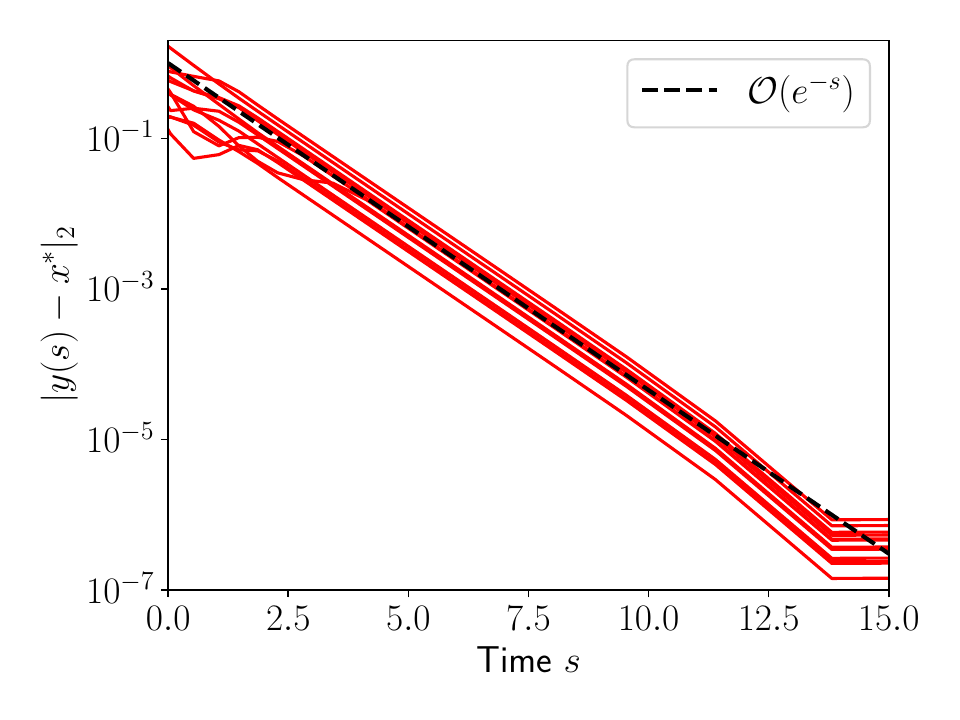}
\includegraphics[width=0.45\textwidth]{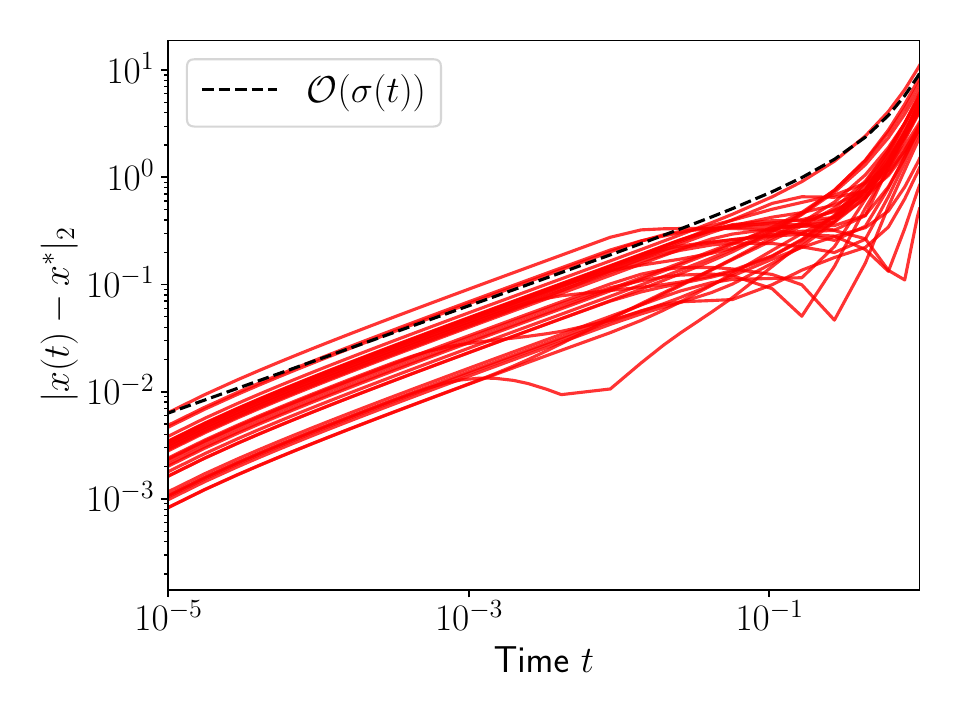}
\caption{Convergence rate of the reverse ODE solutions to the data points for the variance exploding process with the empirical score function for 30 independent trajectories in the transformed time $s$ (\textit{left}) and original time $t$ (\textit{right}). \label{fig:random_exponential_convergence}}
\end{figure}

In Figure~\ref{eq:ODE_twopoints}, we consider a data distribution with $N = 2$ or $N = 4$ points in blue arranged at an equal distance from the origin; the case $N=2$
is studied analytically in Example~\ref{ex:refer_equidistant_data}. The trajectories of the ODE solutions, starting from initial conditions on the boundaries of a square of width $2\sigma(T)$, are plotted in gray. For $N = 2$ data points (left of Figure~\ref{eq:ODE_twopoints}), we observe that trajectories converge to the data points, unless they are initialized on the hyperplane that separates the two data points. When initialized on the hyperplane, the solution remains there for all time $t$ and its limit point is the origin, the empirical mean of the two data points. For $N = 4$ data points (right of Figure~\ref{eq:ODE_twopoints}), we observe that trajectories either converge to the data points or remain on the two hyperplanes separating the data. While the local average of any two immediately neighboring data points is not a limiting point, in practice we observe that the dynamics have limit points at these local averages. This behavior arises because the contribution of the remaining two points to the right hand side of the ODE is exponentially small and numerically zero as $t \rightarrow 0^+$. Thus, the ODE dynamics can be numerically approximated based on only two data points.

\begin{figure}
\centering
\includegraphics[width=0.45\textwidth]{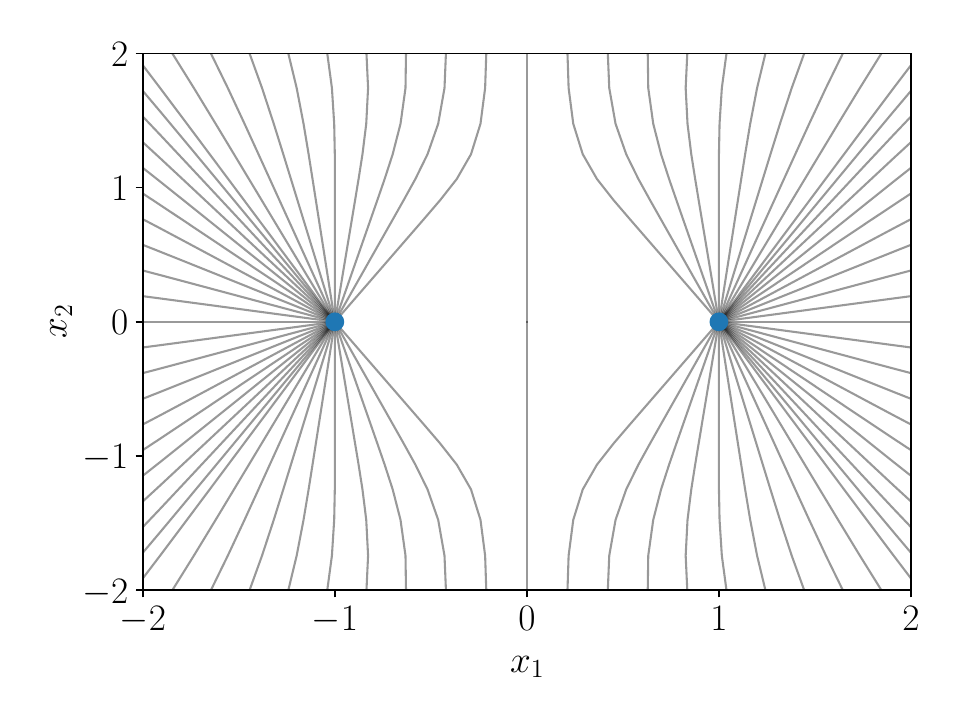}
\includegraphics[width=0.45\textwidth]{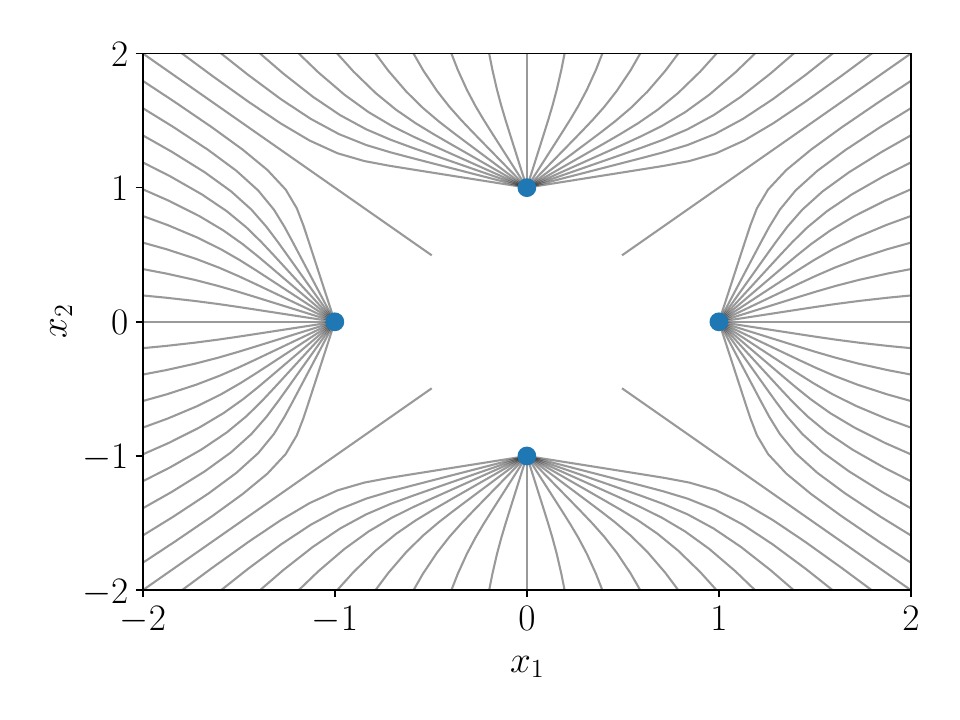}
\caption{Trajectories of the reverse ODE for the variance exploding process with $N=2$ (\textit{left}) and $N=4$ (\textit{right}) data points. The trajectories start from a square of width $2\sigma(T)$. In both settings, we observe that most trajectories converge to the data points in blue, while some remain on the hyperplanes between the data for all time. \label{eq:ODE_twopoints}}
\end{figure}

\section{Regularized Approximation of the Score} \label{sec:regularization}

The results in the previous section demonstrate that score-based diffusion models, if they deploy the {\it exact} minimizer of the
empirical loss function, yield a generative model that memorizes the empirical data distribution. This motivates the need for regularization of the empirical loss minimization
approach to the inverse problem introduced in Subsection~\ref{ssec:RGM}. We consider three types of
regularization: (i) Tikhonov; (ii) empirical Bayes
and (iii) neural network approximation; in case (iii) we consider regularization
through both early stopping of the training and through under-parameterization. Subsection~\ref{ssec:Tikonov} is devoted to  Tikhonov regularization and Subsection~\ref{sec:EmpiricalBayes} to empirical Bayes; both use an example based on data drawn from
a two-dimensional Gaussian distribution. In Subsection~\ref{sec:2DParametricScore} we consider neural network regularization, also for data drawn from the same two-dimensional Gaussian; Subsection~\ref{sec:NNrectangles} includes analogous results when neural network regularization is applied to an imaging problem.

\subsection{Tikhonov Regularization} \label{ssec:Tikonov}

In this section we regularize the score matching problem by adding a term to the objective that penalizes the squared norm of the estimated score function in expectation over space and time. The summary of this study of Tikhonov regularization is that
it can work effectively provided that the time-dependent parameter scaling the regularization term is chosen carefully to blow up inversely with respect to the variance of the forward process, as $t \to 0^+.$

We start by describing the methodology in the population loss
setting; then we generalize to the empirical loss setting.
Having done so, we show how the methodology may be interpreted in terms of stopping the reverse diffusion process at a positive time and how this regularizes. We conclude the subsection with numerical illustrations. We assume throughout the subsection that
the unconditioned score $\nabla_x \log p(x,t)$ and conditioned score $\nabla_x \log p(x,t|x_0)$ are continuous functions of $(x,t) \in \R^d \times (0,T].$

\paragraph{Tikhonov Regularized Score Matching.} Let $|x|_{\Gamma}^2 \coloneqq x^T\Gamma x$ denote the matrix-weighted Euclidean norm for some positive definite and (possibly) time-dependent matrix $\Gamma(t) \in \R^{d \times d}$ for $t \in [0,T]$. Then define the \textit{regularized score matching loss} found from \eqref{eq:score_matching_loss} by
adding a Tikhonov regularizer defined through time-dependent weight
$\Gamma(\cdot):$
\begin{equation} \label{eq:Tikh_regularized_J}
\Jreg(\score) = \int_0^T \lambda(t) \Bigl(\mathbb{E}_{x \sim p(\cdot,t)}|\score(x,t) - \nabla_x \log p(x,t)|^2  +  \mathbb{E}_{x \sim p(\cdot,t)}|\score(x,t)|_{\Gamma(t)}^2 \Bigr) \dd t.
\end{equation}

Let $\sreg$ denote the minimizer of $\Jreg(\score).$
The following result identifies the stationary point of $\Jreg(\cdot)$ in closed-form. The proof requires defining a constant $K_\Gamma$ given by
\begin{align} \label{eq:objective_reg_constant}
K_\Gamma &\coloneqq \int_0^T \lambda(t)\mathbb{E}_{x \sim p(\cdot,t)}\bigl(|\nabla_x \log p(x,t)|^2  -|(I_d + \Gamma(t))^{-1}\nabla_x \log p(x,t)|_{I_d + \Gamma(t)}^2\bigr) \dd t.
\end{align}

\begin{theorem} \label{prop:WeightedReg} Assume that $\Gamma(t)$ is uniformly strictly
positive-definite for $t \in [0,T]$ and that $K_\Gamma \in (0, \infty).$ Then, the infimum of~\eqref{eq:Tikh_regularized_J} is $K_\Gamma$ and is achieved at
\begin{equation} \label{eq:regularized_minimizer}
  \sreg(x,t) \coloneqq \big ( I_d + \Gamma(t) \big )^{-1} \nabla_x \log p(x,t).
\end{equation}
\end{theorem}

\begin{proof}
By completing the square, the objective can be written as
\begin{equation} \label{eq:CompletedSquare}
\Jreg(\score) = \int_0^T \lambda(t)\mathbb{E}_{x \sim p(\cdot,t)}|\score(x,t) - (I_d + \Gamma(t))^{-1}\nabla_x \log p(x,t)|_{I_d + \Gamma(t)}^2 \dd t + K_\Gamma,
\end{equation}
where the constant $K_\Gamma$ is given in~\eqref{eq:objective_reg_constant} and is
finite by assumption. Then, it follows from~\eqref{eq:CompletedSquare} that the minimizer of $\Jreg$ has the form in~\eqref{eq:regularized_minimizer}.
\end{proof}

\begin{remark} Since there is $\gamma>0$ such that, for each $t \in [0,T]$, $\Gamma(t) \succeq \gamma I_d$, it holds that $0<K_\Gamma \leq \int_0^T \lambda(t)\mathbb{E}_{x \sim p(\cdot,t)}|\nabla_x \log p(x,t)|^2 \dd t$. Thus, the assumption on $K_\Gamma$ holds if the exhibited upper-bound is finite. We refer to Remark~\ref{rem:theK}, which provides examples where this condition holds. 
\end{remark}

As in the unregularized setting described in Subsection~\ref{ssec:prac}, we may identify a loss $\Jregzero$ that differs from $\Jreg$ in~\eqref{eq:regularized_minimizer} by an additive constant:
$$\Jregzero(\score) = \int_0^T \lambda(t)\Bigl(\mathbb{E}_{x_0 \sim p_0(\cdot), x \sim p(\cdot,t|x_0)}|\score(x,t) - \nabla_x \log p(x,t|x_0)|^2  + |\score(x,t)|^2_{\Gamma(t)}\Bigr) \dd t.$$
Recall the definition of the constant $K$ in~\eqref{eq:objective_constant}. The
following result is proved similarly to~\Cref{p:ifKf}.
\begin{proposition}
Assume that constant $K$ is finite. Then,
$$\Jreg(\score) = \Jregzero(\score) - K.$$
Hence the minimizers of $\Jreg$ and $\Jregzero$ coincide and the minimizer of $\Jregzero$ is given by $\sreg(x,t) = (I_d + \Gamma(t))^{-1}\nabla_x \log p(x,t)$.
\end{proposition}

\paragraph{Tikhonov Regularized Score Matching in Practice.} We wish to identify the minimizer given an empirical data distribution $p_0^N$ as in~\eqref{eq:empirical} consisting of $N$ i.i.d.\thinspace samples $\{x_0^n\}_{n=1}^N$. Consider now the approximation $\Jregzero^N$ of the loss $\Jregzero$, defined by
\begin{align}
  \Jregzero^N(\score) &= \int_0^{T} \lambda(t) \mathbb{E}_{x_0 \sim \pr_0^N(\cdot), x \sim p(\cdot,t|x_0)}\bigl(|\score(x,t) - \nabla_\bu \log \pr(\bu,t|x_0)|^2 + |\score(x,t)|_{\Gamma(t)}^2\bigr) \dd t. %
  \label{eq:score_matching_regularized_lossN}
\end{align}
Theorem \ref{prop:WeightedReg} shows that the loss function $\Jregzero^N(\cdot)$ is minimized at $$\sreg^N(x,t)= \bigl(I_d + \Gamma(t)\bigr)^{-1}\nabla_x \log p^N(x,t),$$ where we recall the Gaussian mixture density $p^N$ in~\eqref{eq:MarginalDensityXtN}. Using Theorem~\ref{thm:empirical_score_unconditional} for the explicit form of the empirical score function, we have 
\begin{equation} \label{eq:sreg}
\sreg^N(x,t) = -\Bigl(\bigl(I_d + \Gamma(t)\bigr)\sigma^2(t)\Bigr)^{-1} \left(x-m(t)\sum_{n=1}^N w_n(x,t) x_0^n\right).
\end{equation}
The choice of forward process determines the conditional mean $m(t)$ and variance $\sigma^2(t)$ in~\eqref{eq:msig}, and hence the normalized Gaussian weights $\w_n(x,t)$ in~\eqref{eq:normalized_Gaussian_weights}.

\paragraph{Regularizing Effect.}
Recall that $\sigma(t) \rightarrow 0$ as $t \rightarrow 0^+$, and hence the unregularized empirical score function $\score^N$, corresponding to $\Gamma(t) = 0$ in~\eqref{eq:sreg}, has singular behavior in time as $t \rightarrow 0^+$. However, by choosing $\Gamma(t) \propto 1/\sigma^2(t)$,
the regularized empirical score function $\sreg^N$ remains bounded as $t \rightarrow 0^+$. In order to study the effect of this regularization, in the context of memorization of the training data, we focus on the variance exploding setting. 

Choosing $\Gamma(t)=c/\sigma^2(t) I_d$ for some $c>0$, the regularized empirical
score function in \eqref{eq:sreg} takes the form
\begin{equation} \label{eq:sreg2}
\sreg^N(x,t) = -\bigl(c+\sigma^2(t)\bigr)^{-1} \left(x-m(t) \sum_{n=1}^N w_n(x,t) x_0^n\right).
\end{equation}

The variance exploding reverse ODE, found from~\eqref{eq:reverseODE_empirical} in the
case $\beta \equiv 0$, resulting in $m(t) \equiv 1$, and with $\score^N$ replaced by $\sreg^N$ from \eqref{eq:sreg2}, 
then becomes
\begin{align}
\label{eq:mf_ode_empR}
    \frac{\dd x}{\dd t} &=\frac{g(t)}{2(c+\sigma^2(t))}\bigl(x-\xs_N(x,t)\bigr),
\quad x(T)=x_T.
\end{align}
While this is implemented with $x_T \sim \mathfrak{g}_T\coloneqq N(0,\sigma^2(T)I_d)$,
our discussion concerns the behavior of the solution to~\eqref{eq:mf_ode_empR} for any fixed $x_T.$ We apply transformation~\eqref{eq:tt}
to determine $s$ from $t$ based on \eqref{eq:tt2}, 
obtaining the following equation for $y(s) = x(t)$: 
\begin{subequations}
\label{eq:regreg}
    \begin{align}
        \frac{dy}{ds} &=-\bigl(y-\ys_N(y,s)\bigr), \quad y(s_0)=x_T,\\ 
        \ys_N(y,s) &=\sum_{n=1}^N \ww_n(y,s)x_0^n, \quad \ww_n(y,s)=\w_n\Bigl(y,\bigl(\ti(\cdot;c)\bigr)^{-1}(s)\Bigr).
    \end{align}
\end{subequations}
with normalized weights
\begin{equation} \label{eq:Snormalized_Gaussian_weightsC}
  \ww_n(y,s) = \frac{\wwtilde_n(y,t)}{\sum_{\ell=1}^N \wwtilde_\ell(y,s)}, \qquad \wwtilde_n(y,s) = \exp\left(-\frac{|y - x_0^n|^2}{2(e^{-2s}-c)}\right). 
\end{equation}
We solve this equation starting from $s=s_0\coloneqq\ti(T;c),$ and integrate to $s=s_{\infty}\coloneqq\ti(0;c).$ Since $c>0$ and $\sigma^2(0)=0$ we see that $s_{\infty}=-\frac12 \ln(c) \in (s_0,\infty)$ is well-defined. 

To understand the regularization effect of choosing $c>0$, we refer back to 
Remark~\ref{rem:idea}. Note that $e^{-2s_{\infty}}=c$ so that the weight collapse
phenomenon will happen, similarly to the unregularized setting, as $s \to s_{\infty}.$ Thus once again
we may think of making the approximation \eqref{eq:mf_ode4A} to the evolution of $y$,
for $s$ close to $s_{\infty}:$
\begin{equation*}
    \frac{dy}{ds} \approx -\bigl(y-x_0^n\bigr).
\end{equation*}
However, we no longer have an infinite time horizon on
which this approximation is valid since we only integrate until $s_{\infty}<\infty.$ So, whilst we expect the generated sample to be close to the data set, we do not expect to see the memorization effect contained in \Cref{cor:exponential_convergence_newcoords} when taking $s \to \infty.$  
The numerical results that follow confirm this.

\paragraph{Numerical Results.} %
We illustrate the effect of Tikhonov regularization on the memorization
phenomenon, using the same data distribution $p_0^N$ that we considered in Section~\ref{sec:numerics_VE}. %
We learn the score by minimizing~\eqref{eq:score_matching_regularized_lossN} using
the same variance exploding forward process in Section~\ref{sec:numerics_VE} with $g(t) = 10^t$ and $T = 1$. We then integrate ODE~\eqref{eq:mf_ode_empR} backwards from $t=T$ to $t=0$ in order to generate new samples. We study
the generated samples for a set of increasing values of the regularization parameter $c$.

\begin{figure}[!ht]
\centering
\includegraphics[width=0.32\textwidth]{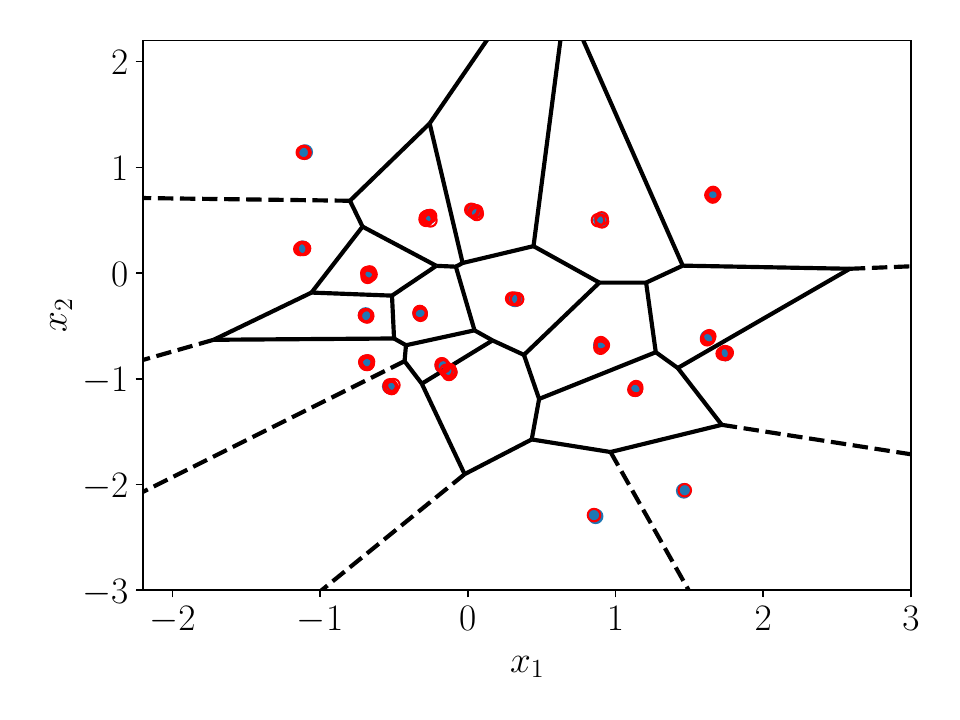}
\includegraphics[width=0.32\textwidth]{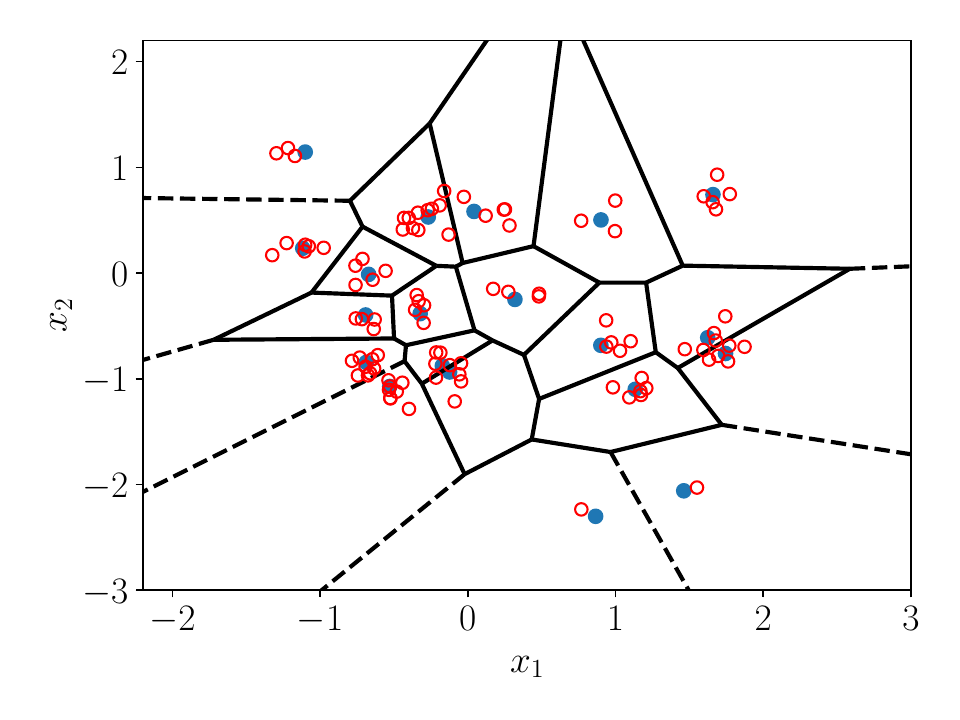}
\includegraphics[width=0.32\textwidth]{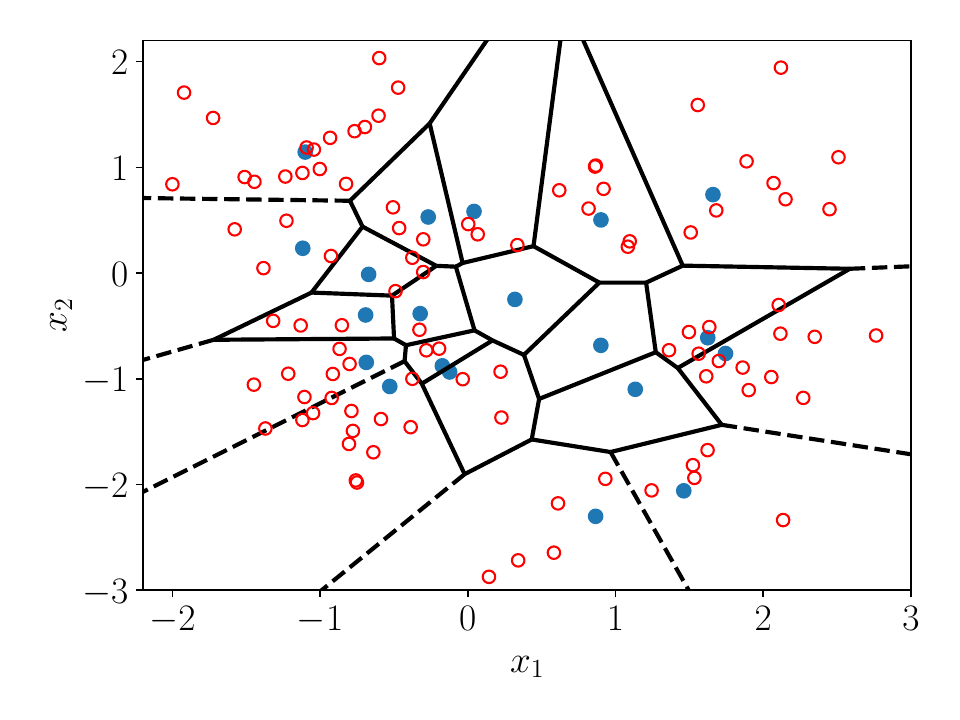}
\caption{$100$ generated samples $x(0)$ found by integrating the reverse ODE 
\eqref{eq:mf_ode_empR} backwards from $T=1.$ Blue dots comprise the data set;
red circles are the generated points. The Tikhonov regularized score function is
used with values $c = 0.0001$ ({\it left}), $c = 0.01$ ({\it middle}) and $c = 0.1$ ({\it right}). The smallest value of $c$ exhibits collapse onto the data set; the
larger values of $c$ prevent memorization. \label{fig:tikonov_regularized}}
\end{figure}

Figure~\ref{fig:tikonov_regularized} shows 100 independent samples for three values
of $c$. The smallest value of $c$, namely $10^{-4}$, exhibits memorization:
all samples collapse onto the data set; the larger values of $c$ prevent this memorization phenomenon. In Figure~\ref{fig:Tikhonov_regularization_collapse} we quantitatively measure the effect of the regularization parameter $c$ on the amount of memorization. For a range of values of $c \in [10^{-5},10^{-1}]$ we plot the fraction of 1000 generated samples $x(0)$ that are in a $\tau-$neighbourhood of the training data, as measured in the
Euclidean norm. In our experiment we set $\tau = 10^{-2}$. While all samples are collapsed for small $c \leq 10^{-5}$, most generated samples remain far from the training data with $c \geq 10^{-2}$; these quantitative results manifest qualitatively in Figure~\ref{fig:tikonov_regularized}.

\begin{figure}
    \centering
    \includegraphics[width=0.5\linewidth]{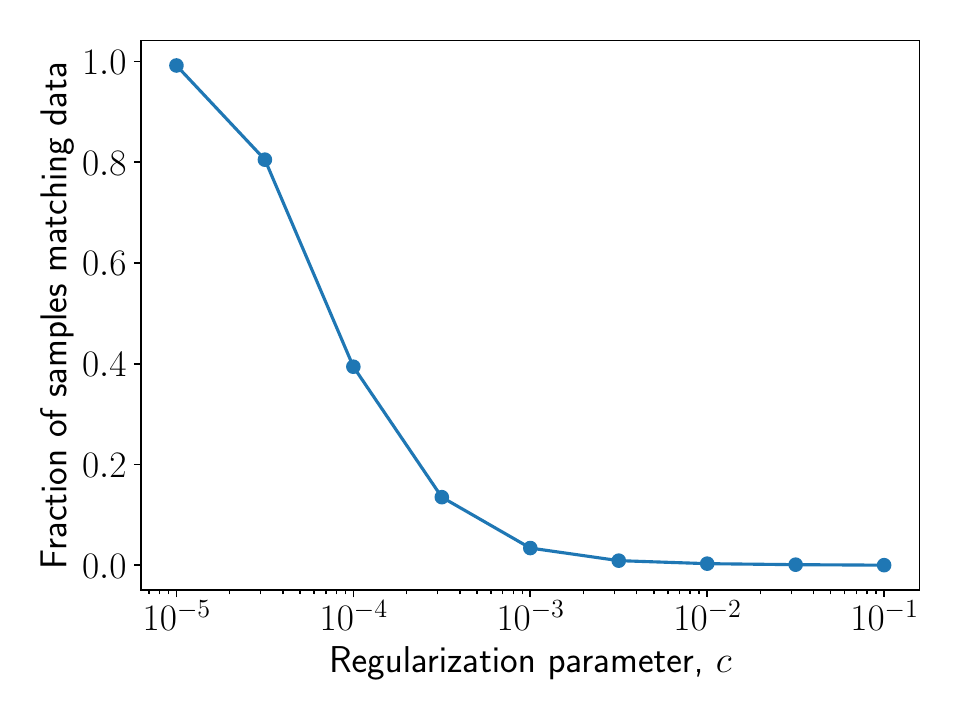}
    \caption{The effect of the regularization parameter on the memorization phenomenon with Tikhonov regularization. The setting is identical to that in Figure \ref{fig:tikonov_regularized}. %
    \label{fig:Tikhonov_regularization_collapse}}
\end{figure}

\subsection{Empirical Bayes Regularization} \label{sec:EmpiricalBayes}

In this section we consider a regularization of the empirical score function, 
as proposed in~\cite{wibisono2024optimal}. That work is focused on 
regularization for the purpose of deriving an asymptotically consistent
score-based generative model in the large data limit $N \to \infty$. Here we define the resulting methodology, show the regularizing effect that it has for fixed $N$, and illustrate the effect through numerical experiments. The effect is similar to that of Tikhonov regularization but has the added advantage of being based on a theory of
asymptotic consistency; on the other hand, it also requires evaluation of the Gaussian mixture at every time-step of the reverse process, which may be prohibitively expensive for large data sets.

\paragraph{Regularized Estimator.}
Recall that the score $\score^N$, defined by~\eqref{eq:empirical_score}, minimizes the empirical loss function~\eqref{eq:score_matching_lossN}. It is defined as
the logarithmic derivative of the Gaussian mixture $p^N(x,t)$ given in~\eqref{eq:MarginalDensityXtN}. In Theorem \ref{thm:memorization}
we show that its use leads to memorization. The empirical Bayes-based methodology proposed in~
\cite{wibisono2024optimal} leads to a regularized score estimator for a constant $c>0$ of the form
\begin{equation} \label{eq:EB_scorefunction}
    \seb^N(x,t) \coloneqq \frac{\nabla_x p^N(x,t)}{\max(p^N(x,t),c)};
\end{equation}
this reduces to $\score^N(x,t)$ if $c=0.$
A scaling rule for the regularization parameter $c>0$, with respect to the number of samples $N$, is derived to create a consistent estimator for the score as $N \rightarrow \infty$. 

\paragraph{Numerical Results.}
Our numerical results are in the setting of the variance exploding forward process in Section~\ref{sec:numerics_VE}. Using the structure of the Gaussian mixture model, the empirical score function in~\eqref{eq:EB_scorefunction} has the form 
\begin{equation}
    \seb^N(x,t) = -\frac{1}{\sigma^2(t)}\sum_{n=1}^N (x - x_0^n) \frac{\wtilde_n(x,t)}{\max(\sum_{m=1}^N \wtilde_m(x,t),c)},
\end{equation}
where the unnormalized weights are given in~\eqref{eq:normalized_Gaussian_weightsVE}.
Note that this reduces to \eqref{eq:SNNalt} if $c=0.$ We again make the choice $g(t)=10^t$ and set $T=1.$ To conduct the
numerical experiments we use ODE~\eqref{eq:reverseODE_empirical} as a generative
model, setting $\beta \equiv 0$ (i.e., the variance exploding case) and replace $\score^N$ by $\seb^N.$ Generation is achieved by sampling the initial condition 
$x(T) \sim \mathcal{N}(0,\sigma^2(T)I_d)$.
Figure~\ref{fig:VE_EmpiricalBayes} plots 100 independent independent samples
for each of the three values of the regularization constant: $c = 0.01,$ $c = 0.1$ and $c = 1$. We observe that increasing the regularization parameter reduces the effect of collapse onto the $N = 20$ data points in the empirical dataset. 

In practice, an optimal value for $c$ should balance the statistical error of estimating the regularized score function from finite samples with the approximation error from using a non-zero constant $c$. Our aim here is simply to demonstrate that an
appropriate choice for $c$ can prevent data collapse and mitigate memorization behavior. Figure~\ref{fig:VE_EB_datacollapse} demonstrates the  sensitivity to the regularization parameter by plotting the fraction of 1000 generated samples $x(0)$ %
that are near a point in the dataset $\{x_0^n\}_{n=1}^N$, up to a small tolerance $\tau = 10^{-2}$ in Euclidean norm. 

\begin{figure}[!ht]
    \centering
    \includegraphics[width=0.32\linewidth]{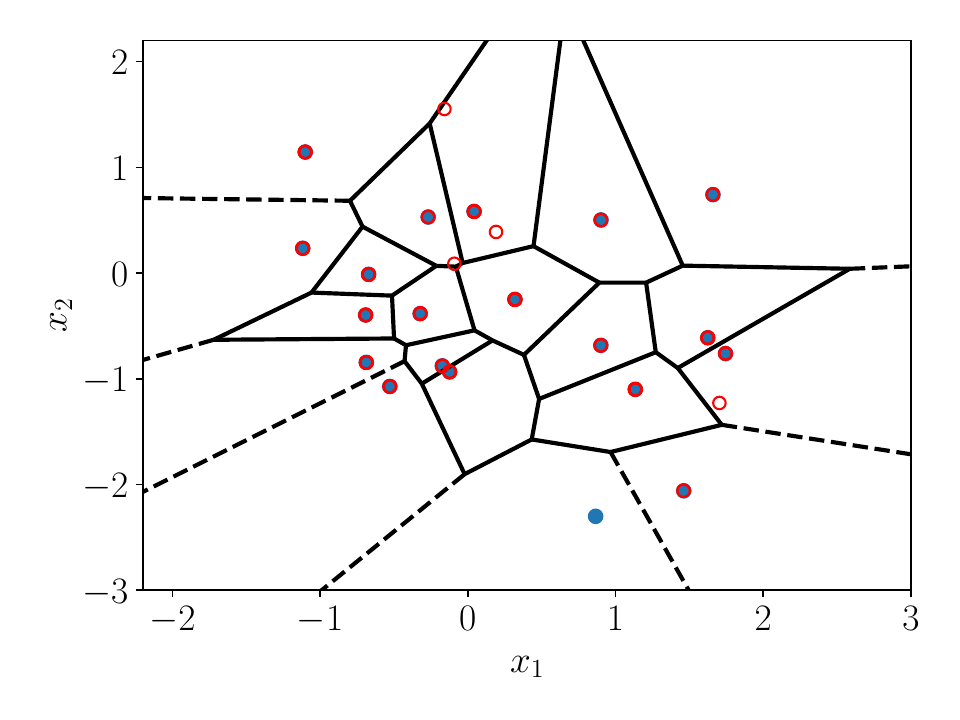}
    \includegraphics[width=0.32\linewidth]{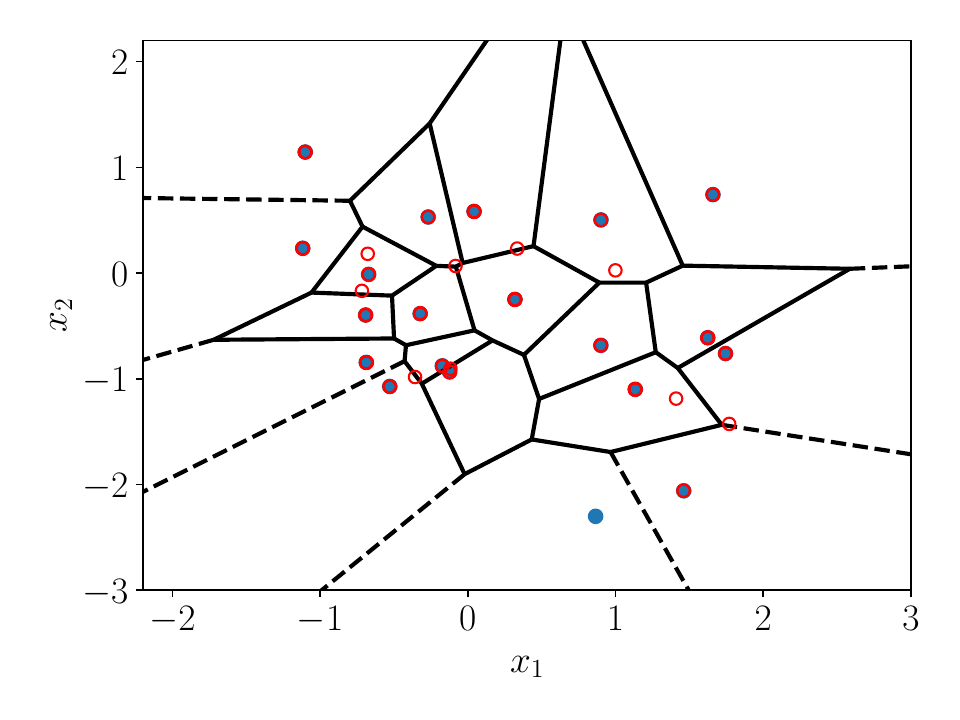}
    \includegraphics[width=0.32\linewidth]{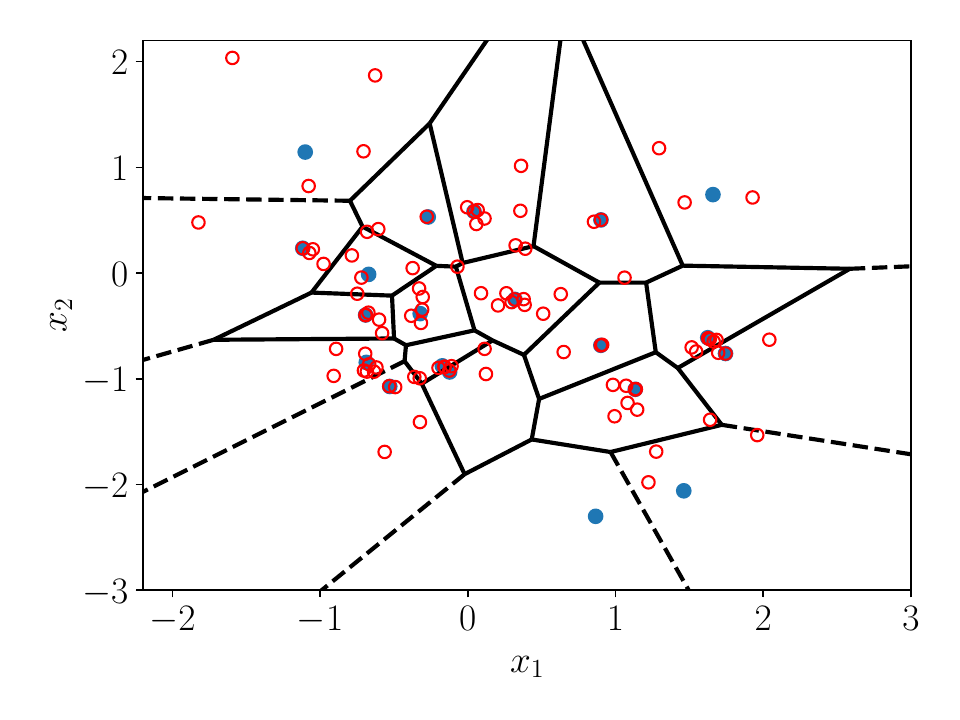}
    \caption{100 generated samples $x(0)$ found by integrating the reverse ODE with the score function given by the empirical Bayes-based estimator.
    Blue dots comprise the data set;
red circles are the generated points. Results are shown with $c = 0.01$ ({\it left}), $c = 0.1$ ({\it middle}), and $c = 1.0$ ({\it right}). \label{fig:VE_EmpiricalBayes}}
\end{figure}

\begin{figure}
    \centering
    \includegraphics[width=0.5\linewidth]{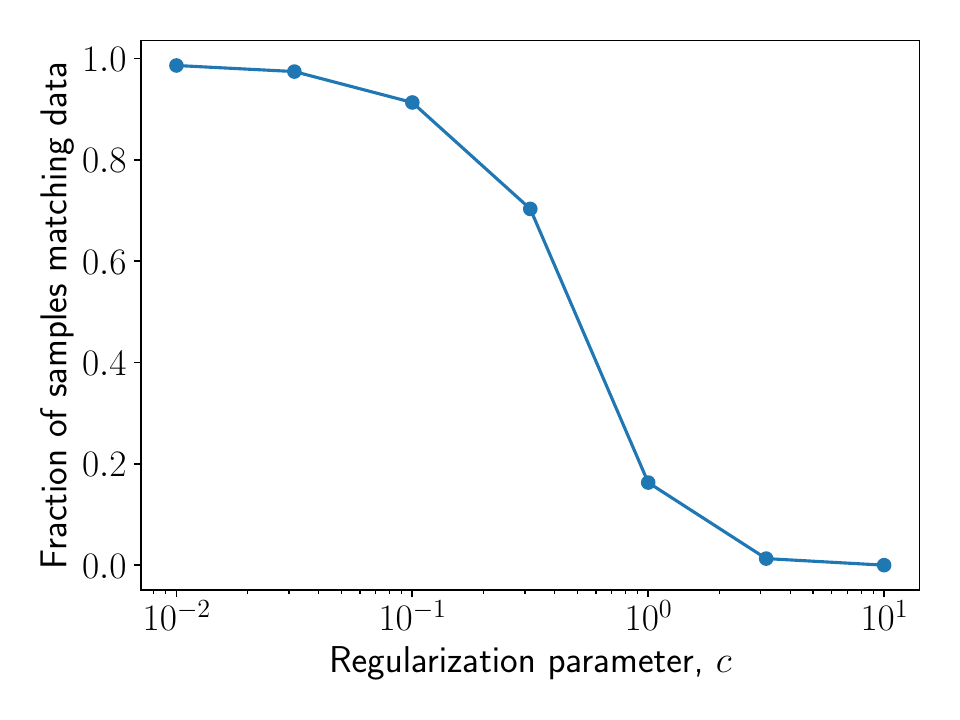}
    \caption{The effect of the regularization parameter on on the memorization phenomenon with empirical Bayes-based regularization. The setting is identical to that in Figure \ref{fig:VE_EmpiricalBayes}. \label{fig:VE_EB_datacollapse}}
\end{figure}

\subsection{Function Class Regularization: Two-Dimensional Dataset} \label{sec:2DParametricScore}

In this section, we investigate memorization with parametric neural network approximations to the score function. In summary, the experiments
show that either early stopping of training, or under-parameterization, can achieve similar regularization effects as those demonstrated by Tikhonov regularization
and by empirical Bayes-based regularization. This can be used to explain the
successful performance of diffusion-based generative models as used in practice
and when tuned to specific data sets. However, neural network approximations come with no theory guiding their regularization capabilities, in contrast to Tikhonov regularization and empirical Bayes regularization methods.

In all our experiments, we train the networks using the Adam optimizer~\citep{kingma2014adam} with a learning rate of $10^{-3}$. The network has two hidden layers, with the same layer width, and deploys the ReLU activation function. Time-dependence of the score is encoded using eight random Fourier features, %
following the approach in~\cite{song2020score}. The embedded time vector is concatenated with the spatial input before the network's first hidden layer. For all experiments in this subsection, we employ the variance preserving forward process from Example \ref{ex:VP} with the specific instantiation used to provide the numerical experiments in Appendix~\ref{ssec:numerics_VP}. %

First, we study the effect of training time on memorization by fixing the network width to 256 neurons and increasing the number of epochs for training. Second, we study the effect of network capacity by fixing the training time to 200,000 epochs and learning score functions with different widths of the hidden layers. In both cases, the neural networks are learned by minimizing the empirical score matching loss $\J_0^N$ in~\eqref{eq:denoising_score_matching_lossN} with the weighting $\lambda(t) = \sigma^2(t)$, which results in the loss function  
$$\J_0^N(\score) = \int_0^T \mathbb{E}_{x_0 \sim p_0^N(\cdot)}\mathbb{E}_{x \sim p(\cdot,t|x_0)}|\sigma(t)\score(x,t) - \sigma(t)\nabla_x \log p(x,t|x_0)|^2 \dd t.$$
Third, we consider neural network approximations to the score function that are learned with the Tikhonov regularized loss $\Jregzero^N$ in~\eqref{eq:score_matching_regularized_lossN}. Choosing the same weighting $\lambda(t) = \sigma^2(t)$, and setting the regularization parameter $\Gamma(t) = c/\sigma^2(t) I_d$ based on the analysis in Subsection~\ref{ssec:Tikonov}, results in the regularized loss function  
$$\Jregzero^N(\score) = \int_0^T \mathbb{E}_{x_0 \sim p_0^N(\cdot)}\mathbb{E}_{x \sim p(\cdot,t|x_0^n)}\Bigl( |\sigma(t)\score(x,t) - \sigma(t)\nabla_x \log p(x,t|x_0)|^2 + c|\score(x,t)|^2 \Bigr) \dd t.$$
Lastly, we contrast the neural network learned using the score matching loss $\J_0^N$ to one learned using the denoising score matching loss $\I_0^N$ presented in Section~\ref{ssec:LTN}, which forces singular behavior in the score function by directly parameterizing and optimizing over the function $\widetilde{\score}(x,t): = \sigma(t)\score(x,t),$. Then, the estimated score function is recovered from $\widetilde{\score}(x,t)/\sigma(t)$. 

Figure~\ref{fig:NNepoch_samples} presents the generated samples from the learned score function with an increasing number of epochs from $5000$ to $300,000$. 
In our experiments we observe that increasing the number of epochs results in more samples closer to a data point. We measure the collapse quantitatively by computing the fraction of 1000 generated samples $x(0)$ that are inside a Euclidean ball of size $\tau = 10^{-1}$ around one of the training data points. We select the threshold $\tau$ to be larger than the one considered in the earlier experiments that use analytic expressions for score functions to account for the stochasticity in the learning process. We note that the threshold $10^{-1}$ is smaller than the average of the minimum pairwise distances between data points of $4.3 \times 10^{-1}$, and similar to the minimum pairwise distance of $0.8 \times 10^{-1}$. Figure~\ref{fig:NNepoch_collapse} illustrates that the generated samples exhibit memorization without early stopping of the training process.

\begin{figure}[!ht]
    \centering
    \includegraphics[width=0.32\linewidth]{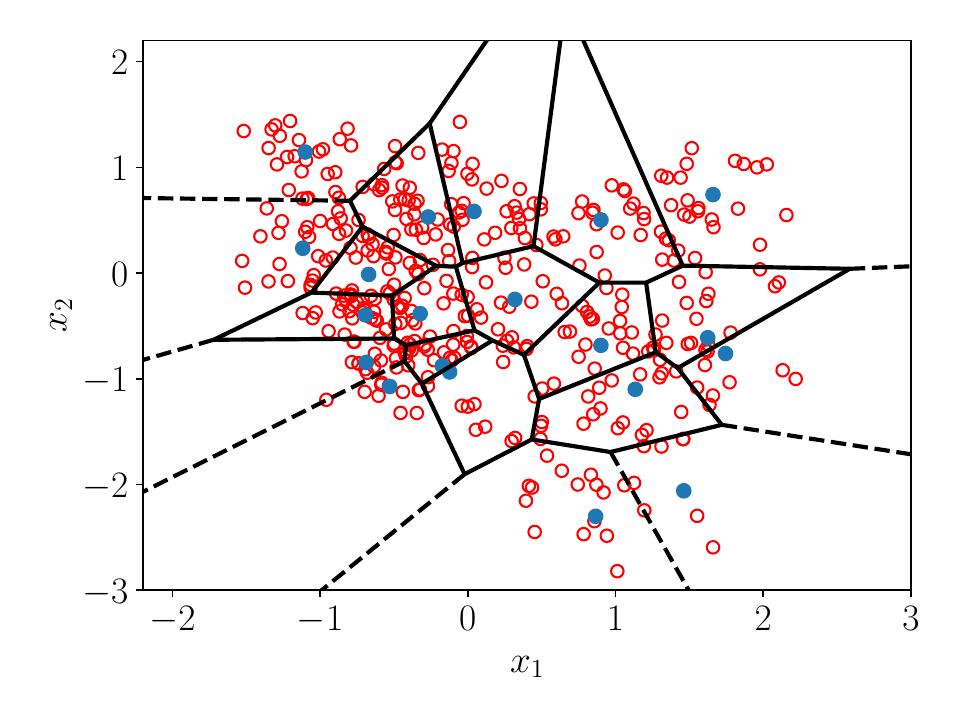}
    \includegraphics[width=0.32\linewidth]{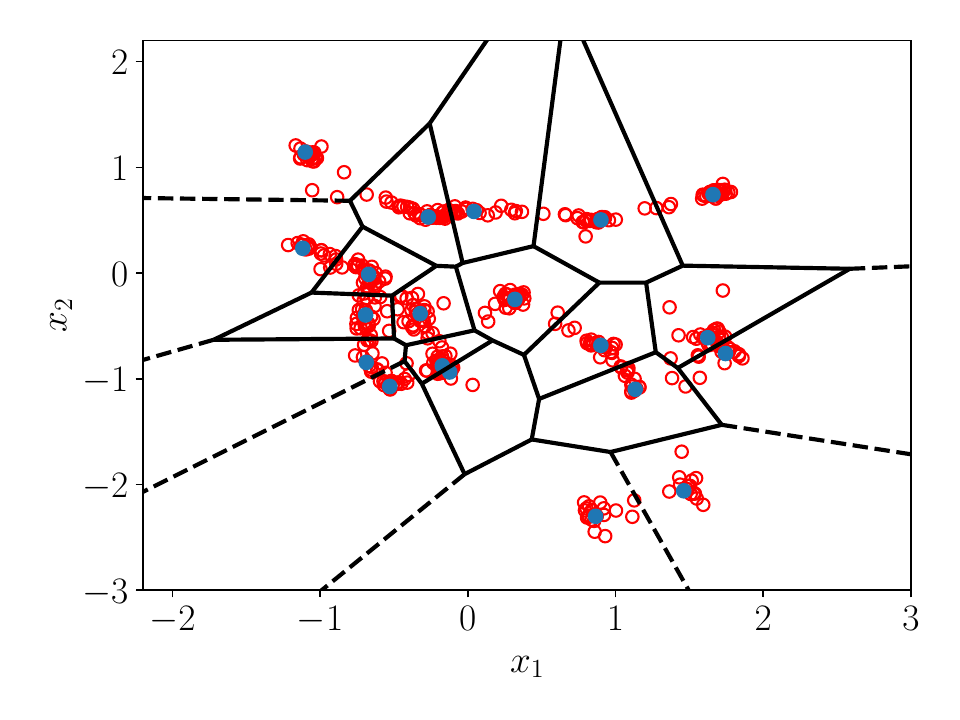}
    \includegraphics[width=0.32\linewidth]{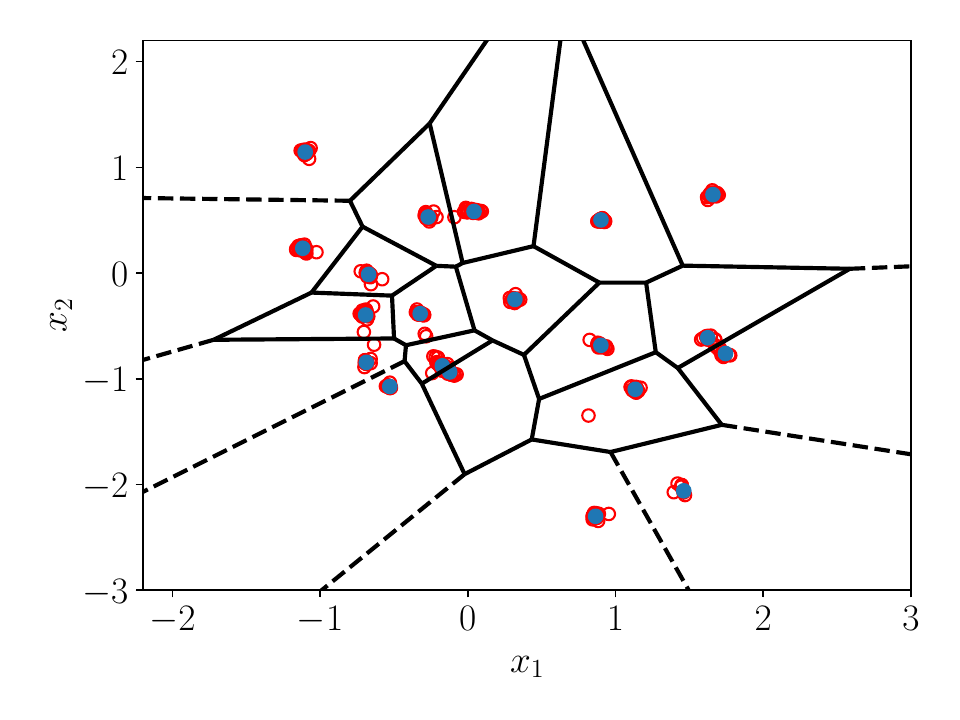}
    \caption{400 generated samples $x(0)$ found by integrating the reverse ODE with the learned score function. Blue dots comprise the data set; red circles are the generated points. The learned neural networks are trained for $5 \times 10^3$ epochs ({\it left}), $3 \times 10^4$ epochs ({\it middle}) and $3 \times 10^5$ epochs ({\it right}). \label{fig:NNepoch_samples}}
\end{figure}

\begin{figure}[!ht]
    \centering
    \includegraphics[width=0.5\linewidth]{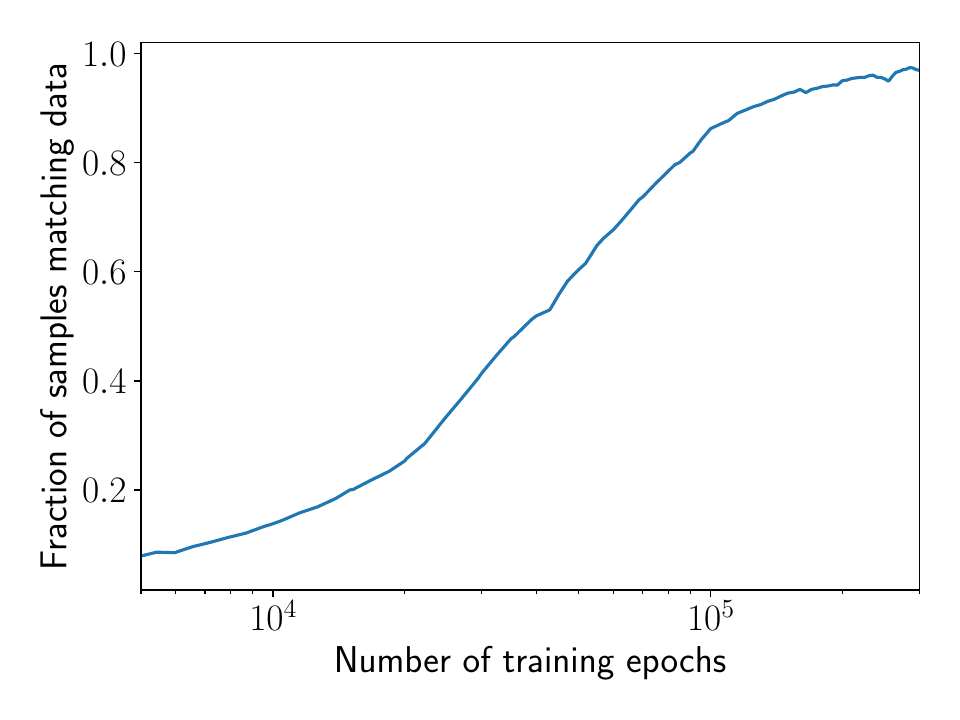}
    \caption{The effect of the training time, as measured with the number of epochs, on the memorization phenomenon. The setting is identical to that in Figure \ref{fig:NNepoch_samples}. \label{fig:NNepoch_collapse}}
\end{figure}

Now we study the effect, on memorization, of increasing the network width. Figure~\ref{fig:NNwidth_samples} plots the generated samples $x(0)$ after training three independent networks with a total of 508, 5068 and 69,388 parameters for widths of size 8, 32 and 128, respectively. We observe that increasing the number of parameters produces generated samples (red dots) that are closer to the training data (blue dots). This behavior is expected when adding more parameters because the learned functions can better approximate the minimizer $\score^N$ of the empirical loss $\J_0^N$ in Theorem~\ref{thm:empirical_score_unconditional}, which results in memorization. Figure~\ref{fig:NNwidth_collapse} illustrates the memorization quantitatively by computing the fraction of 1000 generated samples that are at a Euclidean distance from a training data point that is smaller than the threshold $\tau = 10^{-1}$, the same one
used in Figure~\ref{fig:NNepoch_collapse}.

\begin{figure}[!ht]
    \centering
    \includegraphics[width=0.32\linewidth]{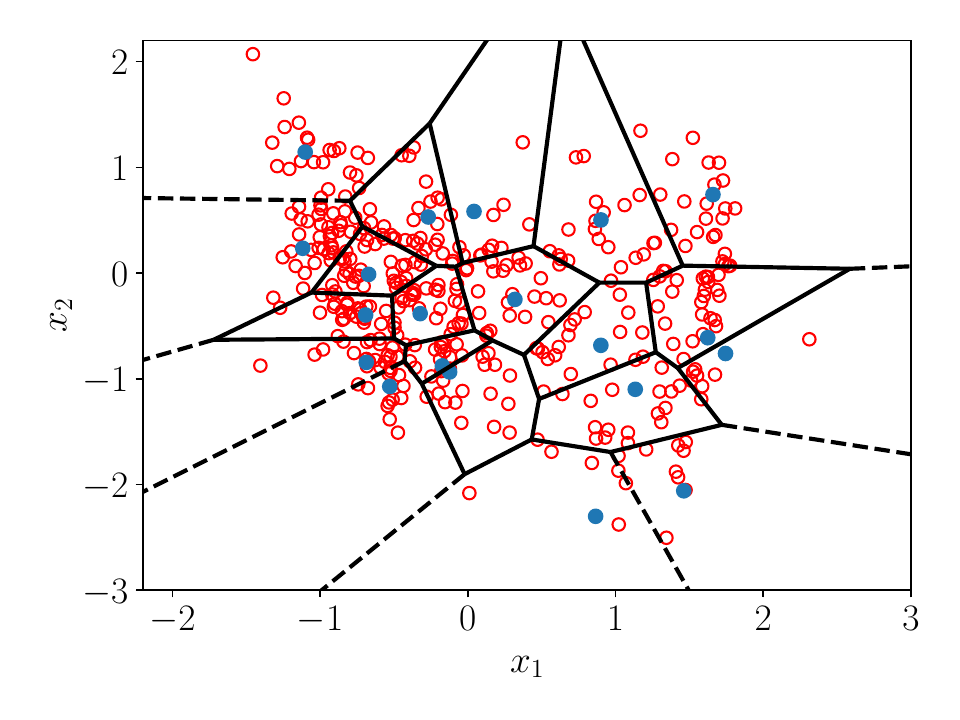}
    \includegraphics[width=0.32\linewidth]{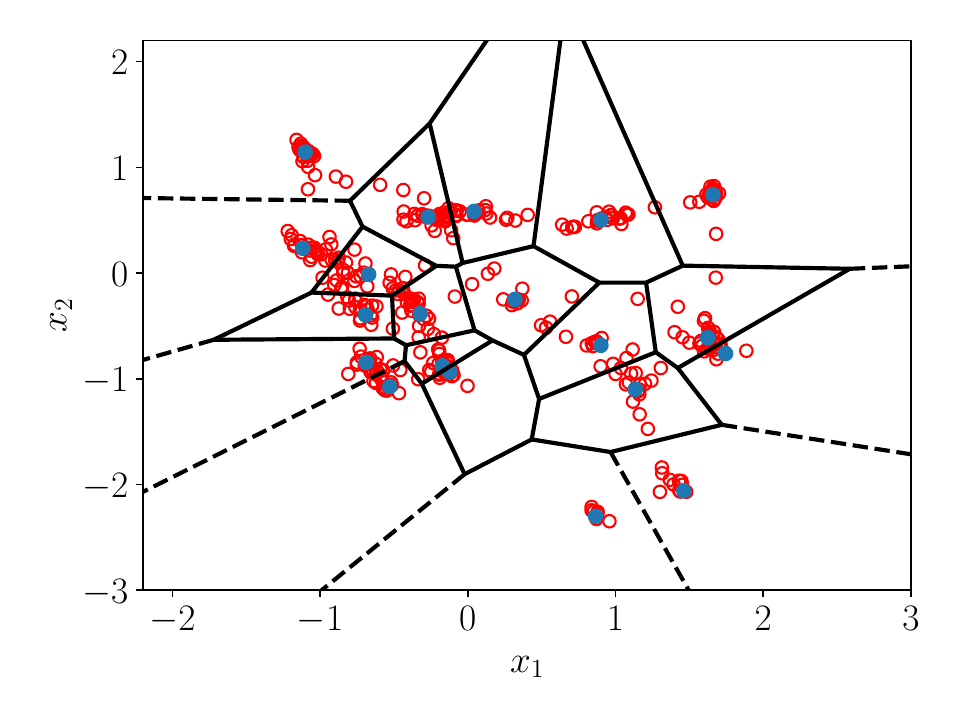}
    \includegraphics[width=0.32\linewidth]{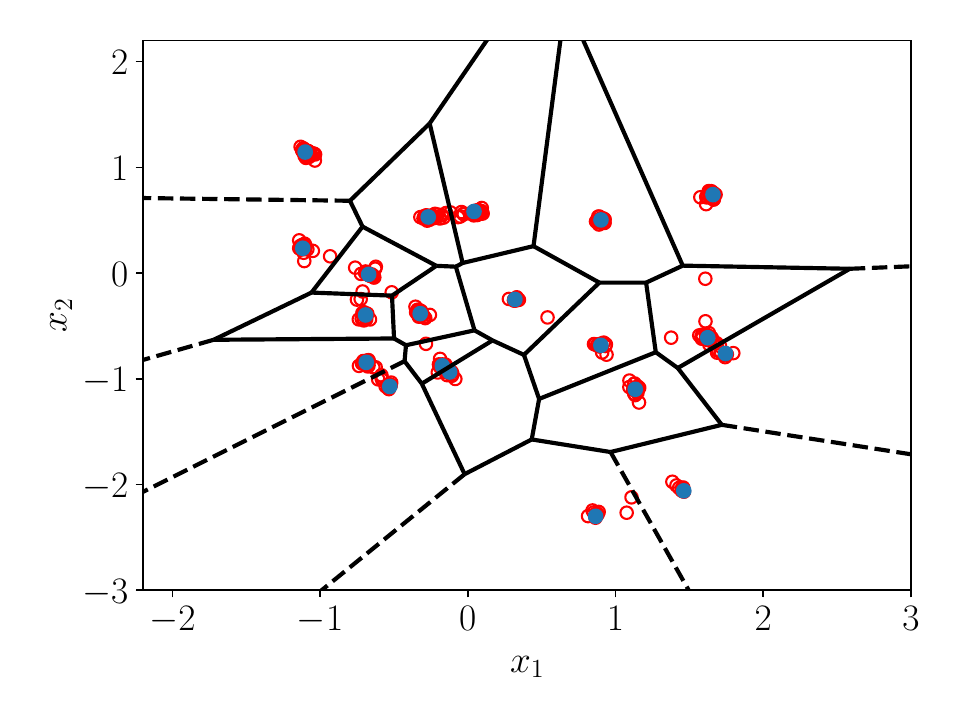}
    \caption{400 generated samples $x(0)$ found by integrating the reverse ODE with the learned score function. Blue dots comprise the data set; red circles are the generated points. The learned score functions are neural network that have widths of size $8$ ({\it left}), $32$ ({\it middle}) and $128$ ({\it right}). \label{fig:NNwidth_samples}}
\end{figure}

\begin{figure}[!ht]
    \centering
    \includegraphics[width=0.5\linewidth]{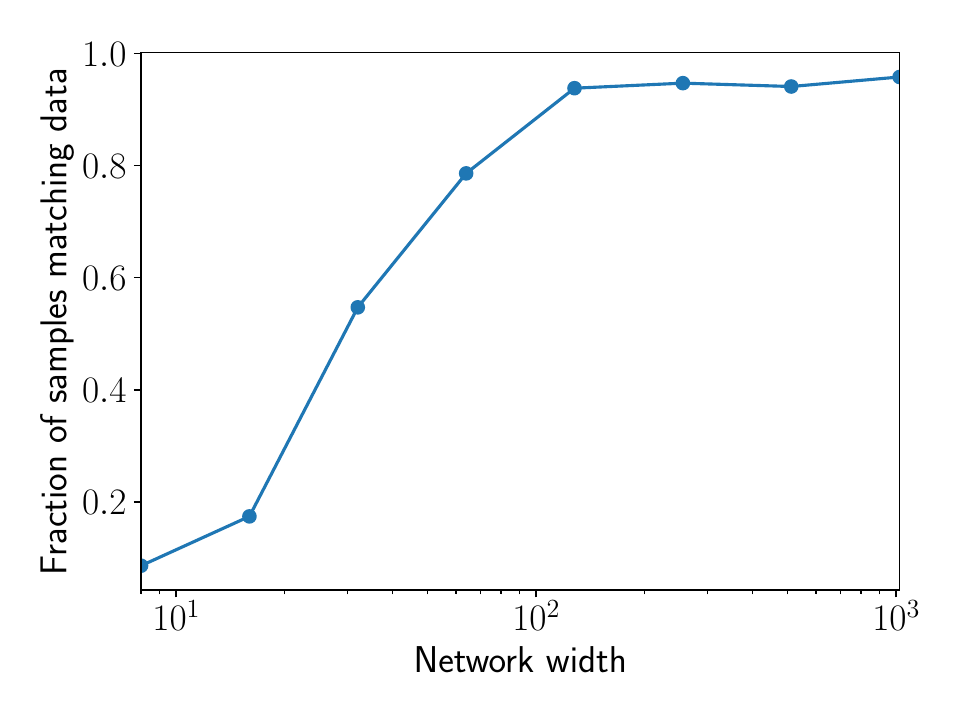}
    \caption{The effect of under-parameterization, as measured with the network width, on the memorization phenomenon. The setting is identical to that in Figure \ref{fig:NNwidth_samples}. \label{fig:NNwidth_collapse}}
\end{figure}

The previous experiments show that both early stopping of training and
under-parameterization act to regularize neural network based learning
of the score. In general, however, it may be difficult to know how 
to precisely implement early stopping or under-parameterization. Thus, it is interesting
to study the effect of using Tikhonov regularization in conjunction with
neural network learning, as introduced in Subsection \ref{ssec:Tikonov}. To this end, we learn a score function by minimizing the %
regularized loss $\Jregzero^N$ with
increasing values of the scalar $c \in \{0.001,0.01,0.1\}$ in the time-dependent regularization parameter $\Gamma(t)=c/\sigma^2(t) I_d$. Figure~\ref{fig:Tikonov_reg_samples} plots the generated samples found by integrating the reverse ODE with the learned score from $t = T$ to $t = 0$. We observe that the samples concentrate closer to the $N = 20$ data points in the empirical distribution for $c$ closer to zero, which corresponds to the unregularized score function. The emergence of memorization is also verified in Figure~\ref{fig:Tikonov_score_collapse} with the convergence of the learned score functions in red to the score of the Gaussian mixture model (GMM) in green as $c$ approaches zero. Figure~\ref{fig:NNTikhonov_collapse} illustrates the memorization quantitatively by computing the fraction of 1000 generated samples $x(0)$ that are at a distance smaller than threshold $\tau = 10^{-1}$ from a training data point, for each value of the regularization parameter.

\begin{figure}[!ht]
    \centering
    \includegraphics[width=0.32\linewidth]{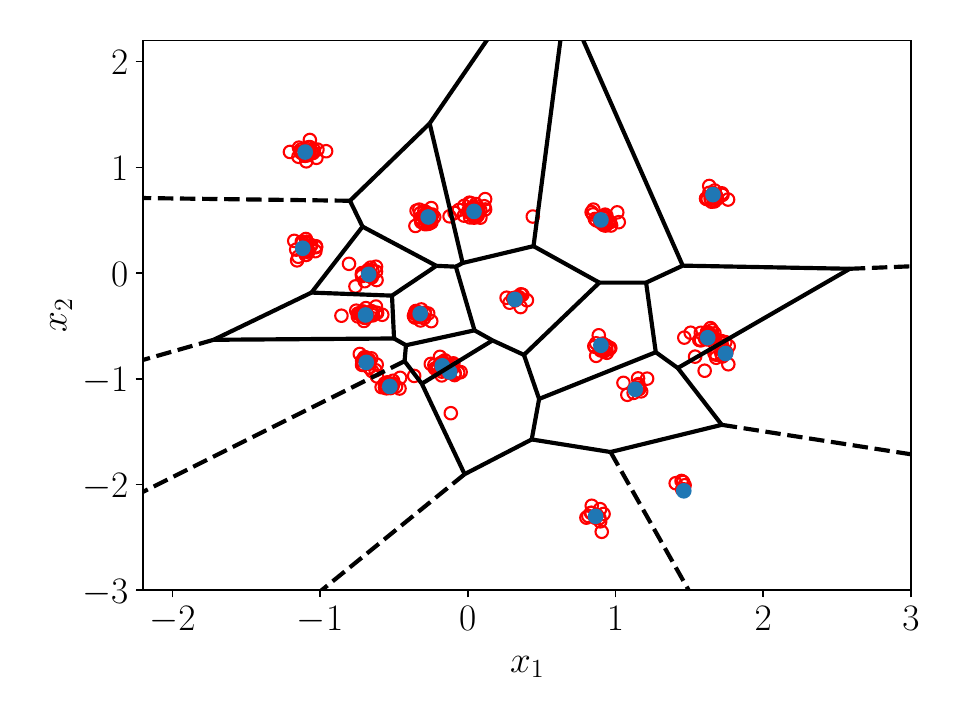}
    \includegraphics[width=0.32\linewidth]{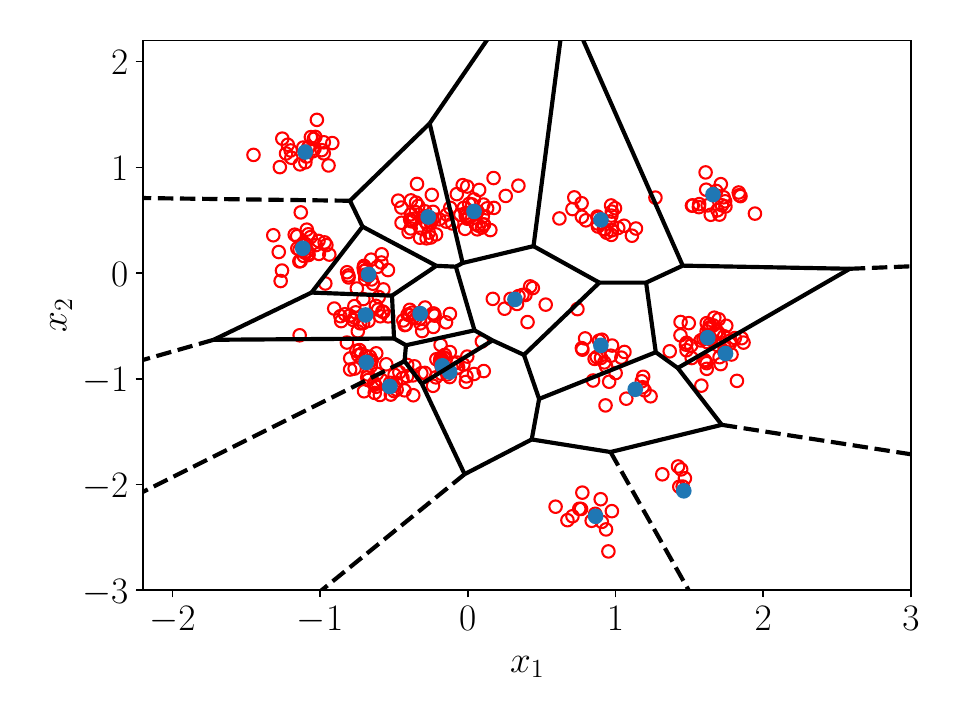}
    \includegraphics[width=0.32\linewidth]{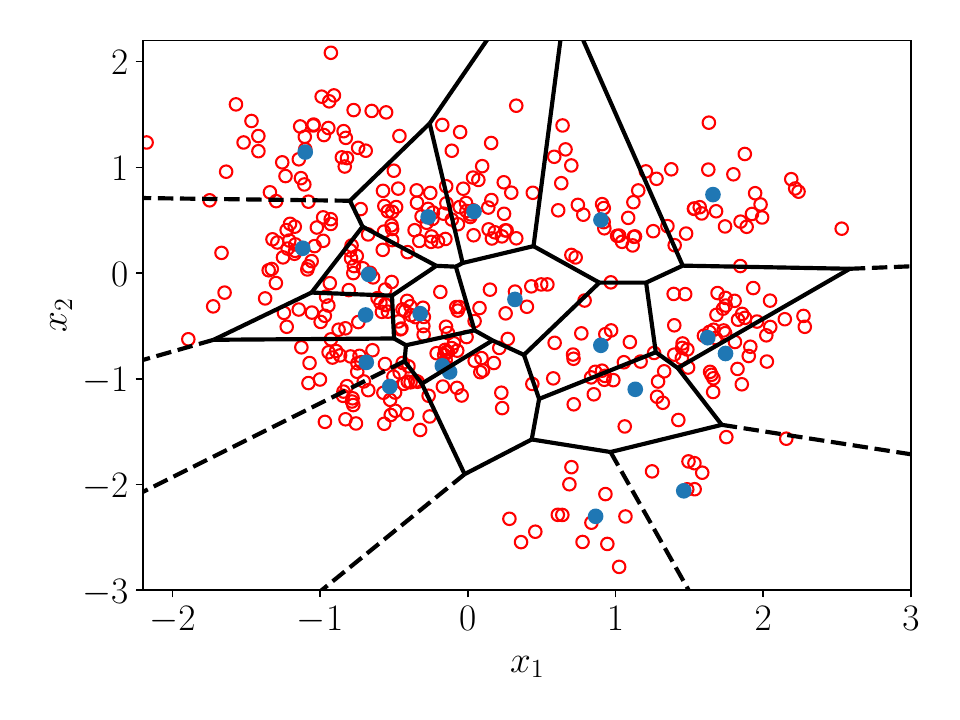}
    \caption{400 generated samples $x(0)$ found by integrating the reverse ODE with the learned score function using Tikhonov regularization. Blue dots comprise the data set; red circles are the generated points. Results are shown for the Tikhonov regularization parameter with $c = 0.001$ (\textit{left}), $c = 0.01$ (\textit{middle}) and $c = 0.1$ (\textit{right}). \label{fig:Tikonov_reg_samples}}
\end{figure}

\begin{figure}[!ht]
    \centering
    \includegraphics[width=0.32\linewidth]{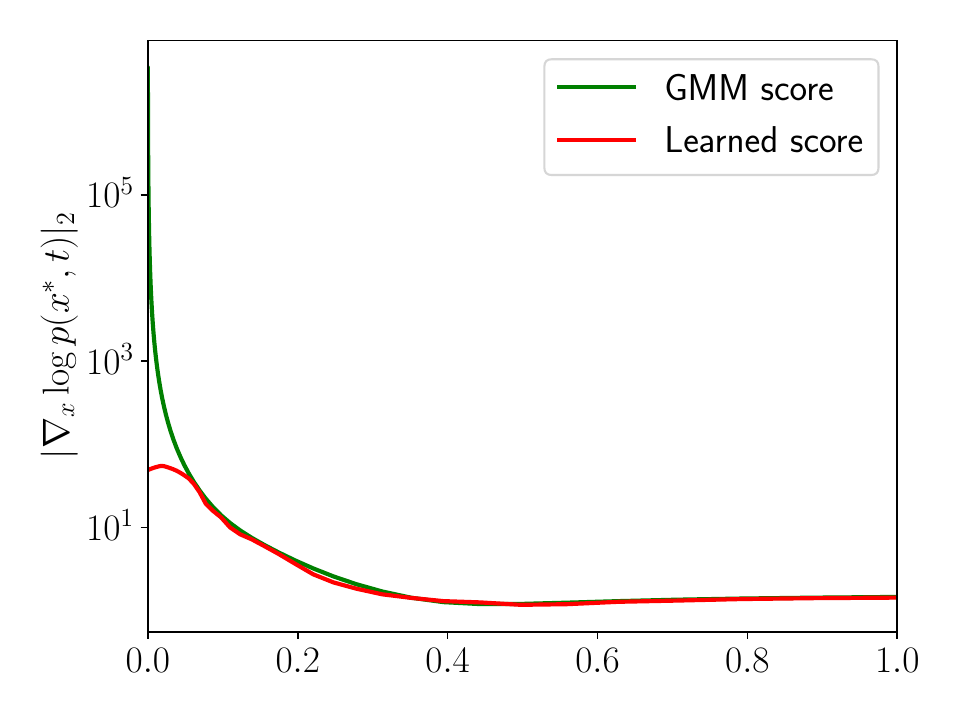}
    \includegraphics[width=0.32\linewidth]{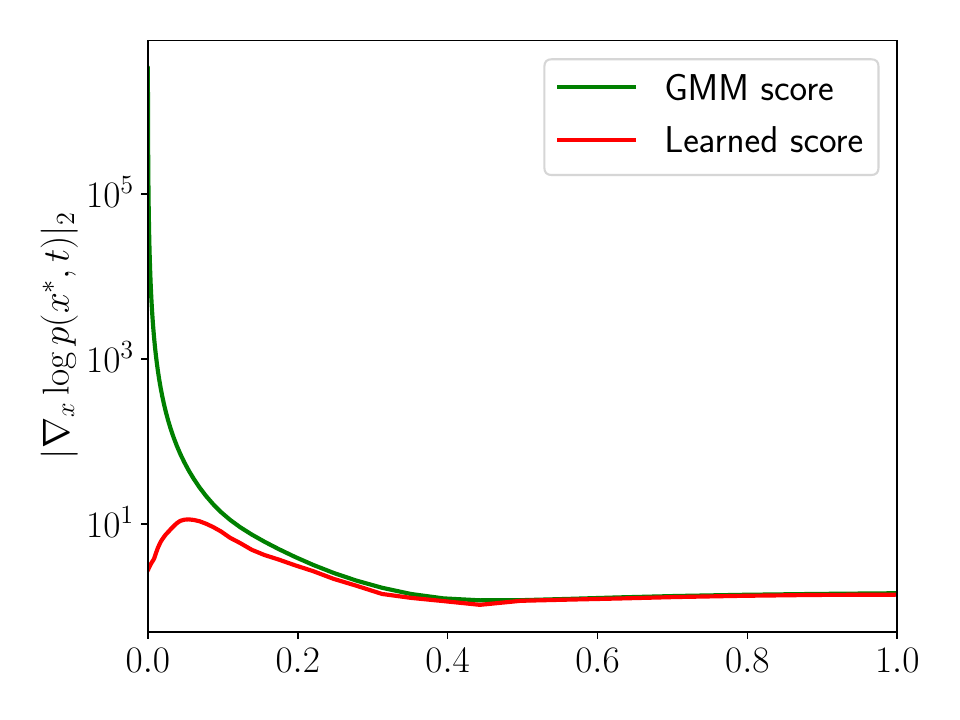}
    \includegraphics[width=0.32\linewidth]{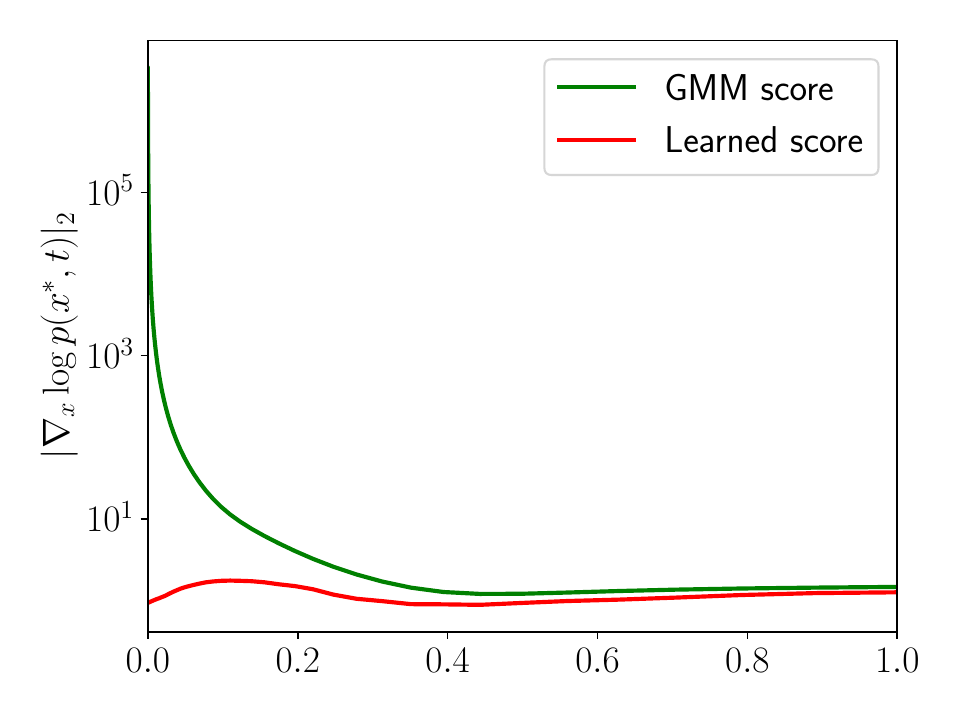}
    \caption{Time-dependence of the learned score function $s(x^\ast,t)$ at a fixed $x^\ast$ away from the data points using Tikhonov regularization for parameter values %
     $c = 0.001$ (\textit{left}), $c = 0.01$ (\textit{middle}) and $c = 0.1$ (\textit{right}). The learned score approaches the singular behavior of the Gaussian mixture near $t = 0$ for smaller values of $c$.  The setting is identical to that in Figure \ref{fig:Tikonov_reg_samples}. \label{fig:Tikonov_score_collapse}}
\end{figure}

\begin{figure}[!ht]
    \centering
    \includegraphics[width=0.5\linewidth]{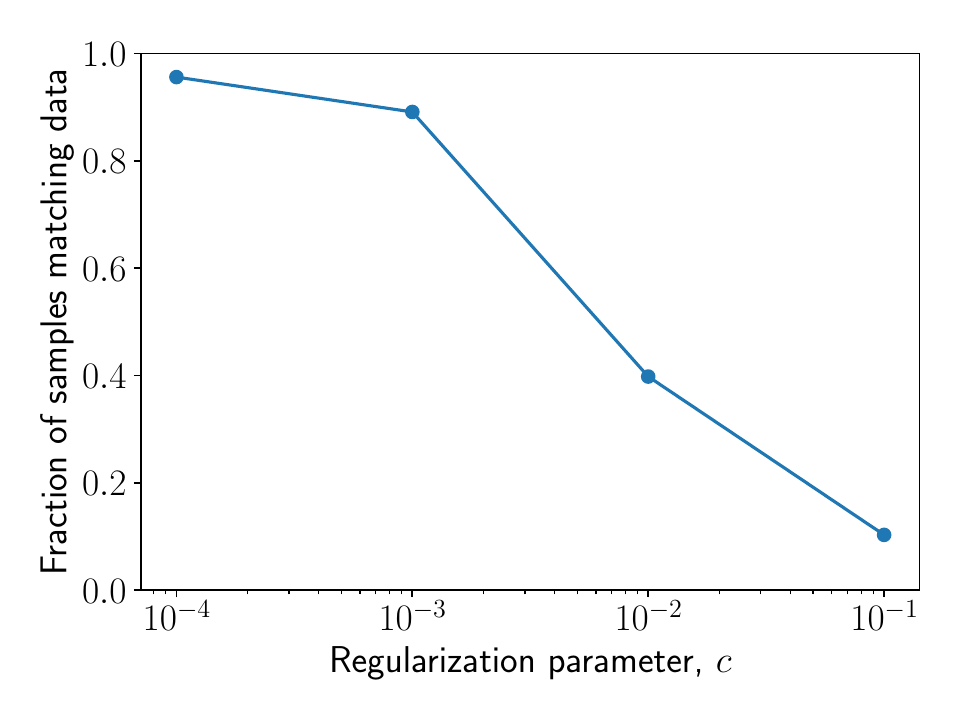}
    \caption{The effect of the regularization parameter on the memorization phenomenon with neural network approximation of the learned score using Tikhonov regularization. The setting is identical to that in Figure \ref{fig:Tikonov_reg_samples}. 
    \label{fig:NNTikhonov_collapse}}
\end{figure}

Lastly, we investigate the effect of the loss function on memorization. In this experiment, we approximate the score by training neural networks with the same configuration %
using one of the two loss functions: $\J_0^N$ and $\I_0^N$. %
Figure~\ref{fig:NNsingular_samples} plots the fraction of 1000 generated samples $x(0)$ for each learned score that are found at a Euclidean distance less than $\tau = 10^{-1}$ from a training data point. The left of Figure~\ref{fig:NNsingular_samples} illustrates memorization  as a function of network width when training for 200,000 epochs, while the right of Figure~\ref{fig:NNsingular_samples} illustrates memorization as a function of training time for a network of width 256. We observe that using the denoising loss $\I_0^N$, which explicitly incorporates the singular behavior near $t = 0$ in the approximate score, yields generated samples that are more concentrated near the training data for each network width and training time than using the loss $\J_0^N$. Thus, it may be advantageous to use the loss $\J_0^N$ to mitigate memorization for certain data distributions as discussed in Section~\ref{ssec:LTN}. %

\begin{figure}[!ht]
    \centering
    \includegraphics[width=0.48\linewidth]{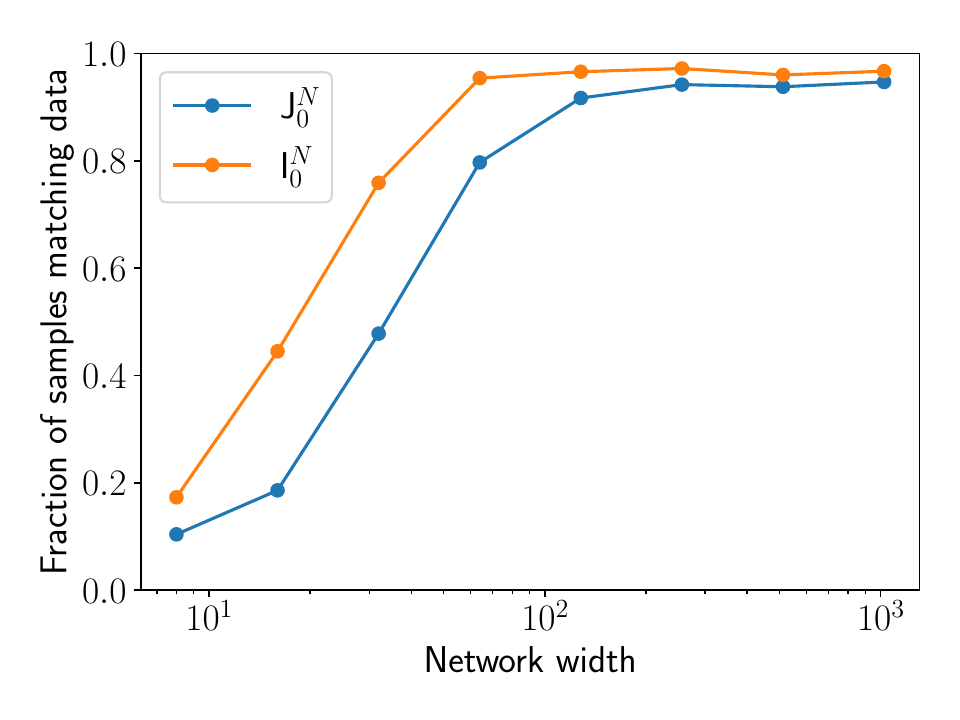}
    \includegraphics[width=0.48\linewidth]{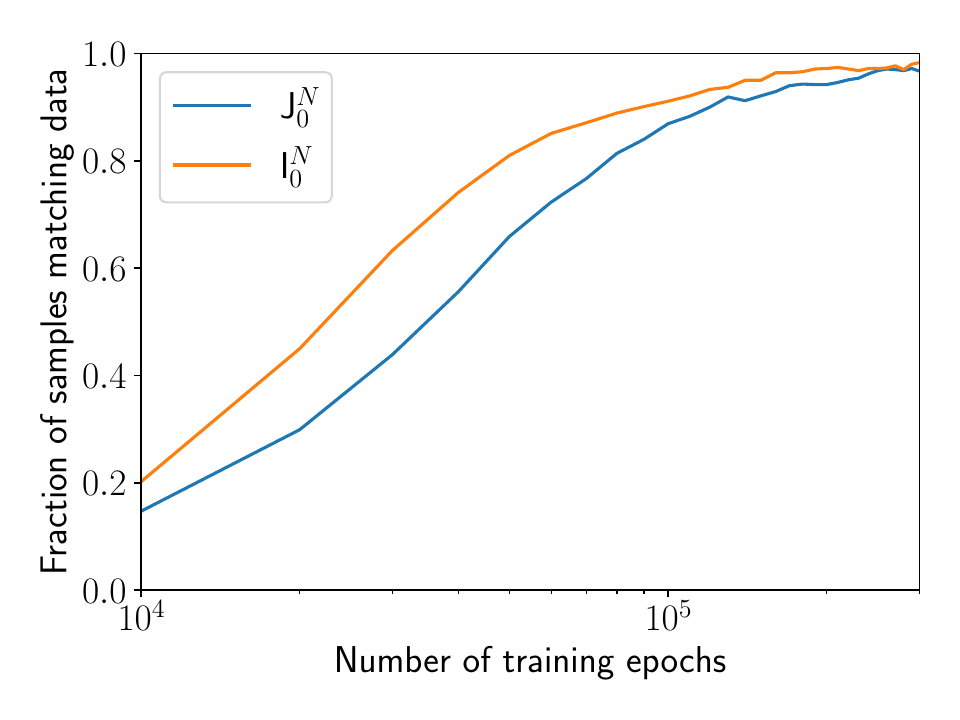}
    \caption{Comparison of the effect, on memorization, of using the using the score matching loss $\J_0^N$ in blue and denoising score matching loss $\I_0^N$ in orange; these yield nonsingular and singular score functions, respectively, by construction, when
    the learned score is parameterized using a neural network. The
    approaches are compared as a function of network width (\textit{left}) and the number of epochs used for training (\textit{right}). \label{fig:NNsingular_samples}}
\end{figure}

\subsection{Function Class Regularization: Image Dataset} \label{sec:NNrectangles}

The previous subsection focused on data sets in $\R^2$, facilitating a wide range of comparative studies of different regularizers. In this subsection we study a high-dimensional image dataset, but restrict ourselves to small data volume in order to
facilitate a wide range of experiments. We consider a target distribution $\pr_0$ over binary images $x_0^n$ of size $64 \times 64$, i.e., $x_0^n \in \{0,1\}^d$ for $d=4096$ after vectorization. The data-generating process for the images is defined by inserting a square of equal length and width taking value one in an array consisting of zeros. The initial training set consists of $N=2$ images (later we consider $N=8$), which has a square located at either the top-left corner or the bottom-right corner of the image. The two images in the training set are displayed in Figure~\ref{fig:Rectangles_TrainingData}. %

\begin{figure}[!ht]
\centering
\includegraphics[trim={2cm 3cm 1.5cm 3cm},clip,width=0.5\textwidth]{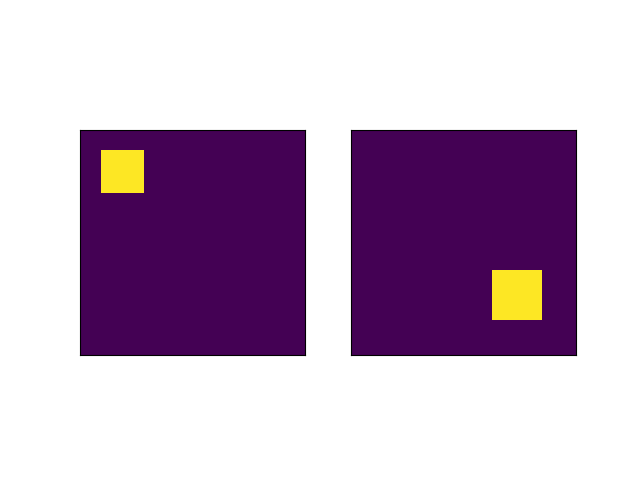}
\caption{Training data for rectangles dataset \label{fig:Rectangles_TrainingData}}
\end{figure}

We optimize the score function by minimizing the denoising score matching loss $\I_0^N$
considered in Section~\ref{ssec:LTN},
using the network implementation from~\cite{karras2022elucidating}. This model consists of a U-Net architecture with positional time embedding. We train with a batch size of 2 and investigate the performance with an increasing number of epochs. The reverse process for sampling is the SDE-based methodology presented in \cite{karras2022elucidating} using 40 backward steps. By using an SDE for data generation, we demonstrate that our memorization results, so far all in the the ODE setting, are also empirically verified in the SDE setting.

Figure~\ref{fig:rectangle_earlystopping} plots sets of 16 generated samples from the diffusion model, for varying numbers of epochs used in training. We observe that early stopping leads to the generation of novel images, not in the data set, but with characteristics of the data set; see, for example, the bottom entry of the third column in the images generated with 20,000 epochs. A similar observation is made in~\cite{aithal2024understanding} in terms of the smooth neural network with early training being capable of ``hallucinating'' new images that did not exist in the training dataset.

\begin{figure}[!ht]
\centering
\begin{center}
2000 epochs \hspace{4.5cm} 4000 epochs
\end{center}
\vspace{-0.2cm}
\includegraphics[trim={1cm 1cm 1cm 1cm},clip,width=0.45\textwidth]{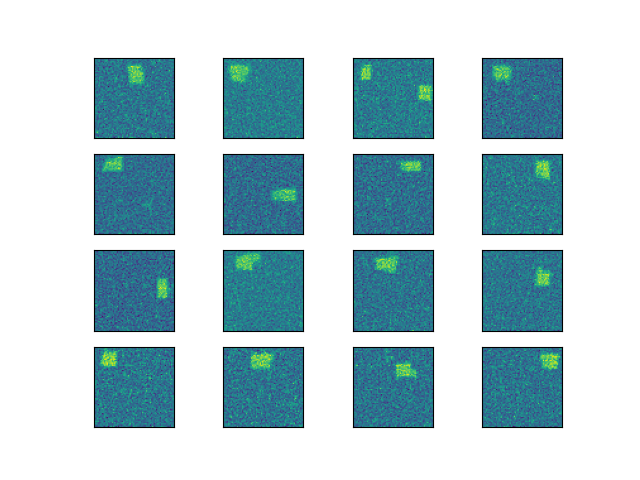}
\includegraphics[trim={1cm 1cm 1cm 1cm},clip,width=0.45\textwidth]{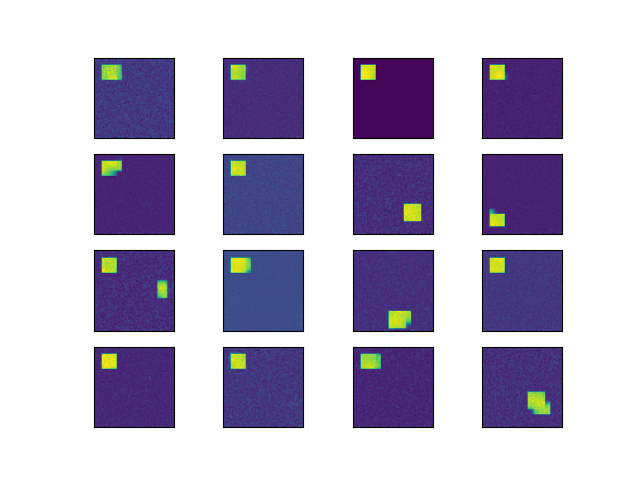}\\
\begin{center}
6000 epochs \hspace{4.5cm} 10,000 epochs
\end{center}
\vspace{-0.2cm}
\includegraphics[trim={1cm 1cm 1cm 1cm},clip,width=0.45\textwidth]{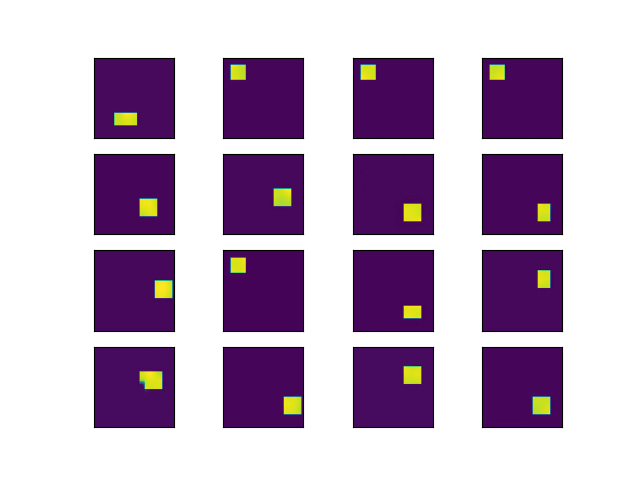}
\includegraphics[trim={1cm 1cm 1cm 1cm},clip,width=0.45\textwidth]{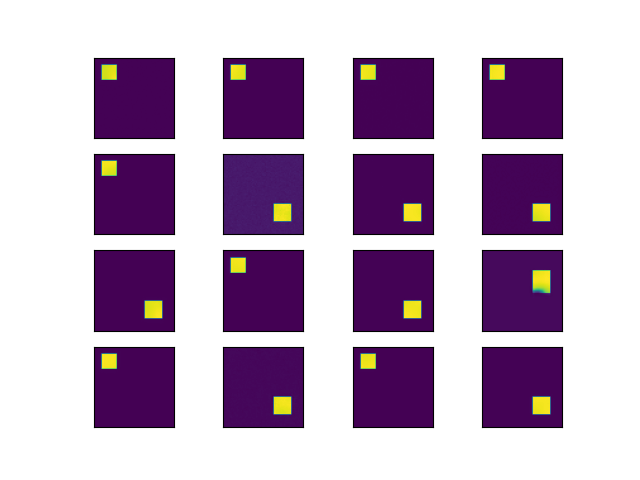}\\
\begin{center}
20,000 epochs \hspace{4.5cm} 50,000 epochs
\end{center}
\vspace{-0.2cm}
\includegraphics[trim={1cm 1cm 1cm 1cm},clip,width=0.45\textwidth]{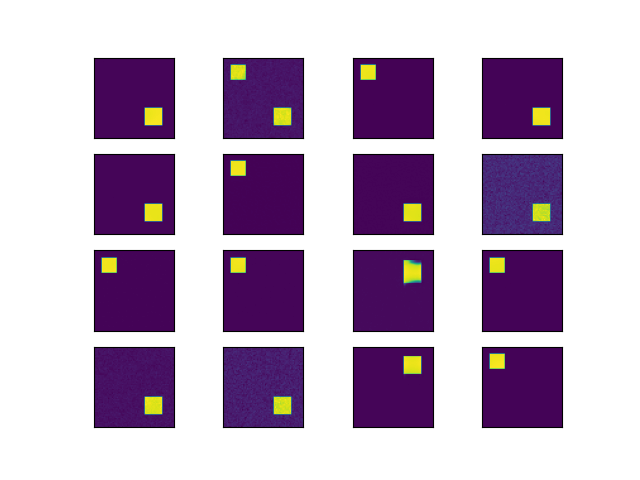}
\includegraphics[trim={1cm 1cm 1cm 1cm},clip,width=0.45\textwidth]{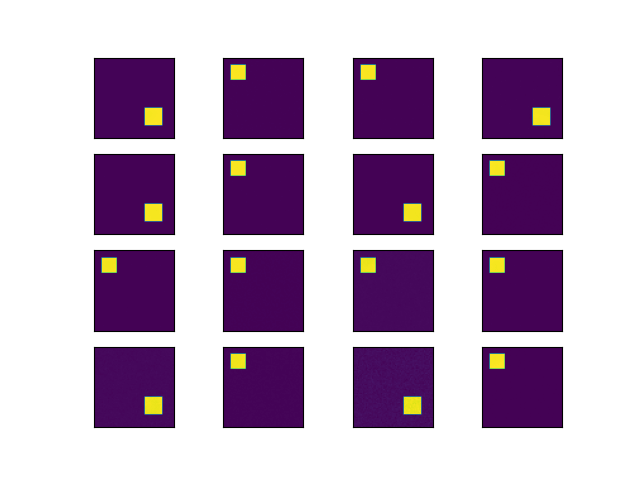}
\caption{Sixteen generated samples with fixed initial conditions and realizations of the noise increments in the reverse process after an increasing number of epochs. We observe that while the generative process creates ``novel'' and noiseless content at 6000, 10,000 and 20,000 epochs, at 50,000 epochs all generated samples match the two training samples. Thus, early stopping can mitigate memorization. \label{fig:rectangle_earlystopping}}
\end{figure}

In the last experiment, we investigate the effect of model capacity and under-parametrization of neural networks on memorization. Here, we vary the number of total model parameters from $\mathcal{O}(10,000)$ to $\mathcal{O}(10,000,000)$. After every 1000 optimization steps, we generate 100 samples from each model and compute the proportion of them that exactly match one of the $N = 2$ training points, i.e., the fraction of samples at a Euclidean distance of zero from $x_0^n$ for some $n$. To avoid small noise errors, we take all the generated samples and do a thresholding where every value of the generated image above 0.5 is mapped to 1 and every value below 0.5 is mapped to 0. In the left of Figure~\ref{fig:rectangles_parametrize}, we show the proportion of memorized images as a function of the optimization step for different model sizes. We observe that all models transition to perfect memorization if trained for a sufficient amount of time. However, increasing the number of model parameters results in this occurring faster during the optimization. We conjecture that this is because there is a higher chance of identifying a model configuration that matches the global minimizer given by the Gaussian mixture model, leading to memorization, when over-parameterized. Furthermore, we augment the size of the training data set to $N = 8$ by adding images containing a rectangle of different sizes positioned at a random location. In the right of Figure~\ref{fig:rectangles_parametrize}, we demonstrate that a similar memorization phenomenon occurs on the larger dataset with some networks showing the transition to memorization earlier in training. %

\begin{figure}[!ht]
\centering
\includegraphics[width=0.48\textwidth]{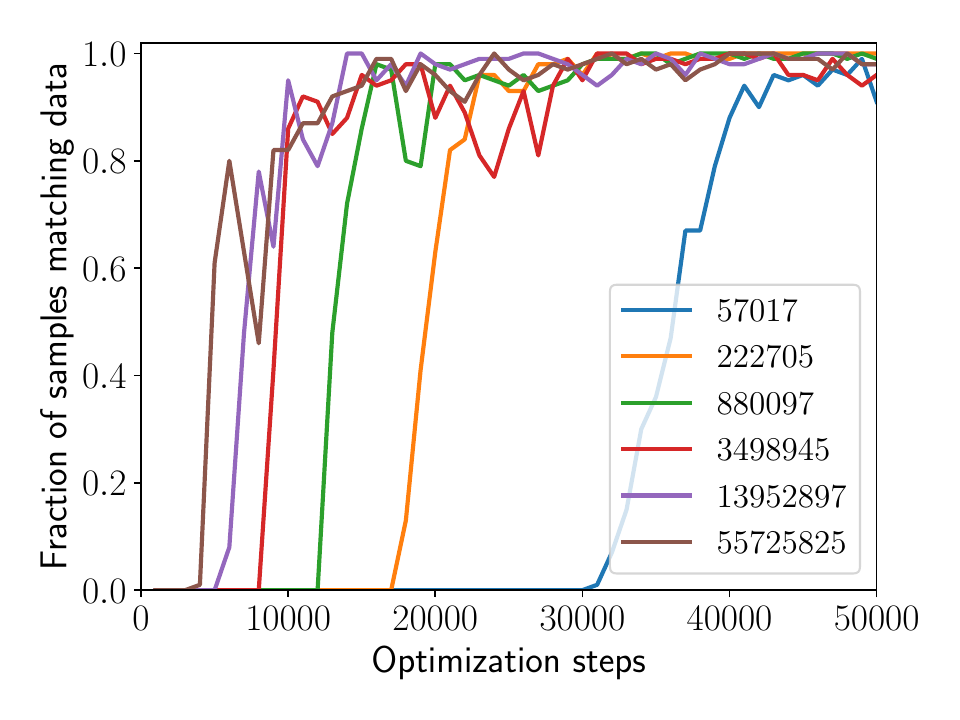}
\includegraphics[width=0.48\textwidth]{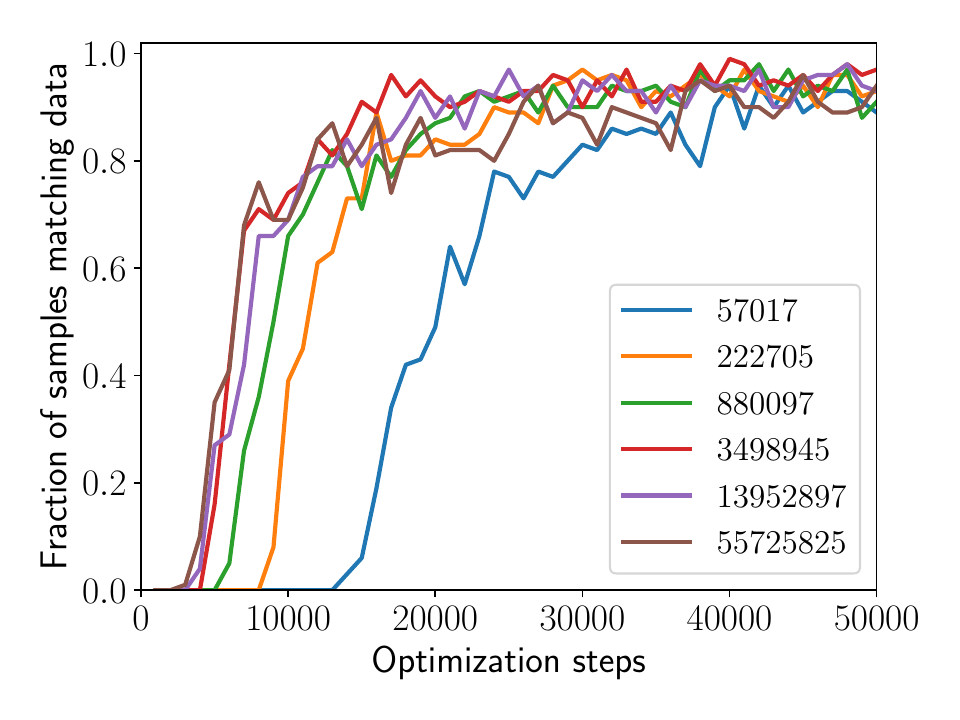}
\caption{Fraction of memorized samples on the rectangles dataset shown as a function of the number of optimization steps. The legend in each plot indicates the number of parameters in the U-Net model for the approximate score obtained by varying the number of channels in the first residual block. The left plot uses $N=2$ training samples while the right plot uses $N=8$.}
\label{fig:rectangles_parametrize}
\end{figure}

\section{Conditional Models} \label{sec:conditioning}

In addition to their use in unconditioned generation as described
in Section~\ref{sec:background}, score-based diffusion models are also used conditionally, to solve inverse problems. To this end, 
we introduce a joint probability distribution on the pair $(\bu_0,y) \in \R^d \times \R^m$, whose density we denote by $p_0(x_0,y)$. The aim of conditional generation
is, given only samples from $p_0(x_0,y)$, to make inference about $\bu_0$ from any given observation $y=y^\ast;$ specifically we aim to generate samples from the target conditional distribution $p_0(x_0|y^\ast)$. 

To sample from the target conditional we first introduce a forward (noising) process which creates the time-dependent distribution $p(x,t|y)$ for each $t \in [0,T]$ and $y$-\text{a.e.} with respect to $p_0(y)$, the $y-$marginal of $p_0(x_0,y)$; the process is
started at $t=0$ from $p(x,0|y)=p_0(x|y).$
Then, we introduce a reverse (denoising) process which
maps the resulting distribution at time $t=T$ back to the target at time $t=0.$
The score function required in the reverse process may be approximated, using
empirical data from $p_0(x_0,y)$. This enables us to 
approximately sample $x_0$ from $p_0(\cdot|y)$ for \textit{any} observation $y=y^\ast$. 
The score function is approximated using score matching. 

In Subsection~\ref{ssec:CSU} we adapt the forward (noising) and backward (denoising) processes from Section~\ref{sec:background}  
to the setting of conditional sampling. This is followed in Subsection~\ref{ssec:CSM} by a discussion of score matching for conditional sampling. The effect of
memorization, for conditional sampling, is described in Subsection~\ref{ssec:CMem}.

\subsection{Forward and Reverse Processes}
\label{ssec:CSU}

Let $x \in C([0,T],\R^d)$ satisfy the forward SDE %
\begin{equation} \label{eq:ForwardSDE_conditional}
  \frac{\dd \bu}{\dd t} = -\frac{1}{2}\beta(t)\bu + \sqrt{g(t)}\frac{\dd \bB}{\dd t}, \quad \bu(0) = x_0 \sim p_0(\cdot|y);  \end{equation}
here $\bB$ is a standard $\R^d$-valued Brownian motion independent of $p_0(\cdot|y)$. Since $\beta$ and $g$ do not depend on the observation variable $y$, it follows that
$p(x,t|x_0,y)=p(x,t|x_0)$ and hence that
\begin{equation} \label{eq:ForwardDist_conditional}
p(x,t|y) = \int p(x,t|x_0)p_0(x_0|y)\dd x_0.
\end{equation}
Note that $p(x,t|x_0)$ is Gaussian; in particular,
$p(x,t|x_0)$ has the same form as for the forward processes defined in Section~\ref{sec:score_matching}.

Analogously to Subsection~\ref{ssec:R},
the backward process that follows the same law as the forward process at each $t \in [0,T]$ is given, for any $\alpha_2 \ge 0,$ by
\begin{equation*} 
    \frac{\dd \bu}{\dd t} = -\frac{1}{2}\beta(t)\bu - \frac12(1+\alpha_2)g(t)\nabla_x \log p(\bu,t|y) + \sqrt{\alpha_2 g(t)}\frac{\dd \bBrev}{\dd t}, \quad \bu(T) = \bu_T  \sim p(\cdot,T|y).
\end{equation*}
 As in the unconditioned setting, this reduces to the following ODE, if $\alpha_2=0$,
\begin{equation} \label{eq:BackwardODE_IP_Cond}
    \frac{\dd \bu}{\dd t} = -\frac{1}{2}\beta(t)\bu - \frac{g(t)}{2}\nabla_x \log p(\bu,t|y) , \quad \bu(T) = \bu_T  \sim p(\cdot,T|y).
\end{equation} The flow map of this ODE defines a deterministic transport that samples from $p_0(x_0|y)$, given a sample $x(T) \sim p(x,T|y)$ drawn from the distribution of the forward process at time $t=T$. In practice, the forward process is chosen so that $p(x,T|y) \approx \mathfrak{g}_T$ where $\mathfrak{g}_T$ is a Gaussian reference distribution that is independent of the initial condition of the forward process. Then, a generative model is given by sampling $x(T) \sim \mathfrak{g}_T$ and evolving the reverse SDE or ODE back to time $t=0.$ Furthermore, in practice, the conditional
score $\nabla_x \log p(\bu,t|y)$ is approximated, using only empirical
data drawn from $p_0(x_0,y)$, by the function $\score(\bu,y,t)$ depending on $\bu,t$ and the observation $y$. We now discuss the construction of the score function.

\subsection{Conditional Score Matching}
\label{ssec:CSM}
 
To identify the score function for the conditional distributions $p(\bu,t|y)$ for any $y \in \R^m$ and $t \in [0,T]$, we consider the \textit{conditional score matching loss}
\begin{align*}\Jcond(\score) &= \int_0^{T} \lambda(t) \mathbb{E}_{y \sim \pr_0(\cdot)}\mathbb{E}_{x \sim \pr(\cdot,t|y)}|\score(x,y,t) - \nabla_\bu \log \pr(\bu,t|y)|^2 \dd t\\
&= \int_0^T \lambda(t)\mathbb{E}_{y \sim \pr_0(\cdot)}\mathbb{E}_{\bu_0 \sim \pr_0(\cdot|y)}\mathbb{E}_{\bu \sim \pr(\cdot,t|x_0)}|\score(\bu,y,t) - \nabla_x \log p(\bu,t|y)|^2 \dd t.
\end{align*}
In the first expression, note the similarity with \eqref{eq:score_matching_loss}:
the only difference is that the exact and approximate scores are all conditioned on $y$
and then everything is averaged over the marginal $p_0(\cdot)$ on $y.$ To obtain
the second expression, we use the law of total expectation for $x \sim p(\cdot,t|y)$ by conditioning on $x_0$, together with the
fact that $\pr(\cdot,t|x_0,y)=\pr(\cdot,t|x_0)$. Note that the loss $\Jcond(\cdot)$ is minimized at 
$s(\bu,y,t) = \nabla_\bu \log \pr(\bu,t|y)$.

\begin{remark} \label{rem:back}
While the expectations over $p_0(y)$ and $p_0(x_0|y)$ require knowledge of the target conditionals, the loss is equivalently expressed using expectation over the joint distribution $p_0(x_0,y)$. The latter expectation is amenable to sample approximation in an inverse problem by drawing $x_0$ from the marginal distribution $p_0(\cdot)$ (the prior) and the observation $y$ from the conditional $p_0(\cdot|x_0)$ (the likelihood):
\begin{subequations}
\begin{align*}
    \Jcond(\score) 
    & = \int_0^T \lambda(t)\mathbb{E}_{(\bu_0,y) \sim \pr_0(\cdot,\cdot)}\mathbb{E}_{\bu \sim \pr(\cdot,t|x_0)}|s(\bu,y,t) - \nabla_x \log p(\bu,t|y)|^2 \dd t\\
    & = \int_0^T \lambda(t)\mathbb{E}_{\bu_0 \sim \pr_0(\cdot)}\mathbb{E}_{y \sim \pr_0(\cdot|x_0)}\mathbb{E}_{\bu \sim \pr(\cdot,t|x_0)}|s(\bu,y,t) - \nabla_x \log p(\bu,t|y)|^2 \dd t.
\end{align*}
\end{subequations}
\end{remark}

To implement score matching in practice, we employ an empirical approximation 
as we did to derive the empirical objective~\eqref{eq:score_matching_lossN}
for the unconditioned score. In this conditioned setting, it is desirable to work
with samples from the joint distribution $\pr_0(\cdot,\cdot).$ To this end, 
let $p_0^N(x_0,y)$ denote the empirical distribution generated by the data pairs 
$\{(x_0^n,y^n)\}_{n=1}^N$, assumed to be drawn i.i.d.\thinspace from $\pr_0(\cdot,\cdot):$
$$p_0^N(x,y) = \frac{1}{N} \sum_{n=1}^N \delta_{(x_0^n,y^n)}.$$

\begin{definition}
    \label{def:ref}
    Given $N$ paired samples $\{(x_0^n,y^n)\}_{n=1}^N$, let $\mathcal{I}_j \subseteq \{1,\dots,N\}$ denote the indices sharing the same observation $y^j$. That is, $\mathcal{I}_j = \{n: \text{s.t. } y^n = y^j\}$. Let $|\mathcal{I}_j|$ denote
    the cardinality of this set.
\end{definition}

Let $p^N(x,t|y)$ denote the probability distribution of the form in~\eqref{eq:ForwardDist_conditional} with $p_0(\cdot|y)$ replaced by $p_0^N(\cdot|y):$ 
\begin{equation*} %
p^N(x,t|y) = \int p(x,t|x_0)p_0^N(x_0|y)\dd x_0.
\end{equation*}
Note that this function is defined only for $y$ in the support of $p_0^N(y)$: the marginal data set $\{y^n\}_{n=1}^N$. It is undefined for $y$ not in this set.
For an element $y^j$ from the marginal data set, the measure $p_0^N(x_0|y^j)$ is a sum of equally weighted Dirac masses:
$$p_0^N(x_0|y^j)=\frac{1}{|\mathcal{I}_j|}\sum_{\ell \in \mathcal{I}_j} \delta_{x_0^\ell}.$$
Thus
\begin{equation*} %
p^N(x,t|y^j) = \frac{1}{|\mathcal{I}_j|}\sum_{\ell \in \mathcal{I}_j} p(x,t|x_0^\ell).
\end{equation*}

We consider now the approximation $\Jcond^N$ of the loss $\Jcond$, defined by
\begin{subequations} \label{eq:score_matching_lossN2}
\begin{align}
  \Jcond^N(\score) &= \int_0^{T} \lambda(t) 
  \mathbb{E}_{(x_0,y) \sim p_0^N(\cdot,\cdot)} \mathbb{E}_{\bu \sim \pr(\cdot,t|x_0)} |\score(\bu,y,t) - \nabla_\bu \log \pr^N(\bu,t|y)|^2 \dd t.
  \end{align}
\end{subequations}
This loss function is minimized by choosing $\score(\bu,y,t)=\nabla_\bu \log \pr^N(\bu,t|y).$
Again this function is defined only for  $y$ in  the marginal data set 
$\{y^n\}_{n=1}^N$. It is undefined for $y$ not in this set.

As in the unconditioned setting, and following the same procedure described in Section~\ref{ssec:prac} to derive objective~\eqref{eq:OFGS}, we introduce the \textit{conditional denoising score matching} loss
\begin{align*}
\Jcondzero(\score)
&= \int_0^T \lambda(t)\mathbb{E}_{y \sim \pr_0(\cdot)}\mathbb{E}_{\bu_0 \sim \pr_0(\cdot|y)}\mathbb{E}_{\bu \sim \pr(\cdot,t|x_0)}|\score(\bu,y,t) - \nabla_x \log p(\bu,t|\bu_0)|^2 \dd t\\
& = \int_0^T \lambda(t)\mathbb{E}_{(x_0,y) \sim \pr_0(\cdot,\cdot)}\mathbb{E}_{\bu \sim \pr(\cdot,t|x_0)}|\score(\bu,y,t) - \nabla_x \log p(\bu,t|\bu_0)|^2 \dd t,
\end{align*}
where $\nabla_x \log p(\bu,t|\bu_0)$ is the known conditional score of the forward process (since $p(\bu,t|\bu_0)$ is Gaussian). Expressing the loss functional in terms of the forward process conditioned
on an initial value $x_0$ leads, as in the unconditioned
case, to straightforward implementation 
using only knowledge of the Gaussian marginals of this process. 
The first expression for $\Jcondzero$ may be derived from \eqref{eq:OFGS} by conditioning on $y$
and then averaging over the marginal $p_0(\cdot)$ on $y,$
using the fact that $\pr(\cdot,t|x_0,y)=\pr(\cdot,t|x_0).$

The following proposition, analogous to~\Cref{p:ifKf},
relates the conditional score matching losses by the constant $K_{\textsf{cond}}$ given by
\begin{align*}
K_{\textsf{cond}} \coloneqq & \int_0^T \lambda(t) \mathbb{E}_{y \sim \pr_0(\cdot)}\mathbb{E}_{x_0 \sim \pr_0(\cdot|y)}\mathbb{E}_{x \sim \pr(\cdot,t|\bu_0)}|\nabla_x \log \pr(\bu,t|\bu_0)|^2 \dd t \\ 
&\quad\quad\quad\quad\quad\quad\quad\quad-\int_0^T \lambda(t) \mathbb{E}_{y \sim \pr_0(\cdot)} \mathbb{E}_{x \sim \pr(\cdot,t|y)} |\nabla_x \log \pr(\bu,t|y)|^2 \dd t.
\end{align*}
The proof is structurally identical to that of~\Cref{p:ifKf}, using additional
averaging over $y$ from the marginal $p_0(\cdot)$ on $y$, and is hence it is not repeated.

\begin{proposition} \label{p:ifKfC} 
Assume constant $K_{\textsf{cond}}$ is finite. Then, 
$$\Jcond(\score) = \Jcondzero(\score) - K_{\textsf{cond}}.$$
Thus, the minimizers of $\Jcond$ and $\Jcondzero$ coincide and hence the minimizer of $\Jcondzero$ is given by $\score(x,t,y) = \nabla_x \log p(x,t|y)$.
\end{proposition}

\subsection{Memorization For Conditional Sampling}
\label{ssec:CMem}

Given $N$ joint data samples $\{(x_0^n,y^n)\}_{n=1}^N \sim p_0(x_0,y)$, we now define the empirical version of the conditional denoising score matching loss:
\begin{subequations} \label{eq:score_matching_lossN2_cond}
\begin{align}
\Jcondzero^N(\score) &=
\int_0^T \lambda(t)\mathbb{E}_{(x_0,y) \sim \pr_0^N(\cdot,\cdot)}\mathbb{E}_{\bu \sim \pr(\cdot,t|x_0)}|\score(\bu,y,t) - \nabla_x \log p(\bu,t|\bu_0)|^2 \dd t\\
&= \frac{1}{N} \sum_{n=1}^N \int_0^T \lambda(t) \mathbb{E}_{\bu \sim \pr(\cdot,t|x_0^n)}|\score(\bu,y^n,t) - \nabla_x \log p(\bu,t|\bu_0^n)|^2 \dd t.
\end{align}
\end{subequations}
The following theorem provides a result analogous to Theorem~\ref{thm:empirical_score_unconditional} for the optimal score function in the conditional setting given an empirical data distribution. The proof can be found in Appendix~\ref{ssec:wnc_conditional_proof}. Recall Definition \ref{def:ref} for notation employed in
the statement and proof.
\begin{theorem} \label{thm:empirical_score_conditional} 
Let $y^j$ be an element of the marginal data set $\{y^n\}_{n=1}^N$.
Then, the score function that minimizes the empirical loss function~\eqref{eq:score_matching_lossN2_cond}, for any weighting $\lambda \in \Lambda$, has the form
$$\score^N(x,y^j,t) \coloneqq \frac{1}{|\mathcal{I}_j|} \sum_{n \in \mathcal{I}_j} \frac{(x - m(t)x_0^n)}{\sigma^2(t)} \w_n(x,t).$$
Here $m(t),\sigma^2(t)$ are defined by~\eqref{eq:msig} and $\w_n(x,t)$ are the normalized Gaussian weights defined in~\eqref{eq:normalized_Gaussian_weights}.
For $y$ outside the marginal data set, the optimal $\score^N(x,y,t)$ is undefined.
\end{theorem}

In both the unconditioned and conditional settings, the empirical loss is minimized at a score function that is a weighted linear combination of the data samples defining $p_0^N$ with Gaussian weights. In the conditional setting, this linear combination may only contain a single sample for each observation in the joint empirical distribution. Moreover, the true minimizer is undefined at observations outside the support of $p_0^N(y)$: the marginal dataset $\{y^n\}_{n=1}^N$. Approximations to the score in this setting rely on the smoothness of the approximation class to define the behavior of the score at observations $y \neq y^n$ for $n=1,\dots,N$. Given that the optimal empirical score functions have the same form in unconditioned and conditioned setting, the reverse ODE with the Gaussian mixture score function will also exhibit the memorization behavior shown in Section~\ref{sec:singularity_data_diffusion} for conditional sampling.

\section{Discussion and Future Work}
\label{sec:conclusions}

This work provides an inverse problems perspective on the memorization behavior of score-based diffusion models. Our main theoretical result uses a dynamical systems
analysis to show that trajectories of the reverse process with the optimal empirical score have limit points, for any initial condition, that are either in the empirical dataset or on the boundaries between the Voronoi cells defined by the data. Moreover, this behavior is present in both unconditioned and conditional settings. While the analysis is conducted for the reverse ODE, a similar analysis for the reverse SDE, that shares the same law, would also be of interest. In our reverse ODE analysis it remains to show if the convergence to the Voronoi boundaries is a set of measure zero with respect to the distribution for the initial condition; the numerical results  strongly suggest the trajectories will converge to one of the data points with probability one.

To mitigate memorization with score-based generative models, it is necessary to regularize the inverse problem for learning the score functions from data. Here, we investigate explicit Tikhonov regularization and relate it to early stopping of the unregularized dynamics. Although we demonstrated the regularization properties of the asymptotically consistent score estimator developed in~\cite{wibisono2024optimal}, a direction of interest is to develop other regularization techniques that mitigate memorization and come with consistency guarantees on the resulting sample generation process.
We also demonstrated empirically that neural network based regularization, arising 
from restricting the approximation space for the scores, and early stopping of the training process, have a regularizing effect on the generated samples.

While the regularization techniques we consider prevent collapse onto the training dataset, they do not necessarily result in the generation of ``novel'' content. In
practice, methods which are able to generate new data points, with the characteristics of
the training data set, are desirable. Our experiments with images of rectangles demonstrate
interesting behavior in this context, and further analysis would be of interest.

A setting where principled regularization is especially needed is that of conditional models. This leads to a learned score function, in the empirical setting, which is completely undefined for observations $y$ outside the empirical dataset $\{y^n\}_{n=1}^N$. As a result, the model behavior for conditional sampling with observations outside of the training dataset is entirely determined by the approximation class used to model the score function's smoothness with respect to $y$. Future work might consider explicit regularizers for the score matching loss that penalize the smoothness of the function $y \mapsto \score(x,y,t)$ for each $(x,t)$. 

Lastly, future work might naturally focus on investigating the memorization behavior and regularization methods for other generative models based on dynamical systems including flow-matching, Schrodinger bridges, and conditional sampling methods for personalized image generation such as DreamBooth~\citep{ruiz2022dreambooth}.

\section{Acknowledgments}
RB and AS acknowledge support from the Air Force Office of Scientific Research MURI on “Machine Learning and
Physics-Based Modeling and Simulation” (award FA9550-20-1-0358), and a Department of Defense (DoD) Vannevar Bush Faculty Fellowship (award N00014-22-1-2790) awared to AS. RB is also grateful for support from the von K\'{a}rm\'{a}n instructorship at Caltech. AD and AAO would like to acknowledge the support of ARO Grant AWD-00009144. The authors also acknowledge the Center for Advanced Research Computing %
at the University of Southern California for providing computing resources that have contributed to the research results reported herein.

\section*{References}
\bibliographystyle{plain}
\bibliography{references}

\appendix

\section{Proofs of Lemmas Underlying Main Results} \label{app:proofs}

\subsection{Singularity as $t \to 0^+$}

\begin{lemma} \label{lemma:IntegralBlowup}
    Let \(h \in C^{1} ([0,T])\) with \(h > 0\) on \((0,T]\) then, for any \(\epsilon > 0\),
    \[\int_0^\epsilon \frac{h(t)}{\int_0^t h(s) \dd s } \: \dd s = \infty.\]
\end{lemma}

\begin{proof}[Proof of Lemma~\ref{lemma:IntegralBlowup}]
    Note that, since \(h > 0\) on \((0,T]\), we have that \(\int_0^t h(s) \dd s > 0\) for any \(t > 0\). Taylor expanding,
    \[h(t) = h(0) + h'(0)t + \mathcal{O}(t^2).\]
    Therefore,
    \[\frac{h(t)}{\int_0^t h(s) \dd s } = \frac{h(0) + h'(0)t + \mathcal{O}(t^2)}{h(0)t + \frac{1}{2}h'(0)t^2 + \mathcal{O}(t^3)} = \mathcal{O}(t^{-1})\]
    and the result follows.
\end{proof}

\subsection{Variance Exploding Setting}

\begin{proof}[Proof of Lemma~\ref{thm:existence_uniqueness}]
We recall that $\ys_N(y;s)=\sum_{n=1}^N x_0^n\ww_n(y,s)$ where $\ww_n$ are weights in $[0,1]$ that satisfy $\sum_{n=1}^N \ww_n(y,s) \equiv 1$ where 
\begin{equation*}
    \ww_n(y,s) \coloneqq \frac{\wwtilde_n(y,s)}{\sum_{m=1}^N \wwtilde_m(y,s)}, \qquad 
    \wwtilde_n(y,s) \coloneqq \exp\left(-\frac{\|y-x^n_0\|^2}{2e^{-2s}}\right). 
\end{equation*}

We obtain obtain priori bounds on the solution of \eqref{eq:mf_ode4}
starting from the prescribed initial condition, and assuming the solution exists. Using an integrating factor, one can deduce that $y(\cdot)$ solves the integral equation
$$y(s)=e^{-s-\ln \sigma(T)}x_T + \int_{-\ln \sigma(T)}^s e^{\tau-s}\ys_N\bigl(y(\tau),\tau\bigr) d\tau.$$
From Lemma \ref{lem:solution_compact_set_general} with $z=0$ we have that
\begin{equation}
    \label{eq:bound}
    \sup_{s \in [-\ln \sigma(T),\infty)}\|y(s)\| \le R(x_T,N) \coloneqq {\rm max}\Bigl(\|x_T\|, \max_{1 \le n \le N} \|x_0^n\|\Bigr)
\end{equation}
and that 
$$\limsup_{s \to \infty}\|y(s)\| \le  \max_{1 \le n \le N} \|x_0^n\|.$$

To prove existence and uniqueness of a solution 
$y \in C^1\bigl([-\ln \sigma(T),\infty);H)$ it suffices (see \cite[Theorem 3.1]{hale2009ordinary}) to prove existence
of a solution $y \in C^1([-\ln \sigma(T),\tau);H)$ for every fixed $\tau>-\ln \sigma(T).$
For this it suffices to show that $\ww_n(\cdot,\cdot)$ is continuous on
$H \times [-\ln \sigma(T),\tau]$ and that $\ww_n(\cdot,t)$ is
Lipschitz continuous, uniformly in $t \in [-\ln \sigma(T),\tau].$ 
In fact, because of the
bound \eqref{eq:bound}, it suffices to show these continuity and Lipschitz continuity
properties on $B \coloneqq B_{H}\bigl(0,R(x_T,N)\bigr).$ 
Continuity of $\wwtilde_n(\cdot,\cdot)$ on $B \times [-\ln \sigma(T),\tau]$ follows from continuity of the norm on $H$;
furthermore $\sum_{m=1}^N \wwtilde_m(\cdot,\cdot)$ is bounded from below by a positive constant on $B \times [-\ln \sigma(T),\tau]$.
Thus $\ww_n(\cdot,\cdot)$ is continuous on $B \times [-\ln \sigma(T),\tau]$. To prove the Lipschitz property it suffices
to prove that $\wwtilde_n(\cdot,t)$ is Lipshitz, again using that $\sum_{m=1}^N \wwtilde_m(\cdot,\cdot)$ is bounded from below.
Since $\exp(-|x|)$ has Lipschitz constant one on $\R$ it follows that
\begin{align*}
 |\wwtilde_n(y,s)-\wwtilde_n(y',s)| & \le \frac{e^{2\tau}}{2}\bigl(\|y-x_0^n\|^2-\|y'-x_0^n\|^2\bigr)\\ 
 & \le \frac{e^{2\tau}}{2}\|y-y'\|\|y+y'-2x_0^n\|\\
 & \le 2e^{2\tau}R(y_0,N)\|y-y'\|.
\end{align*}
The desired result follows.
\end{proof}

Lemma~\ref{lem:solution_compact_set_general} shows that we may potentially improve on the bounds in Lemma~\ref{thm:existence_uniqueness}, 
by optimizing over $z$; such improvements, however, are not needed for our proof of Theorem~\ref{thm:memorization}.

\begin{lemma} \label{lem:solution_compact_set_general} 
Let Assumptions \ref{asp:stand} and \ref{asp:abg} hold.
Consider \eqref{eq:mf_ode4} initialized at $y\bigl(-\ln \sigma(T)\bigr)=x_T.$
Assume that equation \eqref{eq:mf_ode4} has a unique solution $y \in C^1\bigl([-\ln \sigma(T),\infty);H\bigr).$  Then,
for any point $z \in H$, solution to~\eqref{eq:mf_ode4} satisfies 
\begin{subequations}
\label{eq:DistanceLimSup_VP}
\begin{align}
    \|y(s) - z\| &\leq \max\{\|x_T - z\|,\|x_0^1 - z\|,\dots,\|x_0^N - z\|\} \; \forall s \geq 0\\
  \limsup_{s \rightarrow \infty} \|y(s) - z\| &\leq \max\{\|x_0^1 - z\|,\dots,\|x_0^N - z\|\}.  
\end{align}
\end{subequations}
\end{lemma} 

\begin{proof}[Proof of Lemma~\ref{lem:solution_compact_set_general}]
For a time-independent point $z$, the difference $y - z$ satisfies the ODE
\begin{equation} \label{eq:shifted_ODE}
\frac{d(y - z)}{ds} = - (y - z) - (z  - \ys_N(y,s)).
\end{equation}
Using an integrating factor, the solution of~\eqref{eq:shifted_ODE} is given by
$$(y(s) - z) = (x_T - z)e^{-s-\ln \sigma(T)} - \int_{-\ln \sigma(T)}^s (z  - \ys_N(y,\tau))e^{\tau - s} d\tau.$$
Using that the weights $\ww_\ell$ are positive and sum to one, 
$$z - \ys_N(y,s) = \sum_{\ell=1}^N \ww_\ell(y,s) (z - x_0^\ell)$$ 
is bounded by
$$\|z - \ys_N(y,\tau)\| \leq  \sum_{\ell=1}^N \ww_\ell(y,\tau)\|z - x_0^\ell\| \leq \max_{\ell \in \{1,\dots,N\}} \|z - x_0^\ell\| \sum_{\ell=1}^N \ww_\ell(y,\tau) = \max_{\ell \in \{1,\dots,N\}} \|z - x_0^\ell\|.$$ 
Thus, the distance of the solution to $z$ is bounded by
\begin{subequations}
\label{eq:DistanceBoundFixedPoint}
\begin{align}
\|y(s) - z\| &\leq \|{x_T} - z\|e^{-s-\ln \sigma(T)} + \int_{-\ln \sigma(T)}^s \|z  - \ys_N(y,\tau)\|e^{\tau - s} d\tau\nonumber\\
&\leq \|{x_T} - z\|e^{-s-\ln \sigma(T)} + \max_{\ell \in \{1,\dots,N\}} \|z - x^\ell_0\| (1 - e^{-s-\ln \sigma(T)})\\
&\leq \max\{\|{x_T} - z\|,\|x_0^1 - z\|,\dots,\|x_N - z\|\}. 
\end{align}
\end{subequations}
Now note that (\ref{eq:DistanceBoundFixedPoint}b) gives~(\ref{eq:DistanceLimSup}a);
taking $s \to \infty$ in (\ref{eq:DistanceBoundFixedPoint}a) gives~(\ref{eq:DistanceLimSup}b).
\end{proof}

\begin{proof}[Proof of Lemma~\ref{l:NFS}]
The identity~\eqref{eq:NFS99} follows simply from rearranging inner-products
defining \eqref{eq:NFS55}. Thus the solution $y$ defines a hyperplane given by
\begin{equation}
    \label{eq:NFS00}
 x^n_0-y = \frac12 (x^n_0-x^\ell_0)-\frac{\varepsilon^2}{2}\frac{(x^n_0-x^\ell_0)}{|(x^n_0-x^\ell_0)|^2}+q, \quad \forall q \perp  x^n_0-x^\ell_0. 
\end{equation}
If $x-y$ is on the hyperplane then clearly
\begin{equation*}
 \langle x-y, x^n_0-x^\ell_0 \rangle=0,
\end{equation*}
and this inner-product is strictly positive (resp.\thinspace strictly negative) iff $x$ is in the
set containing $x^n_0$ (resp.\thinspace $x^\ell_0$).
\end{proof}

\begin{proof}[Proof of Lemma~\ref{lemma:limiting_points}]
From~(\ref{eq:DistanceLimSup}b) we have
$|x^\star| \leq |x|_\infty$, establishing the desired bound (i) on the limit point. For any trajectory $y(s)$, (\ref{eq:DistanceLimSup}b) also shows that for every $\delta > 0$ there exists some time $s'$ so that $|y(s)| \leq |x|_\infty + \delta$ for $s \geq s'$. Setting $r = |x|_\infty + \delta$, the result (ii) follows.
\end{proof}

\begin{proof}[Proof of Lemma~\ref{lem:weight_bounds}] Without loss of generality we take $n=1.$ For $y \in V_{\delta}(x^1_0)$ 
and $\ell \neq 1$ we have 
\begin{equation*}
\wwtilde_\ell(y,s)/\wwtilde_1(y,s) = \exp\left(-\frac{e^{2s}}{2}(|y - x^\ell_0|^2 - |y - x^1_0|^2)\right) < \exp\left(-\frac{e^{2s}}{2}\delta^2 \right).
\end{equation*}
Hence, the normalized weights satisfy
\begin{align*}
\ww_\ell(y,s) = \frac{\wwtilde_\ell\bigl(y,s\bigr)}{\sum_{\ell=1}^N \wwtilde_\ell\bigl(y,s\bigr)} \leq \frac{\wwtilde_\ell\bigl(y,s\bigr)}{\wwtilde_1\bigl(y,s\bigr)} <  \exp\left(-\frac{e^{2s}}{2}\delta^2\right).
\end{align*}
Furthermore we see that
\begin{align*}
\ww_1(y,s) = \frac{\wwtilde_1\bigl(y,s\bigr)}{\sum_{\ell=1}^N \wwtilde_\ell\bigl(y,s\bigr)} = \frac{1}{1 + \sum_{\ell=2}^N \wwtilde_\ell\bigl(y,s\bigr)/\wwtilde_1(y,s)} \geq \frac{1}{1 + (N-1)\exp(-e^{2s}\delta^2/2)}.
\end{align*}
\end{proof}

\begin{proof}[Proof of Lemma~\ref{lemma:boundary_intersection}] 
Without loss of generality we take $n=1$ and hence drop suffix
$n$ from $\alpha_n$. To establish positivity of $\alpha$ simply use the identity~\eqref{eq:NFS99}. Now define the error 
$$\epsilon_N(y,s) \coloneqq x^1_0 - \ys_N(y,s) = \bigl(1 - \ww_1(y,s)\bigr)x^1_0 - \sum_{\ell=2}^N \ww_\ell(y,s) x^\ell_0.$$
Using the bounds on the normalized weights in Lemma~\ref{lem:weight_bounds}, for $y \in \partial V_\delta^r(x^1_0)$ the error $\epsilon_N$ is bounded for all $s \ge s_0$ by
\begin{align} \label{eq:ErrorUpperBound_Voronoi}
|\epsilon_N(y,s)| %
&\leq |1 - \ww_1(y,s)||x^1_0| + \sum_{\ell=2}^N |\ww_\ell(y,s)| |x^\ell_0| \nonumber \\
&\leq \frac{\exp\left(-\frac{e^{2s}\delta^2}{2}\right)(N-1)}{1 + \exp\left(-\frac{e^{2s}\delta^2}{2}\right)(N-1)}|x^1_0| + \sum_{\ell \geq 2} \exp\left(-\frac{e^{2s}\delta^2}{2}\right) |x^\ell_0| \nonumber \\
&\leq 2(N-1)\exp\left(-\frac{e^{2s}\delta^2}{2}\right)|x|_{\infty}.
\end{align}
For the choice of $s \geq s^\alpha$ we have 
$$|\ys_N(y,s) - x^1_0| \leq 2(N-1)\exp\left(-\frac{e^{2s_0}}{2}\delta^2\right)|x|_{\infty} < \frac{\alpha}{2D^+}.$$
Lastly, decomposing the error and using the bound in~\eqref{eq:minimum_boundary_angle}, we obtain
\begin{align*}
\langle -(y - \ys_N(y,s)), x_0^1 - x^\ell_0 \rangle &= \langle x_0^1 - y, x_0^1 - x^\ell_0 \rangle + \langle \ys_N(y,s) - x_0^1, x_0^1 - x^\ell_0 \rangle \\
&\geq \alpha - |x_0^1 - \ys_N(y,s)||x_0^1 - x^\ell_0|. \\
&\geq \alpha - \frac{\alpha}{2D^+}D^+ = \frac{\alpha}{2}.
\end{align*}
\end{proof}

\begin{proof}[Proof of Lemma~\ref{cor:exponential_convergence_newcoords}]
By Lemma \ref{lem:invariant_sets} we deduce that for $s \ge s^\ast$ the solution enters and remains in $V_{\delta}^{r}(x^n_0)$, for some integer $n$. Without loss of generality we may again take $n=1$ in the proof.
 
 For the ODE~\eqref{eq:mf_ode4}, we deduce that
$$\frac{d(y-x_0^1)}{ds} = -(y - \ys_N(y,s)) = -(y - x_0^1) + \epsilon_N(y,s),$$
for $\epsilon_N(y,s) \coloneqq \ys_N(y,s) - x_0^1$. Multiplying by the integrating factor $e^s$ yields the solution
$$(y(s) - x_0^1)e^{s} = (y(s_0) - x_0^1)e^{s_0} + \int_{s_0}^s e^\tau \epsilon_N(y,\tau)d\tau.$$

Using the time-dependent upper bound on $\epsilon_N$ in~\eqref{eq:ErrorUpperBound_Voronoi},  the distance from $y(s)$ to the point $x_0^1$ is bounded, for $s \in [s_0,s^\ast),$ by
\begin{align}
|y(s) - x_0^1| &\leq |y(s_0) - x_0^1|e^{s_0 - s} + \int_{s_0}^s e^{\tau - s} |\epsilon_N(y,\tau)|d\tau \label{eq:DistancetoX1Bound} \\
&\leq |y(s_0) - x_0^1|e^{s_0 - s} + 2(N-1)|x|_\infty \int_{s_0}^s e^{\tau - s} e^{-e^{2\tau}\delta^2/2} \dd \tau.
\end{align}
The integral satisfies
$$\int_{s_0}^s e^{\tau - s} e^{-e^{2\tau}\delta^2/2} \dd \tau \leq e^{-s_0-s} \int_{s_0}^s e^{2\tau} e^{-e^{2\tau}\delta^2/2} \dd \tau \leq \frac{e^{-s_0-s}}{\delta^2}e^{-e^{2s_0}\delta^2/2}.$$
Using this inequality in~\eqref{eq:DistancetoX1Bound}, we obtain the bound
$$|y(s) - x_0^1| \leq |y(s_0) - x_0^1|e^{s_0 - s} + \frac{2(N-1)|x|_{\infty}}{\delta^2}e^{-s_0 - s}e^{-e^{2s_0}\delta^2/2}.$$ 
\end{proof}

\subsection{Variance Preserving Setting}

\begin{proof}[Proof of Lemma~\ref{lem:existence_uniqueness_VP}]
We first obtain a-priori bounds on the solution of~\eqref{eq:reverse_ODE_VP_empirical} starting from the prescribed initial condition, and assuming the solution exists. 
Using an integrating factor, one deduces that $y(\cdot)$ solve the integral equation
$$y(s) = e^{-s + s_0}x_T + \int_{s_0}^s e^{\tau-s}\frac{1}{\sqrt{1 - e^{-2\tau}}}\ys_N\bigl(y,\tau\bigr) d\tau.$$
Thus we obtain the following bound on the solution:
\begin{align*}
    \|y(s)\| &\leq e^{-s + s_0}\|x_T\| + \int_{s_0}^s e^{\tau-s}\frac{1}{\sqrt{1 - e^{-2\tau}}} \|\ys_N\bigl(y,\tau\bigr)\| d\tau \\
    &\leq e^{-s + s_0}\|x_T\| + \sup_{(y,s)}\|\ys_N\bigl(y,s\bigr)\|  \int_{s_0}^s e^{\tau-s}\frac{1}{\sqrt{1 - e^{-2\tau}}} d\tau\\
    &= e^{-s + s_0}\|x_T\| + \sup_{(y,s)}\|\ys_N\bigl(y,s\bigr)\|   \int_{s_0}^s e^{-s} \frac{d}{d\tau}\left[e^\tau\sqrt{1 - e^{-2\tau}}\right]d\tau \\
    &= e^{-s + s_0}\|x_T\| + \sup_{(y,s)}\|\ys_N\bigl(y,s\bigr)\| \left[\sqrt{1 - e^{-2s}} - e^{-s + s_0}\sqrt{1 - e^{-2s_0}}\right].
\end{align*}
Using that the weights $\ww_\ell$ are positive and sum to one, we have the bound
\begin{equation} \label{eq:soln_upper_bound_VP}
  \|\ys_N(y,s)\| \leq \sum_{\ell=1}^N \ww_\ell(y,s)\|x_0^\ell\| \leq \max_{\ell \in \{1,\dots,N\}} \|x_0^\ell\| \sum_{\ell=1}^N \ww_\ell(y,s) = \max_{\ell \in \{1,\dots,N\}} \|x_0^\ell\|.  
\end{equation}
Then, it follows that 
$$\limsup_{s \rightarrow \infty} \|y(s)\| \leq \sup_{(y,s)} \|\ys_N(y,s)\| \leq \max_{\ell \in \{1,\dots,N\}} \|x_0^\ell\|.$$ 
Letting $R(x_T,N) \coloneqq \max\left(\|x_T\|, \max_{1 \leq n \leq N} \|x_0^n\|\right)$, we have the bound on the solution
\begin{align}
\|y(s)\| &\leq \left[e^{-s + s_0} + \sqrt{1 - e^{-2s}} - e^{-s + s_0}\sqrt{1 - e^{-2s_0}}\right] R(x_T,N) \notag \\
&= \left[\sqrt{e^{-2s}}(1 - \sqrt{1 - e^{-2s_0}})/\sqrt{e^{-2s_0}} + \sqrt{1 - e^{-2s}}\right] R(x_T,N) \notag \\
&\leq \left[\sqrt{e^{-2s}} + \sqrt{1 - e^{-2s}}\right] R(x_T,N),
\end{align}
where we used that $(1 - \sqrt{1 - x})/\sqrt{x} \leq 1$ for $x \in (0,1]$. Given that $\sqrt{x} + \sqrt{1 - x} \leq \sqrt{2}$ for all $x \in [0,1]$, then
$$\sup_{s \in [s_0,\infty)} \|y(s)\| \leq \sqrt{2} R(x_T,N), \qquad .$$ 

Recall that $s_0=-\ln \sigma(T).$ To prove existence of a solution $y \in C^1([s_0,\infty);H)$, as in the proof of Lemma~\ref{thm:existence_uniqueness} it suffices to prove that $\wwtilde_n(\cdot,s)$ is continuous on $H \times [s_0,\tau]$ and that $\ww_n(\cdot,s)$ is Lipschitz continuous, uniformly for 
$t \in [s_0,\tau]$. The continuity of $\wwtilde_n(\cdot,\cdot)$ follows from continuity of the norm on $H$ and the boundedness of $\sum_{m=1}^N \wwtilde_m(\cdot,\cdot)$ from below. To prove the Lipschitz property, we use this boundedness, together with the fact that that $\exp(-|x|)$ has Lipschitz constant one on $\R$, to deduce that 
\begin{align*}
|\widetilde{\ww}_n(y,s) - \widetilde{\ww}_n(y',s)| &\leq \frac{e^{2\tau}}{2}\left(\|y - \sqrt{1 - e^{-2s}} x_0^n\|^2 - \|y' - \sqrt{1 - e^{-2s}} x_0^n\|^2\right) \\
&\leq \frac{e^{2\tau}}{2} \| y - y'\| \|y + y' - 2 \sqrt{1 - e^{-2s}} x_0^n \| \\
&\leq 2 e^{2\tau} \| y - y'\|  \max\{\|y\|,\|y'\|,\|\sqrt{1 - e^{-2s}}x_0^n\|\} \\
&\leq 2e^{2\tau} R(x_T,N) \| y - y'\|,
\end{align*}
where we used that $|\sqrt{1 - e^{-2s}}| \leq 1$ for all $s \geq 0$. The desired result follows.
\end{proof}

\begin{proof}[Proof of Lemma~\ref{lem:weight_bounds_VP}] Without loss of generality we take $n = 1$. Abusing notation we write $m(s) \coloneqq \sqrt{1 - e^{-2s}}$ within this proof. Then, for $\ell \neq 1$ we have
\begin{equation} \label{eq:VP_weightratio}
\wwtilde_\ell(y,s)/\wwtilde_1(y,s) = \exp\left(-\frac{e^{2s}}{2}(|y - m(s)x_0^\ell|^2 - |y - m(s)x_0^1|^2\right).
\end{equation}
For $r > |x|_\infty$ and $s > s^D$ where $s^D$ satisfies~\eqref{eq:sd_constraint} we have
$$(1 - m(s))^2 \leq (1 - m(s)) < \frac{\delta^2}{8D^+r} < \frac{\delta^2}{8D^+|x|_{\infty}}.$$
Thus, we have %
\begin{align*}
(m(s) - 1)^2(|x_0^1|^2 - |x_0^\ell|^2) &= (m(s) - 1)^2|x_0^1 - x_0^\ell||x_0^1 +  x_0^\ell| \leq 2(m(s) - 1)^2D^+|x|_\infty < \frac{\delta^2}{4}.
\end{align*}
For $y \in V_{\delta}^r(x_0^1)$, we also have
\begin{align*}
2(m(s) - 1)y^\top (x_0^\ell - x_0^1) &\leq 2(m(s) - 1)|y||x_0^\ell - x_0^1| \leq 2(m(s) - 1)rD^+ < \frac{\delta^2}{4}.
\end{align*}
Then, the exponent in~\eqref{eq:VP_weightratio} satisfies
\begin{align*}
|y - m(s)x_0^1|^2 - |y - m(s)x_0^\ell|^2 &= |y - x_0^1|^2 - |y - x_0^\ell|^2 + (1 - m(s))^2(|x_0^1|^2 - |x_0^\ell|^2) \\
&\quad - 2(1 - m(s))y^\top(x_0^1 - x_0^\ell) \\
&\leq -\delta^2 + \delta^2/4 + \delta^2/4. \\
&= \delta^2/2.
\end{align*}
Hence, the normalized weights satisfy
\begin{align*}
\ww_\ell(y,s) = \frac{\wwtilde_\ell\bigl(y,s\bigr)}{\sum_{\ell=1}^N \wwtilde_\ell\bigl(y,s\bigr)} \leq \frac{\wwtilde_\ell\bigl(y,s\bigr)}{\wwtilde_1\bigl(y,s\bigr)} < \exp\left(-\frac{e^{2s}}{4}\delta^2\right).
\end{align*}
Furthermore we see that
\begin{align*}
\ww_1(y,s) = \frac{\wwtilde_1\bigl(y,s\bigr)}{\sum_{\ell=1}^N \wwtilde_\ell\bigl(y,s\bigr)} &= \frac{1}{1 + \sum_{\ell=2}^N \wwtilde_\ell\bigl(y,s\bigr)/\wwtilde_1(y,s)} \\
&\geq \frac{1}{1 + (N-1)\exp(-e^{2s}\delta^2/4)}.
\end{align*}
\end{proof}

\begin{proof}[Proof of Lemma~\ref{lemma:boundary_intersection_VP}] Without loss of generality, let $n = 1$. Then,  ODE~\eqref{eq:reverse_ODE_VP_empirical} is equivalently given by
$$\frac{dy}{ds} = - \left(y - \frac{\ys_N(y,s)}{\sqrt{1 - e^{-2s}}}\right) 
= - (y - x_0^1) - \left(x_0^1 - \frac{\ys_N(y,s)}{\sqrt{1 - e^{-2s}}}\right).$$
Define the error 
\begin{equation} \label{eq:error_term_VP}
\epsilon_N(y,s) \coloneqq x^1_0 - \frac{\ys_N(y,s)}{\sqrt{1 - e^{-2s}}} = x^1_0 - \ys_N(y,s) + \ys_N(y,s) - \frac{\ys_N(y,s)}{\sqrt{1 - e^{-2s}}}.
\end{equation}
Thus, we can bound the error as
$$|\epsilon_N(y,s) \leq \left|x^1_0 - \ys_N(y,s)\right| + \left|\ys_N(y,s) - \frac{\ys_N(y,s)}{\sqrt{1 - e^{-2s}}}\right|.$$
First, %
from Lemma~\ref{lem:weight_bounds_VP} for $y \in \partial V_\delta^r(x^1_0)$ and $s \geq s^D$, we have %
\begin{align} \label{eq:ErrorUpperBound_Voronoi_VP}
|x^1_0 - \ys_N(y,s)| %
&\leq |1 - \ww_1(y,s)||x^1_0| + \sum_{\ell=2}^N |\ww_\ell(y,s)| |x^\ell_0| \nonumber \\
&\leq \frac{\exp\left(-\frac{e^{2s}\delta^2}{4}\right)(N-1)}{1 + \exp\left(-\frac{e^{2s}\delta^2}{4}\right)(N-1)}|x^1_0| + \sum_{\ell \geq 2} \exp\left(-\frac{e^{2s}\delta^2}{4}\right) |x^\ell_0| \nonumber \\
&\leq 2(N-1)\exp\left(-\frac{e^{2s}\delta^2}{4}\right)|x|_{\infty}.
\end{align}
Moreover, for $s \geq s^1$ we have 
$$|x^1_0 - \ys_N(y,s)| \leq 2(N-1)\exp\left(-\frac{e^{2s^\alpha}}{4}\delta^2\right)|x|_{\infty} < \frac{\alpha_1}{2D^+}.$$
Second, given that $\sup_{(y,s)}|\ys_N(y,s)| \leq |x|_{\infty}$ for $s \geq s^2$ we have  
$$\left|\ys_N(y,s)\left(1 - \frac{1}{\sqrt{1 - e^{-2s}}}\right)\right| \leq |x|_\infty \frac{\alpha_1}{4D^+|x|_\infty} = \frac{\alpha_1}{4D^+}.$$
Combining these results, we have that for $s \geq \max\{s^1,s^2,s^D\}$ then
\begin{align*}
&\left|\langle -\left(y - \frac{\ys_N(y,s)}{\sqrt{1 - e^{-2s}}}\right), x_0^1 - x^\ell_0 \rangle\right| \\
=&\left|\langle -(y - \ys_N(y,s)), x_0^1 - x^\ell_0 \rangle + \langle \ys_N(y,s) (1 - \frac{1}{\sqrt{1 - e^{-2s}}}), x_0^1  - x^\ell_0\right|\\
\geq&\left|\langle -(y - \ys_N(y,s)), x_0^1 - x^\ell_0 \rangle\right| - \left|\ys_N(y,s) (1 - \frac{1}{\sqrt{1 - e^{-2s}}})\right||x_0^1  - x^\ell_0|\\
\geq& \frac{\alpha_1}{2D^+}D^+ - \frac{\alpha_1}{4D^+}D^+ = \frac{\alpha_1}{4}.
\end{align*}
\end{proof}

\begin{proof}[Proof of Lemma~\ref{cor:exponential_convergence_newcoords_VP}]
Without loss of generality we may again take $n=1$ in the proof. For  ODE~\eqref{eq:reverse_ODE_VP_empirical}, we deduce that the solution also solves
$$\frac{d(y-x_0^1)}{ds} = -\left(y - \frac{\ys_N(y,s)}{\sqrt{1 - e^{-2s}}}\right) = -(y - x_0^1) + \epsilon_N(y,s),$$
with the error term $\epsilon_N$ defined in~\eqref{eq:error_term_VP}. Multiplying by the integrating factor $e^s$ yields the formal solution
$$y(s) - x_0^1 = (y(s_0) - x_0^1)e^{-s + s_0} + \int_{s_0}^s e^{-s + \tau} \epsilon_N(y,\tau)d\tau.$$
Using the form for the error term, the integral on the right hand side satisfies
\begin{align*}
\int_{s_0}^s e^\tau \epsilon_N(y,\tau)d\tau &= \int_{s_0}^s e^{\tau} (x_0^1 - \ys_N(y,\tau))d\tau + \int_{s_0}^s e^{\tau} \ys_N(y,\tau)\left(1 - \frac{1}{\sqrt{1 - e^{-2\tau}}}\right)d\tau \\
&\leq \int_{s_0}^s e^{\tau} \sup_{(y,\tau)} |x_0^1 - \ys_N(y,\tau)| d\tau + \sup_{(y,s)} |\ys_N(y,s)| \int_{s_0}^s e^{\tau} \left| \frac{1}{\sqrt{1 - e^{-2\tau}}} - 1\right|d\tau \\
&= \int_{s_0}^s e^{\tau} \sup_{(y,\tau)} |x_0^1 - \ys_N(y,\tau)| d\tau + \sup_{(y,s)} |\ys_N(y,s)| e^{\tau}(\sqrt{1 - e^{-2\tau}} - 1)\vert_{s_0}^s.
\end{align*}
From the upper bounds in~\eqref{eq:soln_upper_bound_VP} and in~\eqref{eq:ErrorUpperBound_Voronoi_VP}, we have
\begin{align*}
    \int_{s_0}^s e^\tau \epsilon_N(y,\tau)d\tau &\leq  2(N-1)|x|_\infty \int_{s_0}^s e^\tau e^{-e^{2\tau}\delta^2/4} d\tau + |x|_\infty e^{\tau}(\sqrt{1 - e^{-2\tau}} - 1)\vert_{s_0}^s\\
    &\leq 2(N-1)|x|_\infty e^{-s_0} \left(-\frac{2}{\delta^2}\right) \left. e^{-e^{2\tau}\delta^2/4} \right|_{s_0}^s + \left. |x|_\infty e^{\tau}(\sqrt{1 - e^{-2\tau}} - 1)\right|_{s_0}^s\\
    &\leq 2(N-1) |x|_\infty \frac{2e^{-s_0}}{\delta^2}e^{-e^{2s_0}\delta^2/4} \\
    &\quad + |x|_\infty \left[e^{s}(\sqrt{1 - e^{-2s}}  - 1) - e^{s_0}(\sqrt{1 - e^{-2s_0}} - 1)\right]. 
\end{align*}
Substituting the integral bound in the expression for the formal solution, the distance from $y(s)$ to the point $x_0^1$ is bounded as 
\begin{align*}
|y(s) - x_0^1| &\leq |y(s_0) - x_0^1|e^{s_0 - s} + e^{-s}\left|\int_{s_0}^s \epsilon_N(y,\tau)d\tau \right| \label{eq:DistancetoX1Bound} \\
&\leq |y(s_0) - x_0^1|e^{s_0 - s} + 4(N-1)|x|_\infty \frac{e^{-s_0}}{\delta^2}e^{-e^{2s_0}\delta^2/4} e^{-s} + \\
&\quad |x|_\infty \left[|(\sqrt{1 - e^{-2s}}  - 1)| - e^{s_0 - s}(\sqrt{1 - e^{-2s_0}} - 1)\right].
\end{align*}
Using that $|1 - \sqrt{1 -e^{-2s}}| \leq e^{-2s} \leq e^{-s}$, the final result follows.
\end{proof}

\section{Proofs of Additional Results}

\subsection{Solution of Linear Stochastic Differential Equations} \label{app:linearSDEs}

\begin{proof}[Proof of Lemma \ref{lemma:SDEsoln}]
Using an integrating factor $e^{-\frac12 \int_0^t \beta(s)\dd s}$, the formal solution of the linear SDE starting from the initial condition $x(0) = x_0$ is given by
$$x(t) = x_0 e^{-\frac12 \int_0^t \beta(s)\dd s} + e^{-\frac12 \int_0^t \beta(s)ds}\int_0^t \sqrt{g(\tau)} e^{\frac12 \int_0^\tau \beta(s)ds} \dd W_\tau.$$
The second term has mean zero, so the conditional expectation  of $x(t)|x_0$ is given by 
$$\mathbb{E}[x(t)|x_0] = x_0 e^{-\frac12 \int_0^t \beta(s)\dd s} = x_0 m(t).$$
The conditional variance of $x(t)|x_0$ using Ito's formula is given by
$$\text{Cov}(x(t)|x_0) = \mathbb{E}[(x_t - x_0m(t))(x_t - x_0m(t))^T|x_0] = e^{-\int_0^t \beta(s)ds} \int_0^t g(\tau) e^{\int_0^\tau \beta(s) ds} d\tau I_d .$$
Using the definition of $m(t)$ gives us the desired result.
The monotonicity of $\sigma^2(t)$ follows 
from the strict positivity of $g(t).$
\end{proof}

\subsection{Unconditioned Setting}
\label{ssec:wnc}

\begin{proof}[Proof of Theorem~\ref{thm:empirical_score_unconditional}] We recall the empirical objective function $\J_0^N$ is given by 
\begin{align*}
\J_0^N(\score) &\coloneqq \int_0^T \frac{1}{N}\sum_{n=1}^N \mathbb{E}_{x \sim p(x,t|x_0^n)} \left|\score(x,t) + \frac{x - m(t)x_0^n}{\sigma(t)^2}\right|^2 \dd x,\\
&= \int_0^T \int \frac{1}{N}\sum_{n=1}^N \left|\score(x,t) - \frac{x - m(t)x_0^n}{\sigma(t)^2}\right|^2 \frac{1}{\sqrt{(2\pi)^d\sigma^2(t)}}\exp\left(-\frac{|x - m(t)x_0^n|^2}{2\sigma^2(t)} \right)\dd \bu.
\end{align*}
By setting the functional gradient of the integrand with respect to $\score(x,t)$ to zero,  the minimizer of $\J_0^N$ for each $x \in \R^d$ and $t \in [0,T]$ satisfies
$$\sum_{n=1}^N \left(\score^N(x,t) - \frac{x - m(t)x_0^n}{\sigma(t)^2} \right) \exp\left(-\frac{|x - m(t)x_0^n|^2}{2\sigma^2(t)} \right) = 0,$$
which can be re-arranged to give the desired formula.
\end{proof}

\begin{proof}[Proof of Lemma~\ref{p:ifKf}]
By conditioning on $x_0$ we note that
$$\J(\score) = \int_0^{T} \lambda(t) \mathbb{E}_{x_0 \sim \pr_0(\cdot)} \mathbb{E}_{x \sim \pr(\cdot,t|x_0)}|\score(x,t) - \nabla_\bu \log \pr(\bu,t)|^2 \dd t.$$
From this it follows that
$$\J(\score)=\J_0(\score)-K+2\int_0^T \lambda (t) \bigl(I_1-I_2)\bigr)\dd t$$
where
\begin{align*}
    I_1 &= \mathbb{E}_{x_0 \sim \pr_0(\cdot)} \mathbb{E}_{x \sim \pr(\cdot,t|x_0)} \langle \score(x,t), \nabla_x \log p(x,t|x_0) \rangle,\\
    I_2 &= \mathbb{E}_{x_0 \sim \pr_0(\cdot)} \mathbb{E}_{x \sim \pr(\cdot,t|x_0)} \langle \score(x,t), \nabla_x \log p(x,t) \rangle.
\end{align*}
Now note that
\begin{align*}
    I_1 &=  \mathbb{E}_{x_0 \sim \pr_0(\cdot)} \int \langle \score(x,t), \nabla_x p(x,t|x_0) \rangle \dd x,\\
     &=\int \langle \score(x,t), \nabla_x p(x,t) \rangle \dd x,\\
     &=\mathbb{E}_{x \sim \pr(\cdot,t)} \langle \score(x,t), \nabla_x \log p(x,t) \rangle,\\
     &= I_2.
\end{align*}

\end{proof}

\subsection{Conditional Setting} \label{ssec:wnc_conditional_proof}

\begin{proof}[Proof of Theorem~\ref{thm:empirical_score_conditional}] We recall the empirical objective function $\Jcondzero^N$ %
is given by 
\begin{align*}
\Jcondzero^N(\score) &\coloneqq \int_0^T \frac{1}{N}\sum_{n=1}^N \mathbb{E}_{x \sim p(x,t|x_0^n)} \left|\score(x,y^n,t) + \frac{x - m(t)x_0^n}{\sigma(t)^2}\right|^2\dd x \\
&= \int_0^T \int \frac{1}{N}\sum_{n=1}^N \left|\score(x,y^n,t) + \frac{x - m(t)x_0^n}{\sigma(t)^2}\right|^2 \frac{1}{\sqrt{(2\pi)^d\sigma^2(t)}}\exp\left(-\frac{|x - m(t)x_0^n|^2}{2\sigma^2(t)} \right) \dd \bu.
\end{align*}
By setting the functional gradient of the integrand with respect to $\score(x,t)$ to zero, the minimizer of $\Jcondzero^N$ for each $x \in \R^d$, $t \in [0,T]$ and training sample $y^n$ satisfies
$$\sum_{n=1}^N \left(\score^N(x,y^n,t) - \frac{x - m(t)x_0^n}{\sigma(t)^2}\right) \exp\left(-\frac{|x - m(t)x_0^n|^2}{2\sigma^2(t)} \right) = 0,$$
which can be re-arranged to give us the desired formula. We emphasize that the solution for $\score^N(x,y,t)$ is not defined for $y \neq y^n$ for some $n = 1,\dots,N$.
\end{proof}

\section{Variance Preserving Process: Analysis and Numerics}
\label{a:C}

\subsection{Analysis: Variance Preserving Process} \label{ssec:VPanalysis}

Now we consider the ODE~\eqref{eq:reverseODE_empirical} in the case $g(t) = \beta(t)$,
corresponding to the variance preserving process in Example~\ref{ex:VP}. Recall Theorem~\ref{thm:empirical_score_unconditional} giving the normalized weights
\begin{equation*}
\w_n(x,t) = \frac{\wtilde_n(x,t)}{\sum_{l=1}^N \wtilde_l(x,t)}, \qquad \wtilde_n(x,t) = \exp\left(-\frac{|x - m(t)x_0^n|^2}{2\sigma^2(t)}\right),
\end{equation*}
with the conditional mean and variance given by  
$$m(t) = e^{-\frac{1}{2}\int_0^t g(s) ds}, \quad \sigma^2(t) = 1 - e^{-\int_0^t g(s) ds}.$$
We recall the definition of $\xs_N$ in~\eqref{eq:convex_combination_data} as the convex combination of the data points %
$$\xs_N(x,t) = \sum_{n=1}^N x_0^n w_n(x,t).$$
From equation~\eqref{eq:empirical_score} in Theorem~\ref{thm:empirical_score_unconditional} for the empirical score, we see that the ODE~\eqref{eq:reverseODE_empirical} becomes
\begin{equation}
\label{eq:addd} 
\frac{\dd\bu}{\dd t}= -\frac{1}{2}g(t)\bu + \frac{g(t)}{2\sigma^2(t)} \bigl(\bu - m(t)\xs_N(x,t)\bigr), \quad x(T) \sim \mathcal{N}(0,I_d).
\end{equation}
We solve this ODE starting from $T$ sufficiently large, backwards-in-time to $x(0)$. As for the variance preserving ODE, we show that this leads to memorization.

\begin{example} For the variance preserving process in Example~\ref{ex:VP_empirical} with $g(t) \equiv 1$, the ODE in~\eqref{eq:addd} has the form  
$$\frac{\dd x}{\dd t} = \frac{e^{-t/2}}{2(1 - e^{-t})}\bigl(e^{-t/2}x-\xs_N(x,t)\bigr).$$
\end{example}

Returning to the general equation~\eqref{eq:addd} we first observe that, in the variance preserving case, 
\begin{subequations}
\label{eq:willrefer}
    \begin{align}
m(t) &= \exp\Bigl(-\frac12 \int_0^t g(\tau)d\tau\Bigr),\\
\sigma^2(t) &= 1 - \exp\Bigl(-\int_0^t g(\tau)d\tau\Bigr).
\end{align}
\end{subequations}
As for the variance exploding process, it is convenient to define the mapping $t \mapsto s$ by $s = -\frac{1}{2}\ln \sigma^2(t)$ and define
$$y(s) = x(t), \qquad \ys_N(x,s) = \xs_N(x,t).$$
Using \eqref{eq:willrefer} we find that
\begin{subequations}
\label{eq:willrefer2}
\begin{align}
\exp\Bigl(-\int_0^t g(\tau)d\tau\Bigr) & = 1-e^{-2s},\\
m(t) & = \sqrt{1 - e^{-2s}}.
\end{align}
\end{subequations}
It follows that
\begin{equation}
\ys_N(y,s)=\sum_{n=1}^N \ww_n(y,s)x_0^n, \quad \ww_n(y,s)=w_n\bigl(y,(\sigma^2)^{-1}(e^{-2s})\bigr).
\end{equation}
with
\begin{equation} \label{eq:Snormalized_Gaussian_weights_VP}
  \ww_n(y,s) = \frac{\wwtilde_n(y,s)}{\sum_{\ell=1}^N \wwtilde_\ell(y,s)}, \qquad \wwtilde_n(y,s) = \exp\left(-\frac{|y - \sqrt{1 - e^{-2s}}x_0^n|^2}{2e^{-2s}}\right).   
\end{equation}
Then the ODE \eqref{eq:addd} transforms to become
\begin{equation} \label{eq:reverse_ODE_VP_empirical}
  \frac{\dd y}{\dd s} = -\left(y - \frac{1}{\sqrt{1 - e^{-2s}}}\ys_N(y,s)\right), \quad y(s_0)=x_T,
\end{equation}
where $s = s_0 \coloneqq -\ln\sigma(T)$. If $x_T \sim N(0,I_d)$ then, in the new time variable $s$,
integrating forward to $s=\infty$ corresponds to solving
\eqref{eq:addd} down to $t=0^+$.  Note that, for the variance-preserving process, $\sigma(T) \in (0,1)$ and so $s_0 > 0.$

\begin{remark} Since $\sqrt{1 - e^{-2s}} \rightarrow 1$ as $s \rightarrow \infty$ the ODE~\eqref{eq:reverse_ODE_VP_empirical} behaves similarly to the ODE~\eqref{eq:mf_ode4} arising 
and analyzed for the variance exploding process. Our proof of memorization in the
variance preserving case will mimic that in the variance exploding case, with some corrections to
allow for coefficients which are not equal to one, but approach it for large $s$.
\end{remark}

We now state an equivalent form of Theorem~\ref{thm:memorization}, specialized to the variance preserving  setting, and stated in the transformed time variable $s$ instead of $t.$

\begin{theorem_restated}[\textbf{Restatement of 
Theorem~\ref{thm:memorization} (Transformed Time, Variance Preserving Case)}] 
Let Assumptions \ref{asp:stand} and \ref{asp:abg} hold.
Then, given any initial condition $x_T$,
ODE~\eqref{eq:reverse_ODE_VP_empirical} has a unique solution defined for 
$s \in [-s_0,\infty).$
Furthermore,
there exists $r > 0$ such that, for any point $x_T \in \R^d$ and resulting
solution $y(s)$ of ODE~\eqref{eq:reverse_ODE_VP_empirical} initialized at 
$y(s_0) = x_T$, there is $s^\ast \coloneqq s^\ast(r,x_T)$ such that $y(s) 
\in B(0,r)$ for $s \ge s^\ast.$
Any limit point $x^\star$ of the trajectory is either on $\partial V^r$ or in $V^r$. If $x^\star \in V^r$, then it is one of the data points $\{x_0^n\}_{n=1}^N$ and the entire 
solution converges and does so at a rate $e^{-s}$: 
\begin{align} %
    |y(s) - x^\star| &\lesssim e^{-s}, \quad s \rightarrow \infty.
\end{align}
\end{theorem_restated}

\begin{proof}[Proof of Theorem~\ref{thm:memorization}; Variance Preserving Case] 
The proof is very similar to the variance exploding case: it is organized around a sequence
of lemmas which are analogous to those arising in  the variance exploding case.
Statement of the lemmas follows this proof overview and their proofs are contained in the appendix.
By Lemma~\ref{lem:existence_uniqueness_VP}, the solution of equation~\eqref{eq:reverse_ODE_VP_empirical} initialized at
$y\bigl(s_0\bigr)=x_T$,
has a unique solution for $s \in [s_0,\infty).$
Lemma~\ref{lemma:limiting_points_VP} shows that the solution is bounded for all $s \geq 0$ and its $\limsup$ is contained in a ball of radius 
$|x|_\infty \coloneqq \max\{|x_0^1|,\dots,|x_0^N|\}$.
Thus, for any $r > |x|_\infty$, there is $s^\ast \ge 0$
such that $y(s) \in B(0,r)$ for all $s \geq s^\ast$ and, as in the variance exploding case,
$x^\star$ lies either on $\partial V^r$ or in $V^r$. If the limit point is on $\partial V^r$ then the proof is complete; we thus assume henceforth that it is in $V^r$.

By definition of the limit point, there is a sequence $\{s_j\} \to \infty$ 
such that $y(s_j) \to x^\star$. As in the variance exploding case there is 
some $J \in \mathbb{N}$ such that $y(s_j) \in V_{\delta}^{r}(x^n_0)$ for all $j \geq J.$ 
Lemma \ref{lem:weight_bounds_VP} contains bounds on the
weights $\ww(y,s)$ for $y$ within $V_{\delta}^{r}(x^n_0).$ 
Again, as in the variance exploding case, for large enough $s$ we have $\ys_N(y,s) \approx x_0^n$ on $\partial V_{\delta}^{r}(x^n_0)$ -- see Lemma \ref{lemma:boundary_intersection_VP}. 
By Lemma~\ref{lem:invariant_sets_VP} we deduce that $y(s) \in V_{\delta}^{r}(x^n_0)$ for all $s \geq s_J$ and the dynamics of ODE \eqref{eq:reverse_ODE_VP_empirical} is approximated by
\begin{equation*}
    \frac{dy}{ds}=-\bigl(y-x_0^n\bigr).
\end{equation*}
Using a quantitative version of the approximation $\ys_N(y,s) \approx x_0^n$, Lemma~\ref{cor:exponential_convergence_newcoords_VP} establishes that $y(s) \rightarrow x^n_0$ as $s \rightarrow \infty$ 
and that the convergence is at the desired rate $e^{-s}$.
\end{proof}

\paragraph{Existence, Uniqueness and Bounds on the Solution.}

\begin{lemma} \label{lem:existence_uniqueness_VP} Let Assumptions \ref{asp:stand} and \ref{asp:abg} hold. For all $x_T \in H$ equation \eqref{eq:reverse_ODE_VP_empirical} has a unique solution $y \in C^1([-\ln \sigma(T),\infty);H).$ Furthermore, 
\begin{align}
\|y(s)\| &\leq \sqrt{2} \max\{\|x_T\|,\|x_0^1\|,\dots,\|x_0^N\|\} \\
\limsup_{s \rightarrow \infty} \|y(s)\| &\leq \max\{\|x_0^1\|,\dots,\|x_0^N\|\}.
\end{align}
\end{lemma}

 The following lemma identifies a bounded set, defined by the training data, to which any limit point $x^\star$ must be confined. Recall definition \eqref{eq:boundd}.

\begin{lemma} \label{lemma:limiting_points_VP}
Let Assumptions \ref{asp:stand} and \ref{asp:abg} hold.
Any limit point $x^\star$ of the dynamics defined by ODE~\eqref{eq:mf_ode4} is contained in the set
$\in B(0,|x|_\infty).$
 For any $r > |x|_\infty,$ there exists a time $s^r > s_0$ so that $y(s) \in B(0,r)$ for $s > s^r$.
\end{lemma}

\paragraph{Invariant Set for Dynamics.}

As in the variance exploding case, we start by obtaining a bound on the weights for the solution evaluated in $V_{\delta}^r(x^n_0)$ in Lemma~\ref{lem:weight_bounds_VP} and then proceed to show that, on the boundary of this set, the solution points inwards toward the center of the cell in Lemma~\ref{lemma:boundary_intersection_VP}. Recall definition~\eqref{eq:Ld} of function $L(\cdot,x^n_0)$ and definition~\eqref{eq:DD} of the data separation, $D.$

\begin{lemma} \label{lem:weight_bounds_VP} 
Let Assumptions \ref{asp:stand} and \ref{asp:abg} hold.
Choose $s^D > s_0$ so that 
\begin{align} \label{eq:sd_constraint}
1 - \sqrt{1 - e^{-2s^D}} &< \frac{\delta^2}{2D|x|_\infty}.
\end{align}
Then, for all $y \in V_{\delta}^r(x^n_0)$ and $s \geq s^D$, the normalized weights satisfy
\begin{subequations}
\begin{align}
    \frac{1}{1 + (N-1)\exp\left(-\frac{e^{2s}\delta^2}{4}\right)} &< \ww_n(y,s) \leq 1, \\
    0 &\leq \ww_\ell(y,s) < \exp\left(-\frac{e^{2s}\delta^2}{4}\right), \quad \forall\, \ell \neq n.
\end{align}
\end{subequations}
\end{lemma}

Recall the time $s^\alpha$ defined by~\eqref{eq:s_constraint}, noting that it is used in the following lemma. 

\begin{lemma} \label{lemma:boundary_intersection_VP} 
Let Assumptions \ref{asp:stand} and \ref{asp:abg} hold.
Fix any $\delta > 0$ and $r > |x|_\infty$. As in~\eqref{eq:minimum_boundary_angle}, define
\begin{equation} \label{eq:minimum_boundary_angle_VP}
\inf_{y \in \partial V_{\delta}^r(x^n_0)} \min_{\ell \in L_\delta(y,x^n_0)} \langle x^n_0 - y, x^n_0 - x^\ell_0 \rangle \coloneqq 
\alpha_n.
\end{equation}
Then $\alpha_n>0.$ Recall the definition for the maximum distance between data points $D$.  Choose $s^1 > s_0$ and $s^2 > s_0$ so that
\begin{equation} \label{eq:s_constraint_VP_boundary}
(N-1)\exp\left(-\frac{e^{2s^1}\delta^2}{4}\right) \leq \frac{\alpha_n}{4D}, \quad\text{and}\quad \frac{1}{\sqrt{1 - e^{-2s^2}}} - 1 \leq \frac{\alpha_n}{4D}.
\end{equation}
Then, for all $s \geq \max(s^1,s^2,s^\alpha)$ and $y \in \partial V_{\delta}^{r}(x^n_0),$ the vector field for the ODE in~\eqref{eq:reverse_ODE_VP_empirical} satisfies
$$\inf_{y \in \partial V_{\delta}^r(x^n_0)} \min_{\ell \in L_\delta(y,x^n_0)} \left\langle \frac{\ys_N(y,s)}{\sqrt{1 - e^{-2s}}} - y, x^n_0 - x^\ell_0 \right\rangle \geq \frac{\alpha_n}{4} > 0.$$
\end{lemma}

Consider any trajectory with limit point contained in the set of data.
A consequence of the preceding lemma is that there is data point $x_0^n$, depending on the initial
condition $x_T$, such that, after a sufficiently long time defined by conditions~\eqref{eq:s_constraint} and~\eqref{eq:s_constraint_VP_boundary}, 
the dynamics do not leave the constrained Voronoi cell $V_{\delta}^{r}(x^n_0)$. Moreover, the solution will converge exponentially fast in time to the data point $x_0^n$. These
ideas are formalized in the following two lemmas.

Recall the times $s^r$ defined in Lemma~\ref{lemma:limiting_points_VP} and $s^D$ defined in~\eqref{eq:sd_constraint}.

\begin{lemma} \label{lem:invariant_sets_VP} 
Let Assumptions \ref{asp:stand} and \ref{asp:abg} hold and let $y(s)$ solve ODE~\eqref{eq:reverse_ODE_VP_empirical}.
Fix any $\delta > 0$ and $r > |x|_\infty$. If $y(s^\ast) \in V_{\delta}^{r}(x^n_0),$ for some $x^n_0$ and $s^\ast > \max(s^D,s^1,s^2,s^\alpha)$, then $y(s) \in V_{\delta}^{r}(x^n_0)$ for all $s \geq s^\ast$.
\end{lemma}

\begin{lemma} \label{cor:exponential_convergence_newcoords_VP} Let Assumptions \ref{asp:stand} and \ref{asp:abg} hold and let $y(s)$ solve ODE~\eqref{eq:reverse_ODE_VP_empirical}. 
Fix any $\delta > 0$ and $r > |x|_\infty$. If $y(s^\ast) \in V_{\delta}^{r}(x^n_0)$ for some $x_0^n$ and $s^\ast > \max(s^r,s^D,s^1,s^2,s^\alpha)$, then for all $s \geq s^\ast$ we have
\begin{subequations}
\begin{align}
|y(s) - x^n_0| &\leq Ke^{-s}. \\
             K &= |y(s_0) - x_0^1|e^{s_0} + 4(N-1)|x|_\infty \frac{e^{-s_0}}{\delta^2}e^{-e^{2s_0}\delta^2/4} \notag \\
             &\;\;\;\; + |x|_\infty \left[1 - e^{s_0}(\sqrt{1 - e^{-2s_0}} - 1)\right]. 
\end{align}
\end{subequations}
\end{lemma}

\subsection{Numerics: Variance Preserving Process} 
\label{ssec:numerics_VP}

In this section we provide numerical demonstrations that illustrate the behavior of ODE~\eqref{eq:addd} for the variance preserving process with the optimal empirical score function, analogously to Section~\ref{sec:numerics_VE}. We consider the variance preserving process
from Example~\ref{ex:VP} with drift $g(t) = \beta(t) = \beta_{\text{min}} + t(\beta_{\text{max}} - \beta_{\text{min}})$. We choose the parameters $\beta_{\text{max}} = 3$ and $\beta_{\text{min}} = 0.001$. Figure~\ref{fig:random_Voronoi_teselation_VP} plots the dynamics of ODE~\eqref{eq:addd} for the same empirical data distribution $p_0^N$ in Section~\ref{sec:numerics_VE} given by $N = 20$ i.i.d.\thinspace samples drawn from a two-dimensional standard Gaussian distribution $p_0 = \mathcal{N}(0,I_2)$. The trajectories of the ODE solution $x(t)$ for $0 < t \leq T$ integrated backwards-in-time starting from four independent initial conditions $x_T \sim \mathcal{N}(0,I_2)$ at $t = T = 1$ are plotted in red. In each plot, the Voronoi tessellation for the data samples is plotted in black and the points in $p_0^N$ are plotted in blue.

As proved in Appendix~\ref{ssec:VPanalysis}, we observe memorization in all four trajectories: the solution converges to a limit point from the empirical data distribution $p_0^N$. While the dynamics explore the bulk of the distribution initially, after sufficient time has passed, the weights of the ODE concentrate on a single data point. After this time, the solution is always contained within the Voronoi cell of the corresponding point before collapsing to its center. Figure~\ref{fig:random_exponential_convergence_VP} displays the convergence rate for the Euclidean distance between 30 independent trajectories and their limit points in the 
original time variable $t$. The rates match the expected result from Theorem~\ref{thm:memorization}, where the rate is independent of the data distribution.

\begin{figure}[!ht]
\centering
\includegraphics[width=0.48\textwidth]{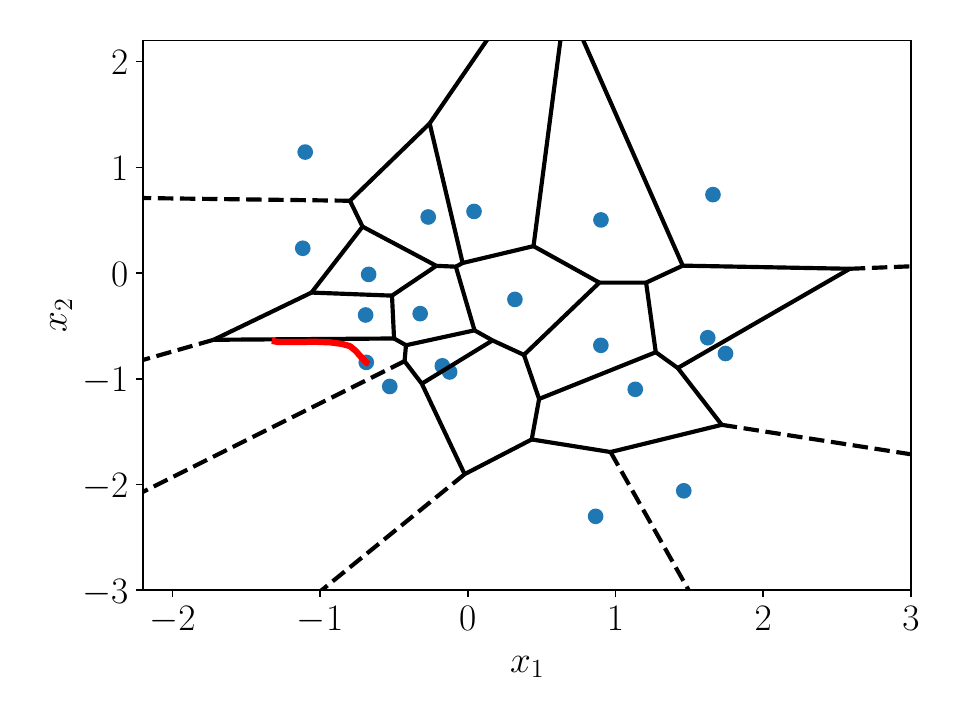}
\includegraphics[width=0.48\textwidth]{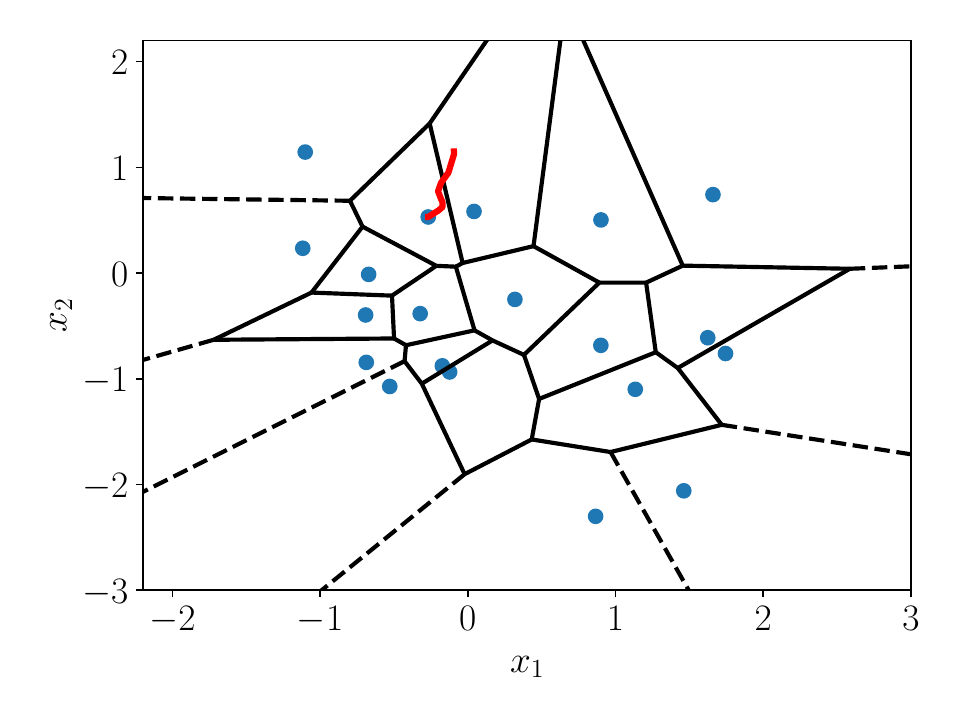}
\includegraphics[width=0.48\textwidth]{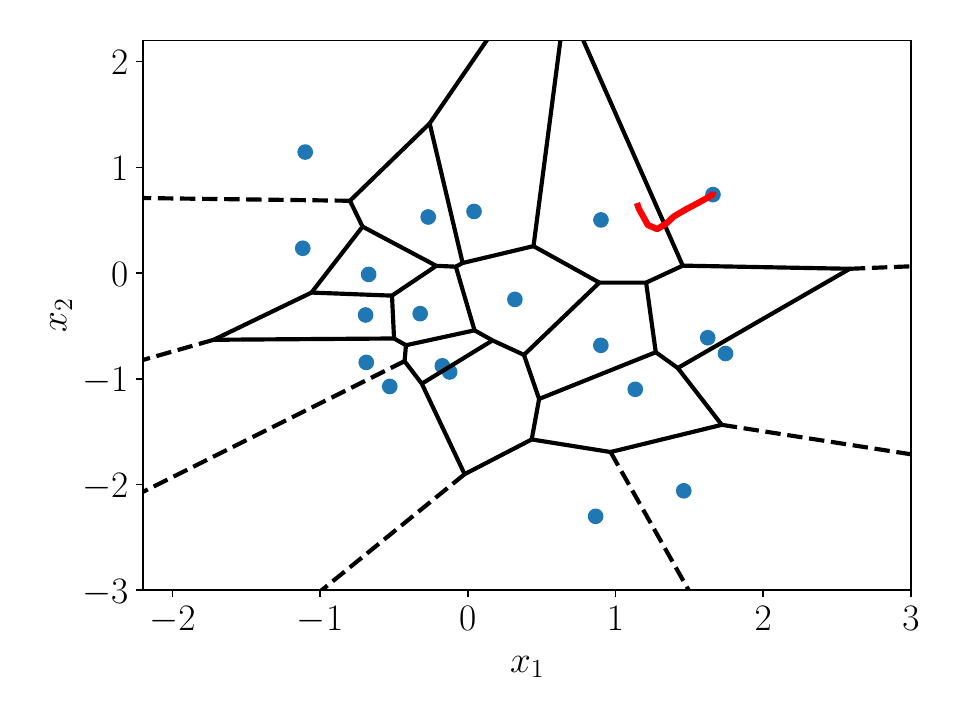}
\includegraphics[width=0.48\textwidth]{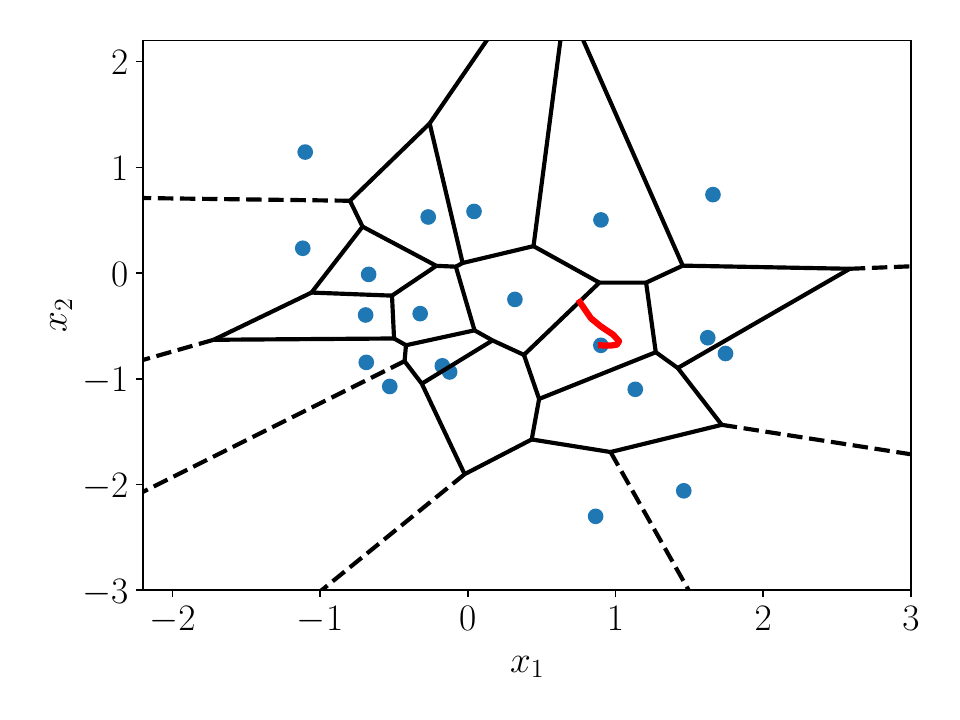}
\caption{Convergence of the reverse ODE trajectories for the variance preserving process to the data points in red when using the empirical score function starting from four initial conditions. The Voronoi tessellation is plotted for $N = 20$ samples in blue. \label{fig:random_Voronoi_teselation_VP}}
\end{figure}

\begin{figure}[!ht]
\centering
\includegraphics[width=0.48\textwidth]{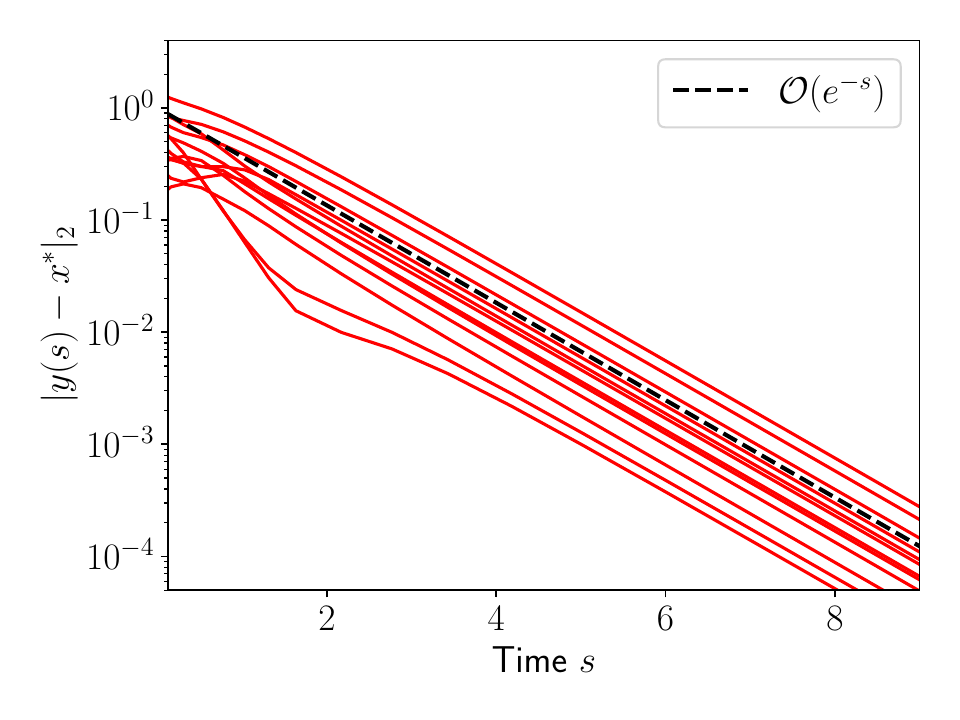}
\includegraphics[width=0.48\textwidth]{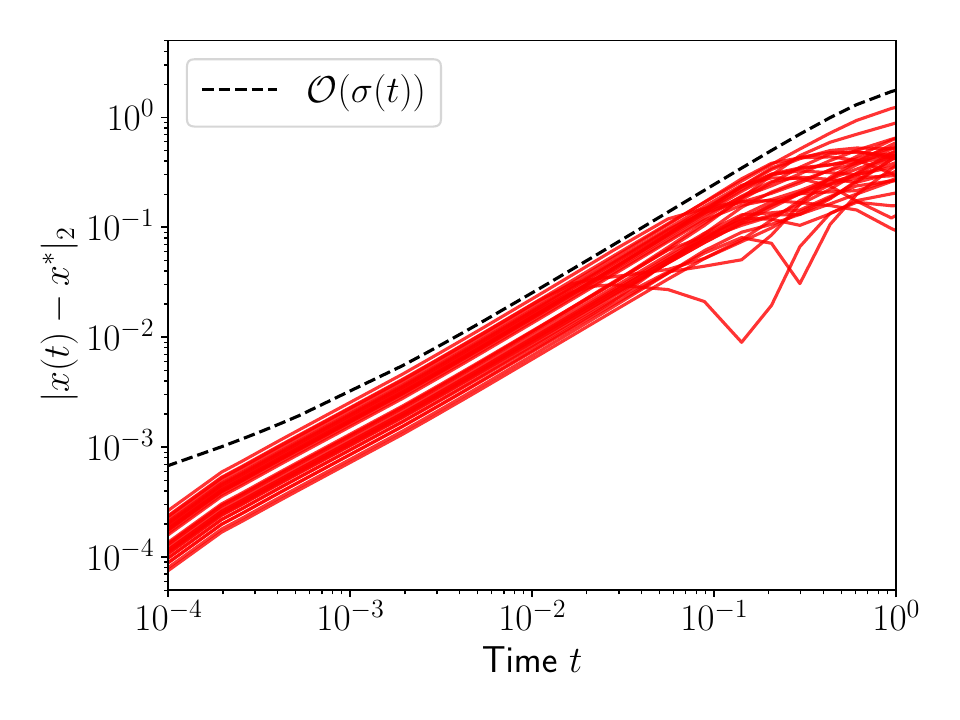}
\caption{Convergence rate of the reverse ODE solutions to the data points for the variance preserving process with the empirical score function for 30 independent trajectories in the transformed time $s$ (\textit{left}) and original time $t$ (\textit{right}).
\label{fig:random_exponential_convergence_VP}}
\end{figure}

\end{document}